%% file: main.tex
\documentclass{article} 
\usepackage{PRIMEarxiv}

\input{math_commands.tex}

\usepackage{hyperref}
\usepackage{url}
\usepackage{amsmath}
\usepackage{amssymb}
\usepackage{booktabs}
\usepackage{multirow}
\usepackage{float}
\usepackage{mathtools}
\usepackage{amsthm}
\usepackage{subcaption}
\usepackage{url}
\usepackage{bbm}
\usepackage{xcolor}
\usepackage{mathrsfs}
\usepackage{enumitem}
\usepackage{hyperref}       
\usepackage{tabularx}
\usepackage{algorithm}      
\usepackage{algpseudocode}  

\usepackage[capitalize,noabbrev]{cleveref}

\theoremstyle{plain}
\newtheorem{theorem}{Theorem}[section]

\newtheorem{lemma}[theorem]{Lemma}
\newtheorem{corollary}[theorem]{Corollary}
\theoremstyle{definition}

\newtheorem{assumption}[theorem]{Assumption}
\theoremstyle{remark}
\newtheorem{remark}[theorem]{Remark}

\title{Distribution-Free Robust Predict-Then-Optimize in Function Spaces}

\author{%
  Yash Patel \\
  Department of Statistics\\
  University of Michigan\\
  Ann Arbor, MI 48104 \\
  \texttt{yppatel@umich.edu} \\
  \And
  Ambuj Tewari \\
  Department of Statistics\\
  University of Michigan\\
  Ann Arbor, MI 48104 \\
  \texttt{tewaria@umich.edu} \\
}

\begin{document}

\maketitle

\begin{abstract}
    The need to rapidly solve PDEs in engineering design workflows has spurred the rise of neural surrogate models. In particular, neural operator models provide a discretization-invariant surrogate by retaining the infinite-dimensional, functional form of their arguments. Despite improved throughput, such methods lack guarantees on accuracy, unlike classical numerical PDE solvers. Optimizing engineering designs under these potentially miscalibrated surrogates thus runs the risk of producing designs that perform poorly upon deployment. In a similar vein, there is growing interest in automated decision-making under black-box predictors in the finite-dimensional setting, where a similar risk of suboptimality exists under poorly calibrated models. For this reason, methods have emerged that produce adversarially robust decisions under uncertainty estimates of the upstream model. One such framework leverages conformal prediction, a distribution-free post-hoc uncertainty quantification method, to provide these estimates due to its natural pairing with black-box predictors. We herein extend this line of conformally robust decision-making to infinite-dimensional function spaces. We first extend the typical conformal prediction guarantees over finite-dimensional spaces to infinite-dimensional Sobolev spaces. We then demonstrate how such uncertainty can be leveraged to robustly formulate engineering design tasks and characterize the suboptimality of the resulting robust optimal designs. We then empirically demonstrate the generality of our functional conformal coverage method across a diverse collection of PDEs, including the Poisson and heat equations, and showcase the significant improvement of such robust design in a quantum state discrimination task.
\end{abstract}

\section{Introduction}\label{section:intro}
Much of engineering design centers on optimizing design parameters under a PDE-constrained functional, such as optimizing car or aircraft designs for drag minimization or structural designs for withstanding stress \cite{sokolowski1992introduction,hsu1994review,challis2009level,dunning2015introducing}. Traditionally, workflows would require repeatedly running domain-specific numerical PDE solvers to evaluate the designs as they were iteratively refined \cite{zhang2006application, caron2025machine}. Such workflows, however, suffered from slow iteration time, as numerical PDE solvers incur a significant computational cost that cannot be amortized over their runs across different designs. For this reason, significant interest has arisen in neural surrogate models that learn a ``flow map,'' with which a PDE can be efficiently approximately solved across different input conditions with a forward pass through a trained neural network \cite{pathak2022fourcastnet,wen2022accelerating,bonev2023spherical}.

One concern with purely relying on such surrogate models in design optimization pipelines is that they lack any guarantees of recovering the true solution functions, unlike classical numerical solvers. This, in turn, has led to the proliferation of methods to provide uncertainty estimates for such models \cite{zou2023hydra,psaros2023uncertainty,zhu2019physics,martina2021bayesian,tripathy2018deep}. Such uncertainty quantification methods, however, rely on distributional assumptions; the guarantees of these methods, in turn, become vacuous under distributional misspecification \cite{sahin2024deep,mollaali2023physics}. For this reason, recent efforts have been directed towards developing distribution-free, data-driven approaches to uncertainty quantification by leveraging conformal prediction \cite{gopakumar2024valid,ma2024calibrated}, a principled framework for producing distribution-free prediction regions with marginal frequentist coverage guarantees \cite{angelopoulos2021gentle, shafer2008tutorial}.

These initial efforts to leverage conformal prediction, however, have two primary shortcomings. The first is that they fail to provide coverage of the infinite-dimensional functions being predicted by neural operators. Initial work in this direction only produced coverage guarantees on predictions of a fixed-discretization, sacrificing the discretization-invariant property that is central to neural operators \cite{gopakumar2024valid,ma2024calibrated}. A recent work took steps towards addressing this deficiency by guaranteeing simultaneous coverage up to some maximally observed resolution; this approach, however, still fails to provide coverage over the untruncated function space \cite{grayguaranteed}. Second, such uncertainty quantification has yet to be leveraged for downstream use cases. The space of finite-dimensional conformal prediction followed a similar trend, with the proliferation of methods that produce calibrated regions with only more recent work discovering their applicability to decision-making tasks \cite{lekeufack2024conformal,cresswell2024conformal,kiyani2025decision,cortes2024utility}.

One such decision-making framework that pairs naturally with conformal prediction is predict-then-optimize. In this setting, a decision-making task is framed as an optimization problem with a known parametric form but an unknown parameter. For instance, one may be interested in a shortest paths task over a city with \textit{unknown} traffic along roads. To resolve this lack of information, one often uses side information to approximate this parameter, against which the decision can be made, hence the name: the parameter is ``predicted'' by an upstream model, after which the final decision can be made by ``optimization.''
Recent work, however, has demonstrated that naively treating the predicted parameter as being the true, unknown parameter can lead to highly suboptimal decision-making if the predicted parameter is misaligned with the target parameter \cite{sadana2025survey}. For this reason, much work has focused on propagating the uncertainty in the upstream predictions to the optimization task to make decisions that are robust to potential misspecification \cite{chenreddy2022data,sadana2024survey,chenreddy2024end}.

In this manuscript, we extend this line of work on robust predict-then-optimize to function spaces. Our main contributions are as follows:

\begin{itemize}
    \item Extending conformal guarantees to operator methods to provide formal coverage guarantees for infinite-dimensional functions over Sobolev spaces.
    \item Providing a framework for leveraging such conformal uncertainty sets for robust design optimization with a novel multi-resolution robust optimization pipeline.
    \item Demonstrating empirically the calibration of the proposed conformal method across a collection of PDEs and the improvement over nominal design optimization across both resource collection and quantum state discrimination tasks.
\end{itemize}

\section{Background}\label{section:background}

Since our contribution sits at the intersection of conformal prediction, predict-then-optimize and operator learning, we begin with reviewing relevant background material from these areas.

\subsection{Conformal Prediction}\label{section:bg_cp}
Conformal prediction is a principled, distribution-free uncertainty quantification method \cite{angelopoulos2021gentle, shafer2008tutorial}. ``Split conformal,'' the most common variant of conformal prediction, is used as a wrapper around black-box predictors $\widehat{f} : \mathcal{X}\rightarrow\mathcal{Y}$ such that prediction \textit{regions} $\mathcal{C}(x)$ are returned in place of the typical point predictions $\widehat{f}(x)$. Prediction regions $\mathcal{C}(x)$ are specifically sought to have coverage guarantees on the true $y := f(x)$. That is, for some prespecified $\alpha$, we wish to have $\mathcal{P}_{X,Y}(Y\in \mathcal{C}(X))\ge1-\alpha$.

To achieve this, split conformal partitions the dataset $\mathcal{D} = \mathcal{D}_{T}\cup\mathcal{D}_{C}$, respectively the training and calibration sets. The training set is used to fit $\widehat{f}$. After fitting $\widehat{f}$, the calibration set is used to measure the anticipated ``prediction error'' for future test points. Formally, this error is quantified via a score function $s(x,y)$, which generalizes the classical notion of a residual. In particular, scores are evaluated on the calibration dataset to define $\mathcal{S} := \{s(x,y)\mid (x,y)\in\mathcal{D}_{C}\}$. Denoting the $k$-th order statistic for $k := \ceil{(N_{C}+1)(1-\alpha)}$ of $\mathcal{S}$ as $\widehat{q}$, where $N_{C} := |\mathcal{D}_{C}|$, conformal prediction defines $\mathcal{C}(x)$ to be $\{y \mid s(x, y) \le \widehat{q}\}$. Such $\mathcal{C}(x)$ satisfies the desired coverage guarantees under the exchangeability of test points $(x',y')$ with points in $\mathcal{D}_{C}$.

While the coverage guarantee holds for any arbitrarily specified score function, the conservatism of the resulting prediction region, known as the procedure's ``predictive efficiency,'' is dependent on its choice \cite{shafer2008tutorial}. The objective of conformal prediction, therefore, is to define score functions that retain coverage while minimizing the resulting prediction region size. 

\subsection{Predict-Then-Optimize}\label{section:bg_pred_opt}
Predict-then-optimize problems are nominally given by
\begin{equation}\label{eqn:nominal_po}
\begin{aligned}
w^{*}(x) := \argmin_{w\in\mathcal{W}} \quad & \mathbb{E}[f(w, C)\mid X = x],
\end{aligned}
\end{equation}
where $w$ are decision variables, $C$ an \textit{unknown} cost parameter, $x$ observed contextual variables, $\mathcal{W}$ a compact feasible region, and $f(w, c)$ an objective function. The nominal approach defines a $\widehat{g} : \mathcal{X}\rightarrow\mathcal{C}$, where the prediction $\widehat{c} := \widehat{g}(x)$ is directly leveraged for the decision making, i.e., taking $w^{*} := \argmin_{w} f(w, \widehat{c})$.

Such an approach, however, is inappropriate in safety-critical settings, given that the predictor function $\widehat{g}$ will likely be misspecified and, thus, may result in suboptimal decisions under the true cost parameter, which we denote as $c$. For this reason, robust alternatives to the formulation given by \Cref{eqn:nominal_po} have become of interest. We focus on the formulation posited in \cite{patel2024conformal}, which extended the line of work begun in \cite{chenreddy2022data,sadana2024survey,chenreddy2024end}. They studied 
\begin{equation}\label{eqn:reform_obj}
\begin{gathered}
w^{*}(x) := \argmin_{w} \max_{\widehat{c}\in\mathcal{U}(x)} \quad f(w, \widehat{c})
\quad \textrm{s.t.} \quad \mathcal{P}_{X,C}(C\in\mathcal{U}(X)) \ge 1-\alpha,
\end{gathered}
\end{equation}
where $\mathcal{U} : \mathcal{X}\rightarrow\mathscr{P}(\mathcal{C})$ is a uncertainty region predictor, with $\mathscr{P}(\cdot)$ denotes the power set. Work in this field typically focuses on both theoretically characterizing and empirically studying the resulting suboptimality gap, defined as
$\Delta^{*}(x, c) := \min_{w} \max_{\widehat{c}\in\mathcal{U}(x)} f(w, \widehat{c}) - \min_{w} f(w, c)$. 
For instance,
in \cite{patel2024conformal}, $\mathcal{U}(x)$ was specifically constructed via conformal prediction to provide probabilistic guarantees; that is, by
taking $\mathcal{U}(x) := \mathcal{C}(x)$ to be the prediction region of a conformalized predictor $\widehat{g} : \mathcal{X}\rightarrow\mathcal{C}$, they demonstrated that, if $f(w,c)$ is convex-concave and $L$-Lipschitz in $c$ for any fixed $w$, $\mathcal{P}_{X,C}\left(0\le \Delta^{*}(X, C) \le L \mathrm{\ diam}(\mathcal{U}(X))\right) \ge 1 - \alpha$. 

\subsection{Sobolev Spaces}\label{section:sobolev_spaces}
The study of numerical simulation of PDEs is a mature field. Sobolev spaces offer a natural framework to reason about PDE solutions. We only provide a brief introduction to the topic, referring readers to the book \cite{brezis2011functional} for an excellent treatment of the relevant materials. Differential problems are posited in the form
\begin{equation}\label{eqn:pde_setup}
    D u(x) = f(x)\quad x\in\Omega
    \qquad u(x) = 0\quad x\in\partial\Omega,
\end{equation}
where $\Omega\subset\mathbb{R}^{d}$ is a compact domain, $f,u : \mathbb{R}^{d}\rightarrow\mathbb{R}$ are scalar fields, and $D$ is a differential operator. The goal in ``solving'' a PDE is to find a function $u$ satisfying \Cref{eqn:pde_setup} for a specified $f$. To do so, however, $u$ has to be sufficiently smooth, else $D u$ may not be well-defined. Formally, this space of smooth functions is a ``Sobolev space,'' over which much of the classical theory of the existence and uniqueness of solutions for PDEs is established. Formally, a Sobolev space is defined as the space of functions with bounded Sobolev norm, i.e., $\mathcal{W}^{s,p}(\Omega) := \{u(x) : \| u \|_{\mathcal{W}^{s,p}(\Omega)} < \infty\}$, where
\begin{equation}\label{eqn:sobolev_norm}
    \begin{gathered}
        \| u \|^{p}_{\mathcal{W}^{s,p}(\Omega)} := \sum_{\alpha\in\Lambda_{\le s}} \| \partial_{x}^{\alpha} u \|_{\mathcal{L}^{p}(\Omega)}^{p}
        \qquad\mathrm{and}\qquad\Lambda_{\le s} := \{\alpha\in\mathbb{N}_{0}^{d} : \| \alpha \|_{1} \le s\}.
    \end{gathered}
\end{equation}
Note that we employ the common condensed notation $\partial_{x}^{\alpha} u := \partial_{x_{1}}^{\alpha_{1}} ...\partial_{x_{d}}^{\alpha_{d}} u$ for $\alpha := (\alpha_{1},...,\alpha_{d})$. Since all partials are with respect to $x$ in this manuscript, we condense the notation further and simply denote this operator as $\partial^{\alpha}$. Notably, this space assumes a Hilbert structure in the special case of $p=2$, which we denote as $\mathcal{H}^{s}(\Omega) := \mathcal{W}^{s,2}(\Omega)$. In the further specialized case of $\Omega = \mathbb{T}^{d}$, the space $\mathcal{H}^{s}(\mathbb{T}^{d})$ can be defined by the equivalent norm over the function's Fourier spectrum arising from Parseval's identity, namely
\begin{equation}\label{eqn:spectral_sobolev_norm}
    \| u \|^{2}_{\mathcal{H}^{s}(\mathbb{T}^d)} := \sum_{n\in\mathbb{Z}^{d}} (1 + \| n \|_{2}^{2})^{s} \langle u, \varphi_n \rangle^{2}
    \quad\mathrm{where}\quad\varphi_n := e^{2\pi i n\cdot x},
\end{equation}
The notion of ``equivalent norms'' is the standard definition, where $\| \cdot \|_{a}$ and $\| \cdot \|_{b}$ are called equivalent if there exist constants $c$ and $C$ such that for any $x$, $c \| x \|_{b}\le \| x \|_{a}\le C \| x \|_{b}$. 

\subsection{Neural Operators}\label{section:neural_operators}
Data-driven approaches to modeling have classically focused on learning maps between finite-dimensional spaces. With the increasing interest in leveraging machine learning in domains such as solving PDEs, however, approaches that learn maps between infinite-dimensional function spaces have emerged. \emph{Operator learning} methods seek to learn a map $\mathcal{G} : \mathcal{A}\rightarrow\mathcal{U}$ between two Banach spaces $\mathcal{A}$ and $\mathcal{U}$, where observations $\mathcal{D}:=\{(a^{(i)}, u^{(i)})\}$ have been made for $a^{(i)}\sim\mu$ with $\mu$ being a probability measure supported on $\mathcal{A}$. We assume that there exists some true, deterministic operator $\mathcal{G}$ such that $u = \mathcal{G}(a)$. While many different learning-based approaches have been proposed to solve this learning problem, they can generally be framed as seeking to recover this true map optimally under the Bochner norm, formally
\begin{equation}\label{eqn:operator_obj}
    \min_{\widehat{\mathcal{G}}} \|\widehat{\mathcal{G}} - \mathcal{G}\|^{2}_{\mathcal{L}_{\mu}^{2}(\mathcal{A},\mathcal{U})}
    := \int_{\mathcal{A}} \|\widehat{\mathcal{G}}(a) - \mathcal{G}(a)\|_{\mathcal{U}}^{2} d\mu(a).
\end{equation}
While the operator learning task can be framed in this general light, most work studying operator learning methods has focused on the setting of PDEs, where it is of interest to learn the solution operator of a given PDE \cite{li2020neural,li2020multipole,li2020fourier,bonev2023spherical}. 


In this manuscript, we seek uncertainty quantification over function spaces; some operator learning methods based on discretized domain representations do not lend themselves naturally to this characterization over the full resolution function spaces. One class of approaches that does, however, is ``spectral neural operators'' \cite{fanaskov2023spectral,wu2023solving,du2023neural,liu2023spfno}. Classical spectral methods posit that the solution function can be decomposed as $u(x) = \sum_n u_{n} \varphi_n(x)$ for some pre-specified $\{\varphi_n\}$ basis and solve for the corresponding $\{u_{n}\}$ coefficients. In most practical uses of spectral methods, these coefficients are solved up to some discretization $N^{d}$. Spectral neural operators, in turn, are based on this decomposition, from which the neural operator objective of \Cref{eqn:operator_obj} reduces to a simple vector-to-vector regression loss for the map $\mathcal{G} : \mathbb{R}^{N^{d}}\rightarrow\mathbb{R}^{N^{d}}$ over the spectral representations of $\vec{a}^{(i)},\vec{u}^{(i)}\in\mathbb{R}^{N^{d}}$. 

\section{Method}
We now discuss our proposed methodology for performing robust engineering design under conformal prediction regions in function spaces. In particular, we introduce the calibration procedure in \Cref{section:op_calibration}, followed by a general discussion of its use in robust functional predict-then-optimize problems in \Cref{section:robust_fpo}. We then demonstrate the utility of this framework in \Cref{section:experiments}.

\subsection{Notation}\label{section:notation}
Throughout this exposition, we assume the PDE surrogate map is learned as a spectral neural operator, as described in \Cref{section:neural_operators}. For this reason, we introduce the following projection notation to denote spectrum truncations $\Pi_N : \mathcal{H}^s(\mathbb T^d) \to \mathcal{H}^s(\mathbb T^d)$.
Critically, while the output spectrum is truncated by virtue of approximation, the input spectrum is assumed to have been specified exactly. That is, the true \textit{output} may have non-zero modes outside the truncation at $N$. We assume, however, that there is no comparable ``approximation error'' on the input side, as the user has full control of the degree of approximation they wish to model the input at. 
Formally, therefore, the spectral operator is learned as a finite-dimensional map $\mathcal{G} : \mathbb{R}^{N^{d}}\rightarrow\mathbb{R}^{N^{d}}$ on a dataset of functions $\{(a^{(i)},\Pi_N u^{(i)})\}$. 

Throughout the exposition, we make extensive use of Sobolev norm measurements across spaces of varying smoothness. We also focus on the periodic boundary condition setting. For this reason, we adopt the condensed notational convention of denoting $\| u \|_{s} := \| u \|_{\mathcal{H}^{s}(\mathbb{T}^d)}$.


\subsection{Spectral Operator Calibration}\label{section:op_calibration}

As mentioned, one crucial detail that distinguishes conformal guarantees in functional settings versus those in typical settings is that the outputs in this setting are fundamentally only \textit{partially} observable. That is, functions are not directly observable: only discrete samplings, either as evaluations in the spatial domain or as truncated spectral representations, can be observed. We, however, seek to provide coverage guarantees on the \textit{full} function, i.e., where the spectral expansion is not truncated. To achieve this, we define the following family of score functions
\begin{equation}\label{eqn:score_func}
    s_{N;\tau}(a, u) := \| \mathcal{G}(a) - \Pi_N u \|_{{s-\tau}}^{2}
    \quad\mathrm{where}\quad\tau\in\{1,...,s\}.
\end{equation}
This choice of score is motivated by its equivalence via Parseval's theorem to the $(s-\tau)$-Sobolev norm residual. Notably, the score function itself is parameterized by two values: the spectral truncation $N$ and the Sobolev smoothness parameter $\tau$. As discussed extensively over later sections, the scoring over variable truncations is leveraged for optimization, for which the efficient computation of the score over multiple choices of $N$ is critical. Importantly, the quantiles over all truncations can be computed with a single vectorized operation.  

Critical to note is that the score was defined using the $(s-\tau)$-Sobolev norm with $\tau\ge 1$ rather than the more natural $s$-Sobolev norm, for two primary reasons. The first is that the resulting prediction region, a ball in the $\mathcal{H}^{s}(\mathbb{T}^{d})$ space under the $(s-\tau)$ norm, is compact only if $\tau > 0$ by the Sobolev embedding theorem. While we only consider discretized representations of functions in the optimization sections of this manuscript (see \Cref{section:multi_stage_opt}), future works may wish to directly optimize over the non-discretized function space. In such cases, convergence guarantees often require compactness of the optimization domain \cite{ulbrich2024generalized}.

In addition, the choice of the $(s-\tau)$ norm guarantees asymptotic conformal coverage of the true function, as we formalize in \Cref{thm:func_cov}. Note again that the statement of \Cref{eqn:coverage_guarantee} is made directly on the \textit{full} spectrum $u'$, not on its finite spectral projection, i.e., not on $\Pi_N u'$. Intuitively, the probabilistic bound is achieved by leveraging the conformal quantile to control the behavior of the observed lower-order modes and the smoothness of functions in $\mathcal{H}^{s}(\mathbb{T}^{d})$ to ensure the decay of higher-order modes. To do so, we require that the output function be bounded as a function of the input instance, as formalized in \Cref{assumption:output_bound}. We defer the full proof of \Cref{thm:func_cov} to \Cref{section:coverage_guarantee}. 

\begin{assumption}[Instance-Dependent Output Bound] \label{assumption:output_bound}
    There exists a measurable function $B : \mathcal{A}\to[0,\infty)$ such that, for any $(A, U)\sim\mathcal{P}$, $\| U \|^{2}_{s}\le B(A)$ a.s.
\end{assumption}

\begin{theorem}\label{thm:func_cov}
    Let $\{(A^{(i)}, U^{(i)})\}\cup (A', U')\overset{\mathrm{iid}}{\sim}\mathcal{P}$ satisfy \Cref{assumption:output_bound} for some $B(A)$. Further, let $\mathcal{D}_{C} := \{(A^{(i)}, \Pi_N U^{(i)})\}$. Let $\alpha\in(0,1)$ and $\widehat{q}_{N;\tau}$ be the $k$-th order statistic for $k := \ceil{(N_{C}+1)(1-\alpha)}$ of \Cref{eqn:score_func} over $\mathcal{D}_{C}$ for a fixed $\mathcal{G}$ and $\tau\in\{1,...,s\}$. Then
    \begin{equation}\label{eqn:coverage_guarantee}
        \mathcal{P}_{\{(A^{(i)}, U^{(i)})\}, (A', U')}\left(
        \| \mathcal{G}(A') - U' \|_{{s-\tau}}^{2}
        \le \widehat{q}_{N;\tau} + B(A') N^{-2\tau} \right) \ge 1-\alpha.
    \end{equation}
\end{theorem}

The implication of such statements is that the typical conformal quantile retains coverage guarantees on the underlying function if augmented with an additional, finite margin, which decays with a ``more complete'' observation of the function, i.e., with larger $N$. Intuitively, with observation of the full function, i.e., when $N\rightarrow\infty$, this margin becomes zero. While such a result is natural, an interesting note is that such a property does \textit{not} hold if we took $\tau = 0$ to define the score in \Cref{thm:func_cov}. This result also motivates the choice of $\tau=s$ to achieve a faster decay rate of the margin; however, this needs to be balanced against the increased conservatism of the prediction regions that results from using higher values of $\tau$.

We denote this margin-padded quantile as $\widehat{q}_{N;\tau}^{*}(a) := \widehat{q}_{N;\tau} + B(a) N^{-2\tau}$, for which the corresponding prediction region $\mathcal{C}_{N;\tau}^{*}(a) := \{ u : \| \mathcal{G}(a) - u \|_{{s-\tau}}^{2} \le \widehat{q}_{N;\tau}^{*}(a) \}$
has marginal coverage guarantees per \Cref{eqn:coverage_guarantee}. Notably, the radius of the prediction region is instance-dependent in this setting. We proceed through the remaining section assuming recovery of $\mathcal{C}_{N;\tau}^{*}(a)$ and demonstrate how this margin can be explicitly procured for specific PDEs in \Cref{section:coverage_guar_exp}. Critically, this coverage of the underlying function lends itself to coverage under arbitrary projections, a property we will exploit for optimization over subsequent sections.

\subsection{Calibration Across PDE Families}
From the generally posited result of \Cref{thm:func_cov}, a number of corollaries of commonly encountered families of PDEs immediately follow by establishing explicit expressions for the bound $B(a)$. 

\subsubsection{Elliptic PDEs}\label{section:elliptic_cor}
In particular, we can immediately make similar coverage claims on a commonly encountered class of elliptic PDEs by appealing to results from classical results of elliptic regularity theory. 
Briefly, an order 2 elliptic PDE operator $L$ is given by
\begin{equation}\label{eqn:elliptic_operator}
    L = -\sum_{k,l=1}^d a_{k,l}(x) \partial_{x_k x_l} + \sum_{k=1}^d b_k(x) \partial_{x_k} + c(x), 
\end{equation}
and is one for which $\sum_{k,l=1}^d a_{k,l}(x)\xi_{k}\xi_{l} > 0$ a.e. If, additionally, there is a lower bound uniformly in $x$, i.e. there exists a $\theta > 0$ such that $\sum_{k,l=1}^d a_{k,l}(x)\xi_{k}\xi_{l} \ge \theta|\xi|^{2}$, then this PDE is \textit{uniformly} elliptic. Here, $x\in \mathbb{T}^{d}$ and $\xi\in\mathbb{R}^{d}$. The ``order'' of such an operator specifies its highest total derivative.
In particular, we can establish the following corollary; we defer the statement of the relevant result of classical regularity theory from which this result follows to \Cref{section:elliptic_bound_pf}. Note that such a result can be practically leveraged, as we assumed the inputs are specified without error, implying that $\| f \|_{s-2}$ can be computed with finitely many terms.

\begin{corollary}\label{corollary:elliptic_coverage}
    For $s\in\mathbb{N}$ such that $s\ge 2$, let $\{(F^{(i)}, U^{(i)})\}\cup (F', U')\sim\mathcal{P}$, where $L U = F$ for a uniformly elliptic operator $L$ of order 2 on $\mathbb{T}^{d}$ and $U \in \mathcal{H}^{s}(\mathbb{T}^{d}) \cap \mathrm{Ker}(L)^{\perp}$ a.s. Further, let $\mathcal{D}_{C} := \{(F^{(i)}, \Pi_N U^{(i)})\}$. Let $\alpha\in(0,1)$, $\mathcal{G}$, $N$, $\tau\in\{1,...,s\}$, and $\widehat{q}_{N;\tau}$ be as defined in \Cref{thm:func_cov} with respect to such $\mathcal{D}_{C}$. Then, there exists $C_{s,L} < \infty$, 
     \begin{equation*}
        \mathcal{P}_{\{(F^{(i)}, U^{(i)})\}, (F',U')}\left(\| \mathcal{G}(F') - U' \|_{s-\tau}^{2}
        \le \widehat{q}_{N;\tau} + C_{s,L} \| F' \|^{2}_{s-2} N^{-2\tau} \right) \ge 1-\alpha
    \end{equation*}
\end{corollary}

Notably, the requirement that $u\in\mathrm{Ker}(L)^{\perp}$ is to ensure $u$ is a unique solution to the posited PDE. An important special case of \Cref{corollary:elliptic_coverage} is when $L = \Delta$. In this case, the requirement that $u$ be a unique solution can be naturally enforced by restricting $u$ to zero-mean fields, i.e., restricting solutions to satisfy $\int_{\mathbb{T}^{d}} u(x) \, dx = 0$. We fully characterize and empirically study this special case in \Cref{section:addn_coverage}.

\subsubsection{Parabolic PDEs}
For a uniformly elliptic differential operator $L$, a parabolic PDE is one for which $\partial_t u = L u$, where now $u : \mathbb{T}^{2}\times[0,T]\rightarrow\mathbb{R}$ denotes a spatiotemporal field. For notational ease, we denote $u_t := u(\cdot, t)$. Unlike the elliptic PDE setting from above, therefore, there is no longer a single solution operator $\mathcal{G}$ but rather a \textit{family} of solution operators indexed by $t$. That is, we have $\{\mathcal{G}(t)\}_{t\in\mathbb{R}}$ is a semigroup of solution operators. Intuitively, this result follows in demonstrating that elliptic operators generate \textit{analytic} semigroups, which have the property that allows for a decaying bound on the norm of the evolved state. The full proof is deferred to \Cref{section:parabolic_bound_pf}.

\begin{corollary}\label{corollary:parabolic_coverage}
    For $s\in\mathbb{N}$ such that $s\ge 2$ and a fixed $T > 0$, let $\{(U_0^{(i)}, U^{(i)}_T)\}\cup (U_0', U_T')\sim\mathcal{P}$, where $\partial_t U = -L U$ for a uniformly elliptic operator $L$ of order 2 on $\mathbb{T}^{d}$ and $U_0 \in \mathcal{H}^{s}(\mathbb{T}^{d})$ a.s. Further, let $\mathcal{D}_{C} := \{(U_0^{(i)}, \Pi_N U_T^{(i)})\}$. Let $\beta\ge 0$, $\alpha\in(0,1)$, $\mathcal{G}$, $N$, $\tau\in\{1,...,s\}$, and $\widehat{q}_{N;\tau}$ be as defined in \Cref{thm:func_cov} with respect to such $\mathcal{D}_{C}$. Then, there exist constants $C_{s,\beta} < \infty$ and $\omega_s\in\mathbb{R}$ such that
     \begin{equation*}
        \mathcal{P}_{\{(U_0^{(i)}, U^{(i)}_T)\},(U_0', U_T')}\left( \| \mathcal{G}(U_0') - U_T' \|_{s-\tau}^2
        \le \widehat{q}_{N;\tau} + C_{s,\beta} (T^{-\beta} + 1) e^{\omega_s T} \| U_0' \|^{2}_{s-2\beta} N^{-2\tau} \right) \ge 1-\alpha
    \end{equation*}
\end{corollary}

\subsection{Robust Functional Predict-then-Optimize}\label{section:robust_fpo}
We now consider the general functional predict-then-optimize problem setting and demonstrate the utility of the uncertainty set in such a setup in bounding the suboptimality of the resulting robust decision, akin to that discussed in \Cref{section:bg_pred_opt}. Formally, the standard, non-robust formulation is
\begin{equation}\label{eqn:fpo}
w^{*}(u) := \argmin_{w\in\mathcal{W}} J[w, u].
\end{equation}
We now introduce the robust counterpart to the nominal problem for a prediction and scoring performed at a spectral truncation $N$:
\begin{equation}\label{eqn:robust_func_pred_opt}
    w^{*}(a, N) := \argmin_{w\in\mathcal{W}} \max_{\widehat{u}\in\mathcal{C}_{N;\tau}^{*}(a)} J[w, \widehat{u}],
\end{equation}
Similar to previous robust predict-then-optimize work leveraging conformal guarantees, we can characterize the resulting suboptimality in this setting, namely
\begin{equation}\label{eqn:subopt_gap}
    \Delta^{*}(a, u, N) := \min_{w\in\mathcal{W}} \max_{\widehat{u}\in\mathcal{C}_{N;\tau}^{*}(a)} J[w, \widehat{u}] - \min_{w\in\mathcal{W}} J[w, u].
\end{equation}
We do so under a smoothness assumption on the objective function, given as follows. The proof of the theorem below is deferred to \Cref{section:coverage_bound}.

\begin{assumption}[Objective Smoothness] \label{assumption:obj_smooth}
    The objective $J[w,u]$ is $L$-Lipschitz in its second argument under the
    $\|\cdot\|_{s-\tau}$ norm, uniformly in $w \in \mathcal{W}$, i.e.,
    \[
    |J[w,u] - J[w,u']| \le L \|u-u'\|_{s-\tau}
    \quad \text{for all } w \in \mathcal{W},\ u,u' \in \mathcal{H}^s(\mathbb T^d).
    \]
\end{assumption}

\begin{theorem}\label{thm:subopt_gap}
    Let $\{(A^{(i)}, U^{(i)})\}, (A', U'), B(A), N, \mathcal{D}_{C},\widehat{q}_{N;\tau},\mathcal{G}$, and $\tau\in\{1,...,s\}$ be as defined in \Cref{thm:func_cov} and $\mathcal{C}_{N;\tau}^{*}(a)$ be the resulting margin-padded predictor for a marginal coverage level $1-\alpha$. Further, let $\Delta^{*}(a, u, N)$ be as defined in \Cref{eqn:subopt_gap} under this region $\mathcal{C}_{N;\tau}^{*}(a)$ for an objective $J[w,u]$. Then, if $J$ satisfies \Cref{assumption:obj_smooth},
    \begin{equation}\label{eqn:subopt_guarantee}
        \mathcal{P}_{\{(A^{(i)}, U^{(i)})\}, (A', U')}\left(0 \le \Delta^{*}(A', U', N) \le 2L \sqrt{\widehat{q}_{N;\tau} + B(A') N^{-2\tau}} \right) \ge 1-\alpha.
    \end{equation}
\end{theorem}

\subsection{Finite Projection Suboptimality}\label{section:optimization_strat}
This characterization of a suboptimality gap in the previous section was for the robust formulation given in \Cref{eqn:robust_func_pred_opt}, in which the adversary has access to perturbations $\widehat{u}\in\mathcal{C}_{N;\tau}^{*}(a)$ over the \textit{full} function space. In practical settings, however, problems are generally solved over finite-dimensional parameters derived from such functions, most commonly as direct finite-dimensional projections of the functions.

Similarly, therefore, we seek to trade off between the computational cost of solving \Cref{eqn:robust_func_pred_opt} and the tightness of the suboptimality guarantees we can produce in increased spectral truncations. Notably, however, this requires that the results of \Cref{thm:subopt_gap} be extended to cases where the adversary is only permitted to act over a finite-dimensional perturbation space.
Formally, we consider the finite-dimensional subspace of $V_N := \operatorname{Range}(\Pi_N) = \{ v\in\mathcal{H}^s(\mathbb{T}^d) : v = \Pi_N v \}$ and define the finite-dimensional adversary with
\begin{equation}\label{eqn:finite_pred_region}
    \mathcal{V}_{N;\tau}(a) := \left\{v\in V_N : 
    \| \mathcal{G}(a) - v \|_{s-\tau}^2 \le \widehat{q}_{N;\tau}, 
    \ \| v \|^{2}_{s} \le B(a) \right\},
\end{equation}
where $B(a)$ is defined as in \Cref{assumption:output_bound}.
Notably, this adversarial set is defined over the \textit{uncorrected} quantile radius $\widehat{q}_{N;\tau}$. We introduce the corresponding suboptimality as
\begin{equation}\label{eqn:finite_subopt_gap}
    \Delta^{(\mathcal{V})}(a, u, N) := \min_{w\in\mathcal{W}} \max_{\widehat{v}\in\mathcal{V}_{N;\tau}(a)} J[w, \widehat{v}] - \min_{w\in\mathcal{W}} J[w, u].
\end{equation}
Critically, this suboptimality can also be bounded analogously to \Cref{thm:subopt_gap}: the main conceptual difference is that the untruncated tail of the spectrum is bounded directly by the Lipschitz smoothness of the objective rather than through the enlarging of the adversarial set. The proof is deferred to \Cref{section:coverage_bound}.

\begin{corollary}\label{thm:subopt_gap_finite}
    Let $\{(A^{(i)}, U^{(i)})\}, (A', U'), B(A), N, \mathcal{D}_{C},\widehat{q}_{N;\tau},\mathcal{G}$, and $\tau\in\{1,...,s\}$ be as defined in \Cref{thm:func_cov} and $\mathcal{V}_{N;\tau}(a)$ be the resulting finite-dimensional predictor for a marginal coverage level $1-\alpha$ given by \Cref{eqn:finite_pred_region}. Further, let $\Delta^{(\mathcal{V})}(a, u, N)$ be as defined in \Cref{eqn:finite_subopt_gap} for an objective $J[w,u]$. Then, if $J$ satisfies \Cref{assumption:obj_smooth},
    \begin{equation}\label{eqn:subopt_gap_finite}
        \mathcal{P}_{\{(A^{(i)}, U^{(i)})\}, (A', U')}\left(0 \le \Delta^{(\mathcal{V})}(A', U', N) \le L \sqrt{4 \widehat{q}_{N;\tau} + B(A') N^{-2\tau}} \right) \ge 1-\alpha.
    \end{equation}
\end{corollary}

\subsection{Multi-Stage Optimization}\label{section:multi_stage_opt}
As with previous works of robust optimization discussed in \Cref{section:bg_pred_opt}, we now seek an efficient optimization strategy to solve the finite-dimensional robust optimization problem posited in \Cref{eqn:finite_subopt_gap}. Naively, this problem can simply be treated as a finite-dimensional robust optimization problem over the desired discretization resolution, from which approaches paralleling those presented in such works could be directly applied. While sufficient, the dimensionality of the resulting optimization problem often renders this naive optimization approach intractably expensive to apply for discretization resolutions of practical interest, especially over 2D and 3D spatial domains. 

For this reason, we propose to leverage the unique discretization invariance offered by such problems framed over function spaces by solving the optimization iteratively over varying resolutions. In particular, a natural strategy paralleling the nominal approach is to progress through iteratively more finely resolved mappings of the solution map (i.e., larger spectral truncations), since the optimization problem can be more efficiently solved over the lower-dimensional, coarser meshes. 
Formally, a stage $t$ will be defined over the set $\mathcal{V}_{N_{t};\tau}(a)$, where $N_{t'}\le N_{t}$ for $t' < t$ and $t\in\{1,...,T\}$. We assume here that the stages are all defined over a single, fixed functional basis $\{\varphi_n\}$.


Notably, while the iterative refinement proposed herein is a novel aspect of the functional nature of this robust optimization problem, each individual optimization stage can be treated as a finite-dimensional optimization problem similar to those solved in previous works. 
In such cases, the robust optimization problem $\min_{w\in\mathcal{W}} \max_{\widehat{c}\in\mathcal{U}(x)} f(w, \widehat{c})$ is solved by first rewriting the objective as $\min_{w\in\mathcal{W}} \phi(w)$, where $\phi(w) := \max_{\widehat{c}\in\mathcal{U}(x)} f(w, \widehat{c})$ \cite{patel2023conformal,patel2024conformallinear}. From here, we can solve $w^*$ with an iterative, gradient-based strategy on $\phi$. Notably, however, $\phi$ only has a well-defined gradient at $w$ if $f(w, \cdot)$ has a \textit{unique} maximizer over $c$. While some previous works have worked in settings where such was the case, others weakened this assumption and instead only assumed access to a subgradient instead, resulting in a final optimization algorithm as follows: $w^{(k+1)}\gets w^{(k)} - \eta \partial_{w} \phi(w^{(k)})$, where $\partial_{w} \phi(w^{(k)}) = \partial_{w} f(w^{(k)}, c^{*}(w^{(k)}))$ \cite{patel2024conformallinear}. 

Given the general setting of interest here, we similarly consider a subgradient-based optimization scheme to avoid the need for the overly restrictive assumption of there existing a unique maximizer. We, therefore, perform the optimization by iteratively updating $w$ with a subgradient $\partial_{w} \phi_{t}(w) = \partial_{w} J[w, v^{(*;N_{t})}(w)]$ where $\phi_{t}(w) := \max_{\widehat{v}\in\mathcal{V}_{N_{t};\tau}(a)} J[w, \widehat{v}]$ and $v^{(*;N_{t})}(w) := \arg\max_{\widehat{v}\in\mathcal{V}_{N_{t};\tau}(a)} J[w, \widehat{v}]$ for stage $t$. The optimum $w_{t}^{*}$ can then be used to initialize the iterative optimization of $w_{t+1}$. The full algorithm is given in \Cref{alg:multistage_rfpto}.

\begin{algorithm}
\caption{Multi-Stage Robust Predict-Then-Optimize}
\label{alg:multistage_rfpto}
\begin{algorithmic}[1]
\Require Conformal sets $\{\mathcal{V}_{N_{t};\tau}(a)\}$, Step sizes $\{\eta_t\}$, Max steps $\{K_t\}$
\vspace{0.25em}
\State $t \gets 1$, $w^{(0)} \in \mathcal{W}$
\While{$t \le T$} 
    \State \textbf{if} $t=1$ \textbf{then} $w^{(0)}_{t} \gets w^{(0)}$
    \textbf{else} $w^{(0)}_{t} \gets w^{*}_{t-1}$
    \For{$k \in\{1,\ldots K_{t}\}$}
        \State $v^{(*;N_{t})}(w^{(k)}_{t}) \gets \arg\max_{\widehat{v}\in\mathcal{V}_{N_{t};\tau}(a)} J[w^{(k)}_{t}, \widehat{v}]$
        \State $w^{(k+1)}_{t} \gets \Pi_{\mathcal{W}} (w^{(k)}_{t} - \eta_t \partial_{w} J[w^{(k)}_{t}, v^{(*;N_{t})}(w^{(k)}_{t})])$
    \EndFor
\EndWhile
\State \textbf{return} $w^\star \gets w_{T}^{*}$
\end{algorithmic}
\end{algorithm}


\subsubsection{Multi-Stage Optimization Convergence Analysis}\label{section:multi_stage_opt_conv}
We now study the convergence properties of this strategy. The computational cost of each step of \Cref{alg:multistage_rfpto} is dominated by the computation of $v^{(*;N_{t})}(w)$ and is where the cost-savings come from in leveraging coarsened spatial domains. Intuitively, if $J[w^{(t)}, \cdot]$ is sufficiently smooth in $\widehat{v}$, the solution $w_{t-1}^{*}$ will be close to $w_{t}^{*}$, thus making this latter, more expensive optimization problem more efficiently solvable via this proposed multi-stage approach than the naive approach of directly solving a finely resolved problem formulation (i.e., directly solving over $N_T$). If the objective smoothness ensures the proximity of $w_{t-1}^{*}$ with $w_{t}^{*}$ for each pair of stages, the cost reduces to just the cost of solving the first stage, from which it follows that each stage thereafter is a no-op. In the best case, the optimization is fully completed in the coarsest resolution, with most problems lying along this efficiency spectrum as dictated by their smoothness. 

To characterize this phenomenon, we bound the cost for the proposed multi-stage approach below. To do so, we suppose that each stage is to be solved to a fixed schedule of suboptimalities $\{\varepsilon_t\}$ and that we wish to characterize the cost under an optimal optimization scheduler, that is with $K^{*}_{t} := \argmin_k \{ k\in\mathbb{N} : \phi_{t}(w_{t}^{(k)}) - \phi_{t}(w_{t}^{*}) \le \varepsilon_{t}\}$ optimally chosen.
Note that, in practice, these quantities $\{K^{*}_t\}$ are generally unknown or known only loosely, as they rely on function properties that can only be roughly specified. For this reason, the analysis that follows is \textit{not} intended to have immediate algorithmic consequences but rather provide qualitative insights on how properties of the problem manifest in the advantages of the multi-stage approach, suggesting when it would be most effective.

To ease the notation, we consider a fixed input parameter $a$ in this section, which we then drop from the explicit notation, meaning decisions are denoted $w_{t}^{*} := \argmin_{w\in\mathcal{W}} \phi_{t}(w)$. The overall algorithmic cost is then given by the solution time over the intermediate stages:
\begin{equation}\label{eqn:multi_stage_cost}
    \mathcal{E}(\{N_{t}\}) := \sum_{t=1}^{T} \mathcal{E}^{(t)}\left(w_{t-1}^{*}, \mathcal{V}_{N_{t};\tau}(a), \varepsilon_{t}\right)
    \ 
    \mathrm{where}
    \ \mathcal{E}^{(t)}(w_{\mathrm{init}}, \mathcal{V}_{N_{t};\tau}(a), \varepsilon_{t})
    := C_{N_{t}} K^{*}_{t},
\end{equation}
where 
$C_{N_{t}}$ is the cost of computing a \textit{single} gradient step in this $N_t$-truncated robust optimization problem. As discussed earlier, $C_{N_{t}}$ should be monotonically increasing in $t$ for an appropriately constructed sequence of spectral truncations.

With this framing, we establish the bound in \Cref{lemma:multi_stage_opt}. We see that the cost established in the final expression increases with $L_{t}$, where $\phi_{t}(w)$ is $L_t$-smooth. Recalling that ``smoother'' functions correspond to \textit{smaller} values for $L_{t}$, this formalizes the intuition discussed above, where multi-stage optimization is more beneficial when the optima varies smoothly over optima. The full proof is deferred to \Cref{section:multi_stage_cost_pf}.

\begin{assumption}[Objective Regularity] \label{assumption:obj_regularity}
    For each truncation level $N_t$, let
    \[
    \phi_t(w) := \max_{v \in \mathcal{V}_{N_t;\tau}(a)} J[w,v],
    \]
    where $\mathcal{V}_{N_t;\tau}(a)$ is as defined in \Cref{eqn:finite_pred_region}. We assume that $\phi_t$ is $\mu$-strongly convex and $L_t$-smooth on $W$, with constants $L_t \ge \mu > 0$.
\end{assumption}

\begin{assumption}[Adversarial Set Structure] \label{assumption:set_structure}
    The sequence of truncation levels $\{N_t\}_{t=1}^T$ satisfies $N_{t-1} \le N_t$,
    and the corresponding adversarial sets are nested:
    \[
    \mathcal{V}_{N_{t-1};\tau}(a) \subset \mathcal{V}_{N_t;\tau}(a)
    \quad \text{for all } t \ge 2.
    \]
\end{assumption}

\begin{lemma}\label{lemma:multi_stage_opt}
    Suppose that $\{N_t\}_{t=1}^{T}$ is a sequence of truncation points such that $N_{t-1}\le N_{t}$. Let $\{(A^{(i)}, U^{(i)})\}, (A', U'),B,\mathcal{G}$, and $\tau\in\{1,...,s\}$ be as defined in \Cref{thm:func_cov}, with $\mathcal{D}^{(t)}_{C}$ defined with respect to $\Pi_{N_{t}}$
    for each $N_{t}$ and $\widehat{q}_{N_{t};\tau}$ defined over $\mathcal{D}^{(t)}_{C}$ for a coverage level $1-\alpha$. Suppose the resulting finite-dimensional predictors $\{\mathcal{V}_{N_t;\tau}(a)\}$ satisfy \Cref{assumption:obj_regularity,assumption:set_structure}. 
    Then, if iterates $w_{t}^{(k)}$ are obtained with \Cref{alg:multistage_rfpto} such that $\eta_t = 1/L_t$ for an objective $J[w,u]$ satisfying \Cref{assumption:obj_smooth} and $\{\mathcal{E}^{(t)}\}$ are defined as in \Cref{eqn:multi_stage_cost}, we have
    \begin{equation}\label{eqn:multi_stage_bd}
        \mathcal{E}(\{N_{t}\}_{t=1}^{T}) \le
        C_{N_1} \left\lceil \frac{L_{1}}{2\mu} \log\left(\frac{L_{1} \| w_{1}^{*} - w^{(0)} \|^{2}}{2\varepsilon_1}\right) \right\rceil +
        \sum_{t=2}^{T} C_{N_{t}} \left\lceil \frac{L_{t}}{2\mu} \log\left(\frac{L_{t} L B^{1/2} N_{t-1}^{-\tau}}{\mu\varepsilon_t}\right) \right\rceil
    \end{equation}
\end{lemma}

\section{Related Work}\label{section:related_works}
The work presented herein has focused on leveraging conformal prediction over infinite-dimensional function spaces to produce robust decisions with probabilistic guarantees. While past works have yet to study this intersection of function-space uncertainty quantification and decision-making, the two have separately been studied in the literature, as we highlight below.

\subsection{Conformal Prediction over Function Spaces}
We now discuss the existing work that has studied conformal prediction over neural operators. The only works the authors are aware of in this vein are \cite{ma2024calibrated,gopakumar2025calibrated}. In \cite{ma2024calibrated}, the authors focused on the setting where the calibration dataset consists of samples observed at $S$ different spatial resolutions, $\mathcal{D_C} = \cup_{i=1}^{S} \mathcal{D}_{C}^{(h_i)}$. From here, they proceeded in the standard approach described in \Cref{section:bg_cp} using a normalized $\mathcal{L}^2$ score evaluated with a \textit{fixed} mesh discretization across all samples. While this approach does retain coverage guarantees if the mesh observed at test time is a subset of that observed in calibration, it fails to provide any coverage guarantees on the underlying function. Notably, \cite{gopakumar2025calibrated} extended this work to enable discretization-agnostic functional coverage with a data-free approach that measures the operator difference between the learned surrogate and a trusted numerical solver. In certain settings, such as that considered herein, however, the operator of interest is not known, 
rendering this approach unusable.

\subsection{Shape Optimization}\label{section:shape_opt}
Shape optimization is a mature subfield of optimal control. We provide a brief overview here but direct readers to great introductions in \cite{allaire2021shape,delfour2011shapes}. Most generally, shape optimization seeks an optimal design $\Omega^{*}$ as: 
\begin{equation}\label{eqn:shape_opt}
\begin{aligned}
\Omega^{*} := \argmin_{\Omega\in\mathcal{U}_{\mathrm{ad}}} J(\Omega) = \int_{\Omega} j(u_{\Omega}) \\
\textrm{s.t.} \quad G_{i}(\Omega) = 0 
\quad H_{j}(\Omega) \le 0,
\end{aligned}
\end{equation}
where $\mathcal{U}_{\mathrm{ad}}$ is the space of admissible designs, $u_{\Omega}$ is the solution of an associated PDE defined over the domain $\Omega$, $i\in[1,p]$ indexes equality constraints, and $j\in[1,q]$ indexes inequality constraints. For instance, in aerospace design, computational fluid dynamic (CFD) simulations will typically be run for candidate designs $\Omega$ to find the vector field $u_{\Omega}$ describing the fluid flow around the wing, after which evaluation can be done via summarization by $J(u_{\Omega})$ into a scalar quantity, such as drag.

Shape optimization, therefore, can be seen as a frequently encountered special case of the design optimization paradigm considered herein. However, in this work, we specifically focused on scenarios where the design optimization is \textit{strictly} downstream of the operator prediction, by which suboptimality guarantees of the form discussed in \Cref{section:robust_fpo} could be provided. In shape optimization, on the other hand, operator predictions are made iteratively as the design is refined, rendering this predict-then-optimize formalism not directly extensible. Future work, however, should explore this direction of robust iterative prediction.


\section{Experiments}\label{section:experiments}
We now wish to demonstrate the improvement in leveraging the described framework of uncertainty quantification for spectral operators and its use in various functional predict-then-optimize tasks. We empirically validate the coverage claims of \Cref{section:op_calibration} in \Cref{section:coverage_guar_exp}.
We then showcase the improvements with robustness over nominal predictions in resource collection problems in \Cref{section:robust_collection} and for a robust quantum communication task in \Cref{section:state_disc_exp}. Code to recreate the experiment results is provided in \url{https://github.com/yashpatel5400/fpo/}.

Throughout these experiments, inputs were sampled from a Gaussian Random Field with parameters $(\alpha,\beta,\rho)$, which can be sampled as
\begin{equation}\label{eqn:sample_grf}
    a := \sum_{|n|_{\infty}\le N/2} \left(Z_{n} \alpha^{1/2} (4\pi^2 \| n \|^2_{2} + \beta)^{-\rho / 2}\right) e^{2 \pi in\cdot x}
    \quad\text{ where }Z_{n}\sim\mathcal{N}(0, 1),
\end{equation}
where we assume a discretization resolution of $N\times N$ for the fields being studied. The data are generated in full resolution and their spectra then truncated to $N$ in $\mathcal{D}$. Such $\mathcal{D}$ is used to train a spectral neural operator $\mathcal{G}$, whose architecture details we defer to \Cref{section:exp_details}.


\subsection{Truncated Calibration}\label{section:coverage_guar_exp}
We first study the empirical calibration of the method proposed in \Cref{section:op_calibration} for quantum wavefunctions in \Cref{section:wavefunc_coverage}; such coverage is of interest for subsequent use in the robust quantum state discrimination task studied in \Cref{section:state_disc_exp}. We then demonstrate the generality of the proposed spectral conformal coverage methodology by additionally applying it to the Poisson and heat equations in \Cref{section:addn_coverage}.

\subsubsection{Quantum Wavefunctions}\label{section:wavefunc_coverage}
Quantum cryptography is the study of leveraging quantum states for the secure transmission of information. One critical use case of this is quantum key distribution, in which two parties seek to establish a common shared key for the sending and deciphering of encrypted messages \cite{scarani2009security,xu2020secure}. In classical key distribution, such a shared secret could, for instance, be established using a Diffie-Hellman key exchange. In the quantum communication analog, however, there is an additional complication that 
the quantum state can only be deciphered probabilistically by the receiver \cite{bedington2017progress,zhang2024continuous,jain2022practical}. For this reason, quantum transmission protocols must additionally consider the interaction dynamics between the prepared quantum state on the sender side and the transmission channel to ensure the receiver has a high probability of recovering the prepared state \cite{vasani2024embracing,sidhu2021advances,gisin2007quantum}.


To design this protocol, we first wish to study the time evolution dynamics of the quantum state as it is transmitted over the channel, given by a time evolution operator $\mathcal{G}$. Note that, throughout this presentation, we choose to represent quantum states as $\psi$ instead of invoking the common bra-ket notation of quantum mechanics to be consistent with the remaining presentation. We, therefore, wish to estimate the time evolution operator $\mathcal{G} : \psi_{b}\to\psi_{a}$, which models the state evolution from \textit{before} transmission $\psi_{b}$ to \textit{after} $\psi_{a}$. 
In this setting, we switch from the $(a, u)$ notation used heretofore to represent input-output pairs to $(\psi_{b},\psi_{a})$. In this experiment, we take $s = 2$ and $d = 2$, i.e., states exist $\psi\in\mathcal{H}^{2}(\mathbb{T}^{2})$. 

In typical quantum settings, this operator is unitary and is explicitly given by $\mathcal{U} = \exp(-(\mathrm{i}T/ \hbar) \widehat{H})$, assuming the state is evolved over $T$ time steps with a time-independent Hamiltonian operator $\widehat{H}$. The Hamiltonians of interest take the form $\widehat{H} = -\frac{\hbar^2}{2m} \Delta + V$ for different potentials $V$; the collection of potentials considered in the experiments are listed in \Cref{table:potentials} and are classically studied Hamiltonians for common quantum waveguides \cite{kitagawa2015quantum,collado2024harmonic}. We assume natural units are taken such that $\hbar = 2m = 1$. Additional details of hyperparameters selected for the experiments are given in \Cref{section:exp_setup_details}.

\begin{table}[h!]
\caption{\label{table:potentials} Optical–waveguide potentials used to define Hamiltonians across experiments.}
\centering
\renewcommand{\arraystretch}{1.2}
\begin{tabular}{|c|c|}
\hline
\textbf{Structure} & \textbf{Potential} \\
\hline
Step-index fiber & \(\displaystyle V(x) = \begin{cases} 0, & |x|\le a \\ V_0, & |x|>a \end{cases}\) \\ \hline
GRIN fiber & \(\displaystyle V(x) = -C x^{2}\) \\ \hline
\end{tabular}
\end{table}

As discussed in \Cref{section:op_calibration}, we seek coverage on $\psi_{a}$ despite only observing truncated wavefunctions in the calibration set. To make use of \Cref{thm:func_cov}, we first provide the requisite bound $B$ on $\| \psi^{(a)} \|^{2}_{s}$ in the statement below, whose explicit derivation is provided in \Cref{section:func_bounds}. 

\begin{lemma}\label{lemma:wavefunc_bound}
    Let $V\in\mathcal{C}^{\infty}(\mathbb{T}^{d})$. Let $\widehat{H} = -\Delta + V$ and $\mathcal{U} := \exp(-\mathrm{iT}\widehat{H})$ for $T\ge 0$. Let $s\in[0,2]$. For any $\psi\in\mathcal{H}^{s}(\mathbb{T}^{d})$ such that $\| \psi \|_{\mathcal{L}^{2}} = 1$,
    \begin{equation}
        \| \mathcal{U} \psi \|^{2}_{s}
        \le (\sqrt{2} (\max\{1 + 2 \| V \|^{2}_{\infty}, 2 \}))^{s} \| \psi \|^{2}_{s}
    \end{equation}
\end{lemma}


To test calibration, 300 i.i.d.\ training data points were generated with an additional 150 points used for calibration and 150 for testing coverage. Coverage was computed by assessing whether the untruncated score (i.e., LHS of \Cref{eqn:coverage_guarantee}) was bounded by the quantile after correction. Again, the model was trained with data only of a truncated spectrum $N_{\max} \le N/2$ but coverage was sought over the \textit{full} $N/2 \times N/2$ spectrum. In particular, three distinct neural operators were trained, each on a different truncation of the full resolution of the data, namely for $N_{\max}\in\{16,24,32\}$ for the full resolution $N = 64$. The results are shown in \Cref{fig:step_index_calibration} and \Cref{fig:grin_calibration}, which respectively are the coverages for the step-index and GRIN fibers.

\begin{figure}
  \centering 
  \includegraphics[width=0.97\textwidth]{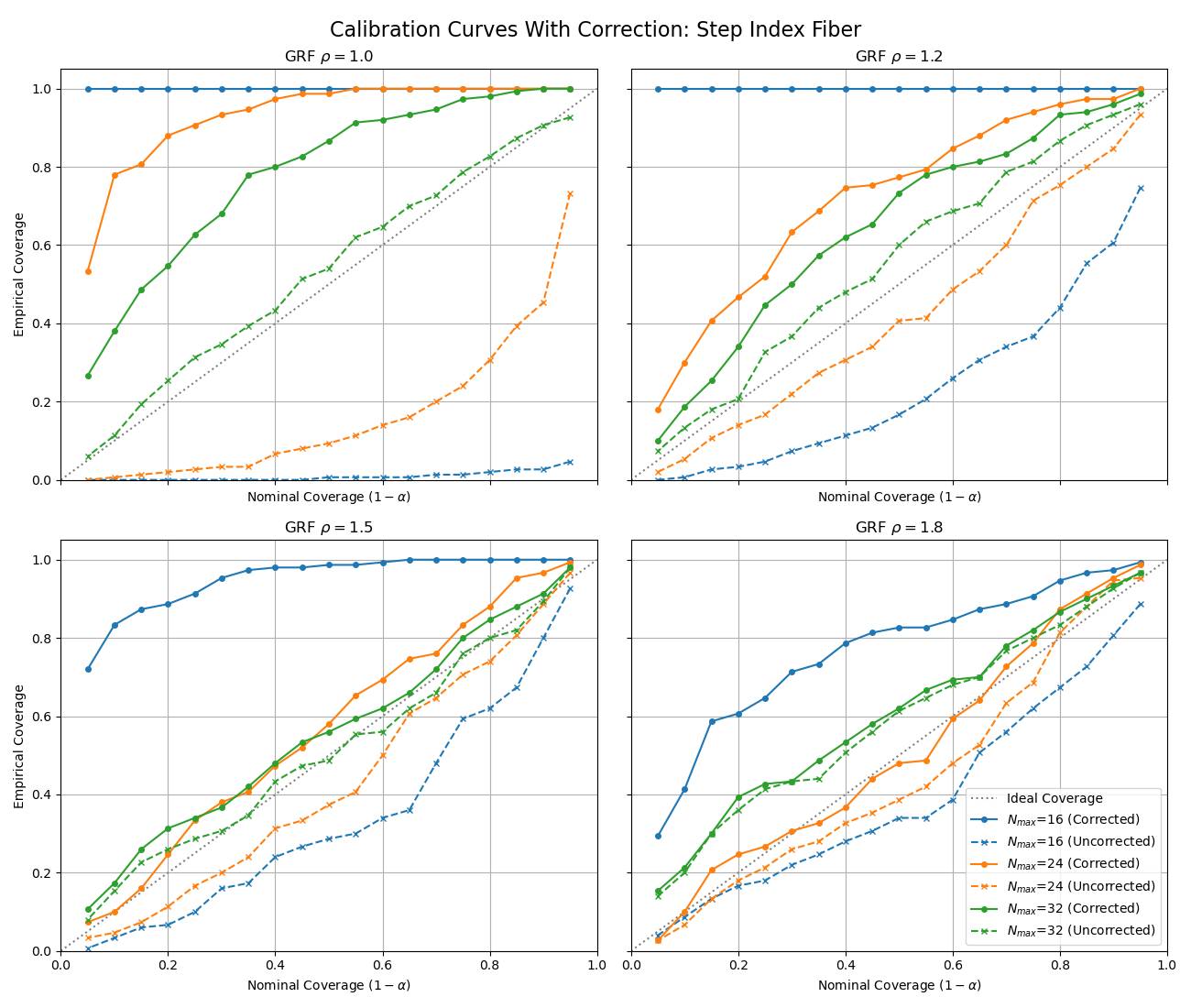}
  \caption{Calibration curves with (solid) and without (dashed) spectral correction factors across data of varying GRF smoothness parameters for the step-index fiber Hamiltonian. Calibration is performed for three models that act on data at different spectral truncations of the full resolution $N = 64$ data.}
  \label{fig:step_index_calibration}
\end{figure}

\begin{figure}
  \centering 
  \includegraphics[width=0.97\textwidth]{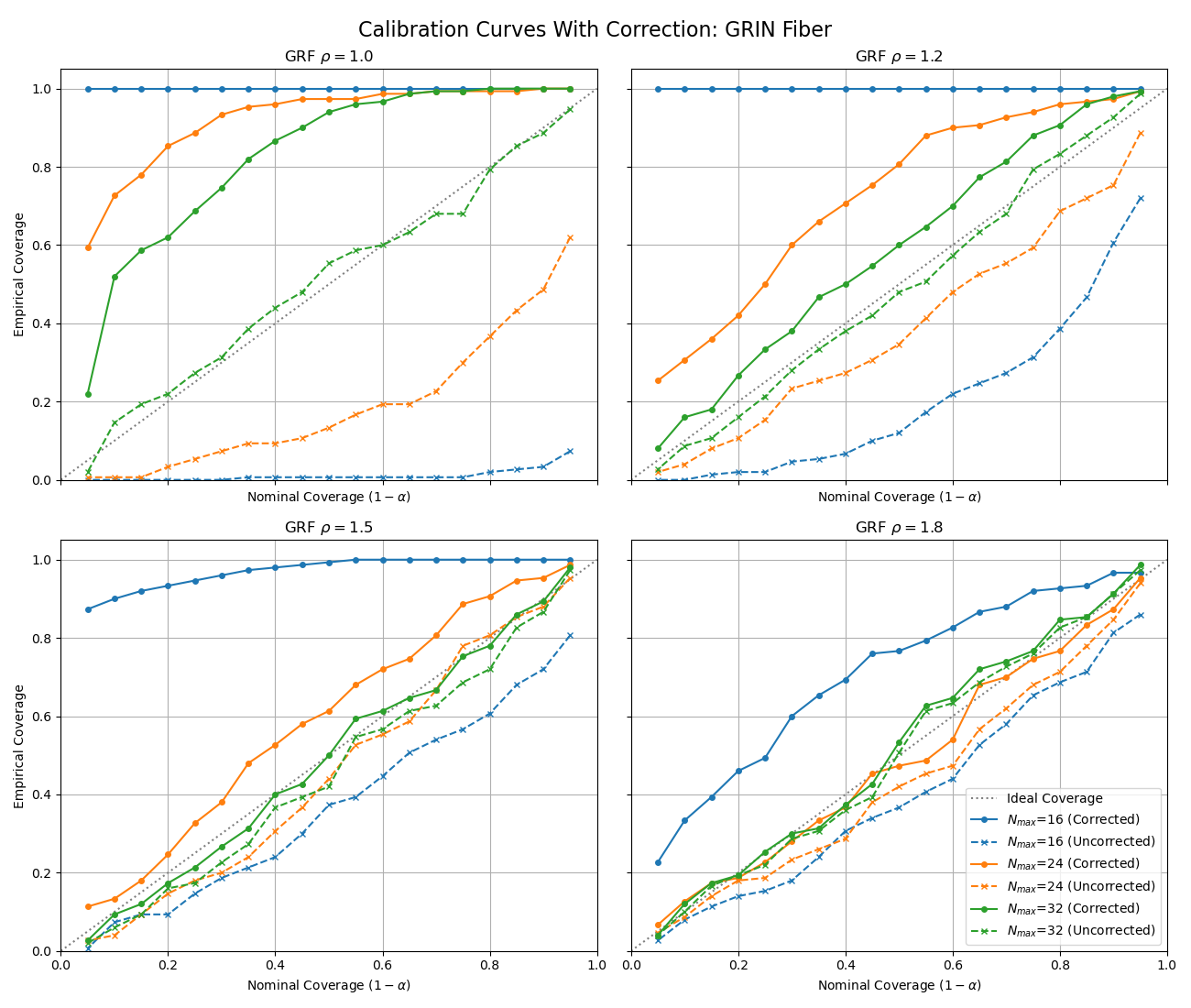}
  \caption{Calibration curves with (solid) and without (dashed) spectral correction factors across data of varying GRF smoothness parameters for the GRIN fiber Hamiltonian. Calibration is performed for three models that act on data at different spectral truncations of the full resolution $N = 64$ data.}
  \label{fig:grin_calibration}
\end{figure}

There are several noteworthy features of these figures. We first focus on the insights to glean from studying a fixed GRF, such as the case of $\rho=1.2$. First, we notice that, without correction, the models fail to achieve coverage and are consistently underneath the desired calibration curve. After correction, however, all curves achieve coverage at the potential expense of being conservative and overcovering in certain cases. We additionally notice that, as more of the complete spectrum is made visible for training, the uncorrected curves approach the optimal calibration. This is expected as the correction margin is only required to account for unseen modes in the data. In this vein, as the $\rho$ GRF parameter increases, the resulting fields become smoother, thus decreasing the magnitude of the higher, unseen modes in such cases. We, therefore, see that the uncorrected curves get closer to the desired calibration with increasing $\rho$. We additionally notice that, as the correction term is proportional to $ \| \psi \|^{2}_{s}$ (see \Cref{lemma:wavefunc_bound}), the increasing $\rho$ too is reflected in the correction margin, where the correction desirably decreases as the functions get smoother. This can be seen from the decreasing distance between the uncorrected and corrected calibration curves as $\rho$ increases. Similarly, we notice that the correction margins decrease with increasing $N$. This reflects that, as more of the spectrum is observed, there are less unobserved modes that need to be accounted for. Mathematically, this effect emerges from the decay factor $N^{-2\tau}$ of the margin. 



\subsubsection{Spectral Calibration Across PDEs}\label{section:addn_coverage}

\begin{figure}
  \centering 
  \includegraphics[width=0.97\textwidth]{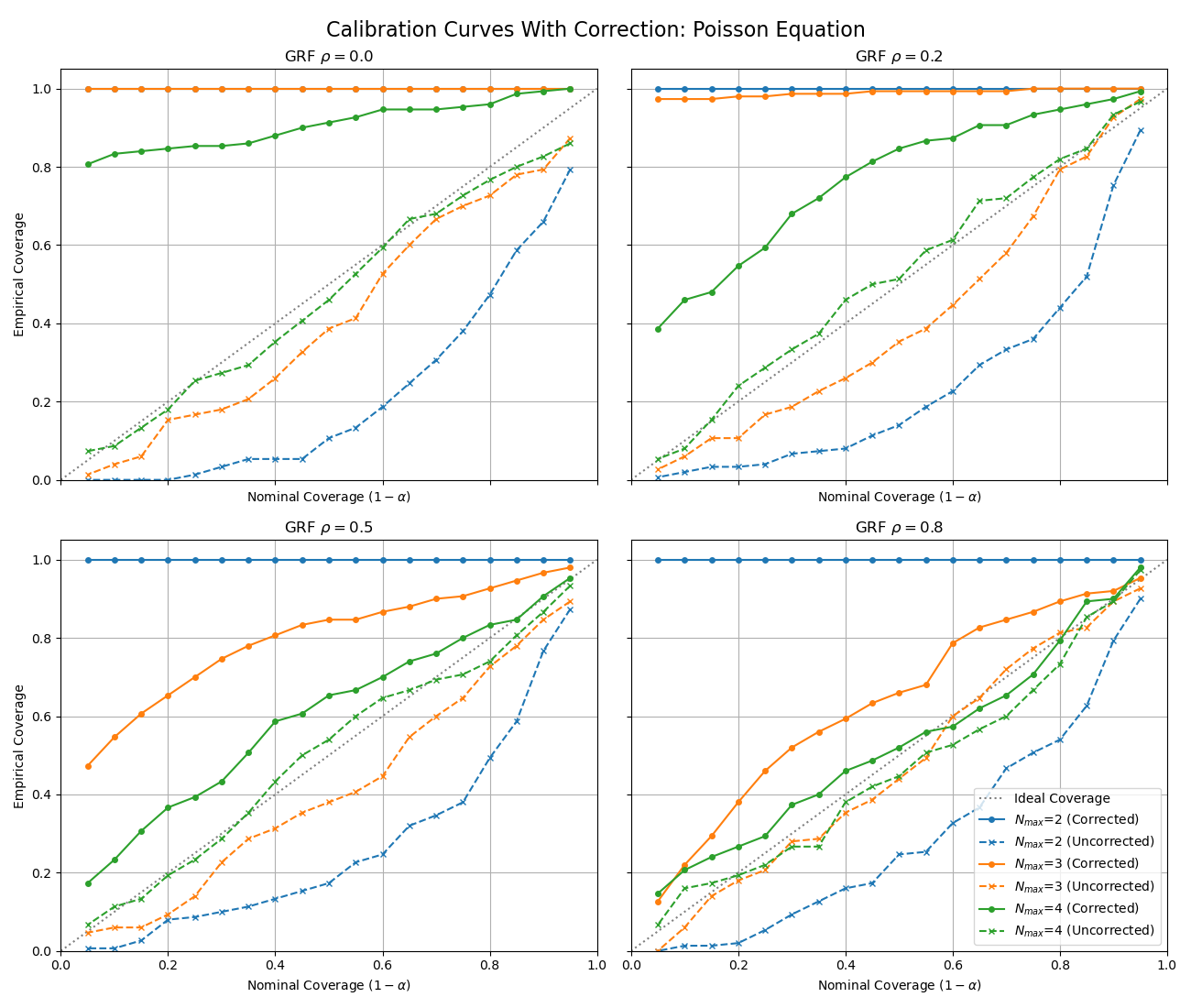}
  \caption{Calibration curves with spectral correction factors across data of varying GRF smoothness parameters for the Poisson equation. Calibration is performed for three models that act on data at different spectral truncations of the full resolution $N = 64$ data.}
  \label{fig:calibration_poisson}
\end{figure}

As discussed, the calibration procedure presented herein is generally applicable to any PDEs for which a non-trivial bound in the manner of \Cref{lemma:wavefunc_bound} can be expressed. We, therefore, here demonstrate such calibration across two distinct PDE settings, namely the 2D Poisson and 2D heat equations. The 2D Poisson equation is given by
    $\Delta u(x) = f(x)$,
where $u : \mathbb{T}^{2}\rightarrow\mathbb{R}$ is the scalar field of interest and $f : \mathbb{T}^{2}\rightarrow\mathbb{R}$ the source field. In this setting, we wish to learn the solution operator $\mathcal{G} : f\to u$. The heat equation is given by
    $\partial_{t} u = \tau \Delta u,$
where now $u : \mathbb{T}^{2}\times[0,T]\rightarrow\mathbb{R}$ is a spatiotemporal field. We, therefore, consider a fixed time point $T$, from which the solution operator of interest becomes purely spatial, mapping between $u(\cdot, 0)$ and $u(\cdot, T)$. For each PDE, a dataset of 300 i.i.d.\ training data points, 150 calibration points, and 150 test points was generated. 


\begin{figure}
  \centering 
  \includegraphics[width=0.97\textwidth]{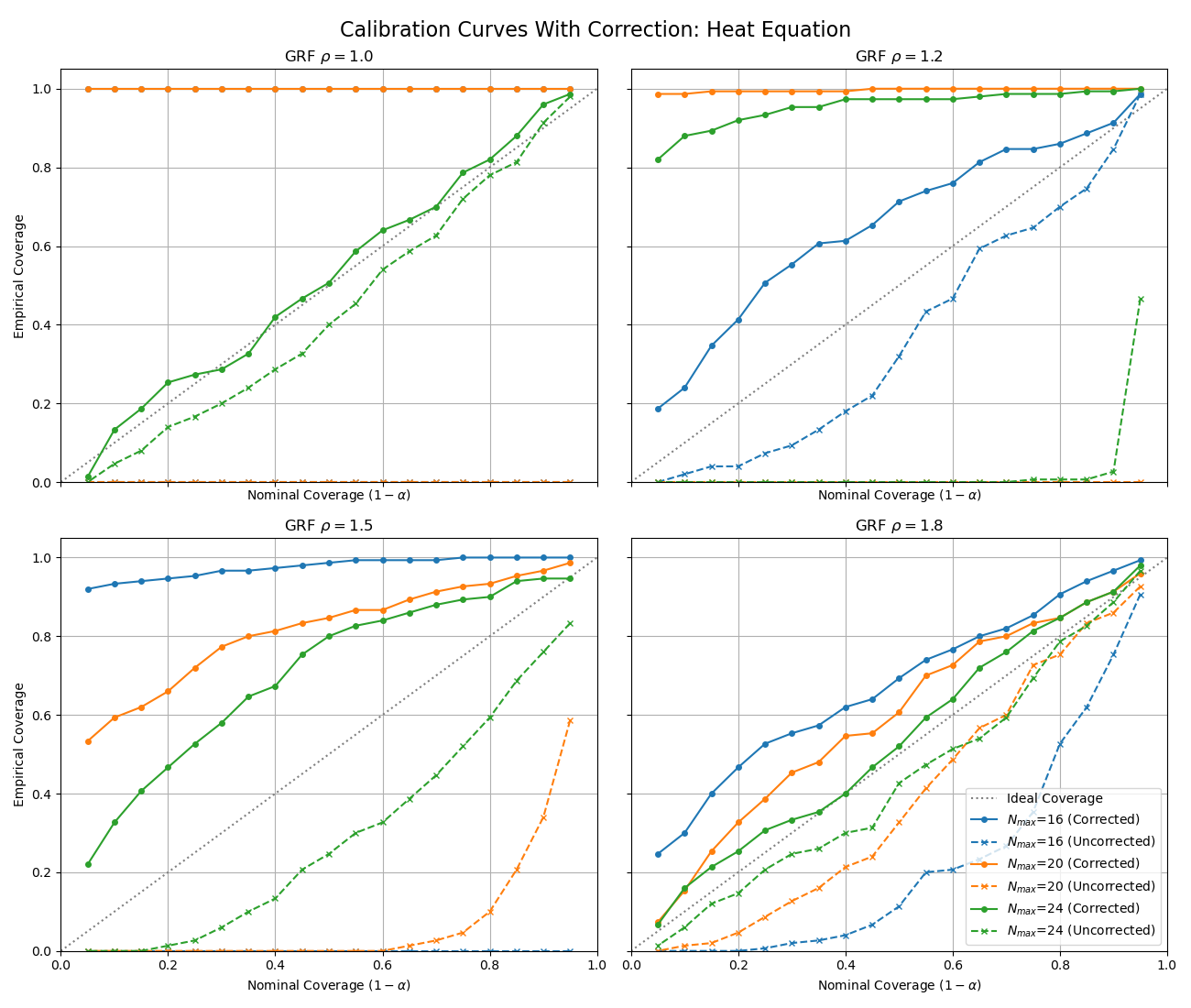}
  \caption{Calibration curves with spectral correction factors across data of varying GRF smoothness parameters for the heat equation. Calibration is performed for three models that act on data at different spectral truncations of the full resolution $N = 64$ data.}
  \label{fig:calibration_heat}
\end{figure}

We once again provide bounds by which we can invoke the results of \Cref{thm:func_cov}. The statement of the Poisson solution bound follows from \Cref{corollary:elliptic_coverage}. We here, however, require the explicit value of $C$ and, thus, provide a full derivation of this bound in \Cref{section:func_bounds}.

\begin{remark}\label{remark:poisson_bound}
    For any zero-mean $u, f\in\mathcal{H}^{s}(\mathbb{T}^{d})$, i.e., $\int_{\mathbb{T}^{d}} u(x)\, dx = \int_{\mathbb{T}^{d}} f(x)\, dx = 0$, for which $\Delta u = f$, $\| u \|^{2}_{s} \le 4 \| f \|^{2}_{s-2}$.
\end{remark}

\begin{remark}\label{remark:heat_equation_bound}
    For $u(\cdot, 0)\in\mathcal{H}^{s}(\mathbb{T}^{d})$ such that $\partial_{t} u = \tau \Delta u$ and $T > 0$, $\| u(\cdot, T) \|^{2}_{s} \le \| u(\cdot, 0) \|^{2}_{s}$.
\end{remark}

Using these bounds, we obtain the results shown in \Cref{fig:calibration_poisson} and \Cref{fig:calibration_heat}, where we consider calibration again under different observed spectral truncations across varying smoothness parameters of the GRF. As seen in \Cref{section:wavefunc_coverage}, while all truncations of the spectrum achieve coverage after correction, those with shorter spectral truncations are more sensitive to changes to the correction factor and, hence, result in great conservatism. Additionally, we see as $\rho$ and $N$ increase, all curves converge to being calibrated; this is intuitively expected again as the resulting functions are increasingly smooth with increasing $\rho$ and, hence, the need for a correction margin decreases, since the higher order modes decay to 0.

\begin{table*}[t]
\centering
\caption{Robust vs nominal collection improvement across PDEs, varying GRF smoothness parameters $\rho$, and spectral truncations $N$. Results are presented over 200 i.i.d. test samples, with the columns indicating the average suboptimality $\Delta^{*}(a, u)$ and the $p$-value of the paired t-test of $\Delta^{*}(a, u) > 0$. Bolded $\Delta$ indicate significance at the 0.05 level.}
\label{tab:rob_collection}
\resizebox{\textwidth}{!}{%
\begin{tabular}{c|cc|cc|c|cc|cc}
\toprule
\multicolumn{5}{c|}{\textbf{Poisson}} & \multicolumn{5}{c}{\textbf{Heat Equation}} \\
\cmidrule(lr){1-5}\cmidrule(lr){6-10}
 {$\rho$} & \multicolumn{2}{c|}{$N=6$} & \multicolumn{2}{c|}{$N=8$} & {$\rho$} & \multicolumn{2}{c|}{$N_{\text{out}}=32$} & \multicolumn{2}{c}{$N_{\text{out}}=40$} \\
\cmidrule(lr){2-3}\cmidrule(lr){4-5}\cmidrule(lr){7-8}\cmidrule(lr){9-10}
 & {$\Delta$} & {$H:\Delta>0$ (p)} & {$\Delta$} & {$H:\Delta>0$ (p)} &  & {$\Delta$} & {$H:\Delta>0$ (p)} & {$\Delta$} & {$H:\Delta>0$ (p)} \\
\midrule
0.25 & \textbf{5.779} & 0.0433 & -8.633 & 0.96 & 1.25 & 0.118 & 0.181 & -0.207 & 0.896 \\
0.50 & \textbf{9.853} & $1.14\times10^{-10}$ & \textbf{7.202} & $1.13 \times 10^{-5}$ & 1.50 & 0.071 & 0.179 & \textbf{0.118} & 0.00822 \\
0.75 & \textbf{3.869} & $4.19 \times 10^{-10}$ & \textbf{2.103} & $1.11 \times 10^{-4}$ & 1.75 & \textbf{0.043} & 0.0393 & 0.028 & 0.298 \\
\bottomrule
\end{tabular}%
}
\end{table*}

\subsection{Robust Collection Problems}\label{section:robust_collection}
We now seek to demonstrate the utility of robust functional decision-making over the naive, nominal approach. In the following discussion, we focus on functionals with a particular structure, namely those in a subdomain of spatial optimization known as continuous space maximal coverage problems, where we wish to maximally accrue a ``resource'' with the placement of collection facilities. A full review of such problems is available at \cite{wei2015continuous}. The placement of such facilities is typically predicated on the geospatially predicted demand of a resource, such as that of electric vehicles \cite{huang2016design}, natural disaster relief \cite{yang2020continuous}, or Wi-Fi \cite{fajardo2018placing}. Formally, such a problem can be specified with the following functional
\begin{equation}\label{eqn:collection_prob}
w^{*}(u) 
:= \argmax_{w\in\mathcal{W}} \int_{\Omega} \left(\sum_{n\in\mathbb{Z}^{d}} u_n \varphi_{n}(x)\right) k_{w}(x) dx,
\end{equation}
where $w := \{w_i\}_{i=1}^{K}$ for $w_i\in\Omega$ are a set of $K$ ``collection locations'' and $k_{w}(x) : \Omega\to\mathbb{R}$ is a known mapping from how such locations are defined to a ``collection field.'' In the simplest case, for instance, $K = 1$ and $k_{w}(x) = \mathbbm{1}[x\in\mathcal{B}_{r}(w)]$, in which case this objective is simply
\begin{equation}\label{eqn:simple_collection_prob}
J[w, u] := \int_{\mathcal{B}_{r}(w)} \left(\sum_{n\in\mathbb{Z}^{d}} u_n \varphi_{n}(x)\right) dx.
\end{equation}
In fact, as shown in \Cref{section:lin_robust}, the robust counterpart to this simple case of a \textit{single} collector in \Cref{eqn:simple_collection_prob} displays no separation from its nominal counterpart, rendering it uninteresting for further study. Separations, however, exist if we consider \textit{multiple} collection locations. We, therefore, now empirically demonstrate the utility of the proposed robust approach for the multi-location collection problem.

We now study this problem over solutions of both the Poisson and heat equations, as respectively set up in \Cref{section:addn_coverage}. Collection problems over fields governed by the Poisson equation naturally arise where a resource diffuses dynamically based on environmental conditions, in which steady states are governed by the Poisson equation, such as for the detection or remediation of a toxic substance, as studied in \cite{carnevale2021system,vianna2019set,boubrima2017optimal,kessler1998detecting}. The optimal ``collection'' of heat similarly arises in optimal heat sink placement.

We now compare the performances of the robust and nominal solutions using the predictors and calibrations procured in \Cref{section:addn_coverage}. In particular, for each $a$ in the test set, we predict the corresponding $\mathcal{G}(a)$ and compute the optimal $w_{\mathrm{nom}}^{*}(a)$ via gradient descent. We then compute the robust solution $w_{\mathrm{rob}}^{*}(a)$ via the multi-stage optimization discussed in \Cref{section:multi_stage_opt}, from which we finally conduct a hypothesis test evaluating each on the \textit{true} underlying field corresponding to $a$, namely testing
\begin{equation*}
    H_{0} : J[w_{\mathrm{nom}}^{*}(a), u] = J[w_{\mathrm{rob}}^{*}(a), u]
    \qquad
    H_{A} : J[w_{\mathrm{rob}}^{*}(a), u] > J[w_{\mathrm{nom}}^{*}(a), u]
\end{equation*}
We conduct these tests across 100 i.i.d.\ test samples using paired t-tests. We again test across several parameter configurations, namely across different spectral truncations and GRF smoothness factors $\rho$ with a \textit{fixed} coverage level of $\alpha=0.1$ for the robust approach. We specifically focus on the larger values of $\rho$ and spectral truncations, where the prediction regions are not overly conservative. The results are presented in \Cref{tab:rob_collection}. From this, we see that the robust solution offers consistent, significant improvements over the nominal counterparts. We additionally visualize the comparisons of the robust vs. nominal collector placements in \Cref{section:rob_collection_viz}, from which it becomes clear how the robust approach hedges against the naive strategy of concentrating collectors around local peaks of the predicted field.


\subsection{Robust State Discrimination}\label{section:state_disc_exp}

We now demonstrate the utility of the robust functional predict-then-optimize framework on a quantum state discrimination task. As discussed in \Cref{section:wavefunc_coverage}, this task is especially of interest for quantum cryptography. Notably, this setup contrasts with that studied in the previous section, as here, the function parameter manifests \textit{indirectly} as a derived parameter in the final problem. We, thus, highlight here how the upstream functional coverage can be propagated to this derived parameter.

Intuitively, the goal for optimal quantum key distribution is for Alice and Bob to maximize the mutual information between their measured states. In this setup, Alice begins by transmitting one state of the finite collection $\{\psi_{k}\}_{k=0}^{M-1}$, which is then measured by Bob \textit{after} it evolves according to the channel transmission dynamics $\mathcal{G}$. Under a common symmetry assumption of $\{\psi_{k}\}_{k=0}^{M-1}$ known as ``geometric uniform symmetry,'' the mutual information objective can be explicitly characterized by the eigenvalues $g$ of the interaction matrix $G := [(\mathcal{G}([\psi_{b}]_{\ell})^{\dagger} (\mathcal{G}([\psi_{b}]_{k}) ]_{\ell, k}$ and a measurement parameter $\phi\in[0,2\pi)^{M}$; we provide the derivation of this in \Cref{section:quantum_disc_problem_setup_extensive}. The objective is, thus, to select the measurement parameter
\begin{equation}
    \phi^*(g) := \argmax_{\phi\in[0,2\pi)^{M}} I_{S;R}(\phi, g)
\end{equation}
Notably, despite it being optimal to model $g$ in the selection of $\phi$, a commonly employed practical strategy is to fix $\phi = 0$ without modeling $g$; this baseline is known as the ``pretty good measurement'' (PGM), which we compare against in the experiments below. 

In reality, the transmission dynamics can seldom be specified precisely by hand, for which reason data-driven estimates $\widehat{\mathcal{G}}$, such as those explored in \Cref{section:wavefunc_coverage}, must be used. In reality, wavefunctions can only be measured as finite-dimensional projections: here, we assume the wavefunction is of a single quantum particle on the domain $\mathcal{H} = \mathcal{H}^{s}(\mathbb{T}^{d})$ for $s= 2$ and $d=2$.
Thus, we ultimately have $\mathcal{D} := \{(\psi_{b}, \Pi_N \psi_{a})\}$, which we then use to train a spectral neural operator to estimate the map $\widehat{\mathcal{G}} : \psi_{b}\to\Pi_N \psi_{a}$.

\subsubsection{Coverage Guarantee Propagation}\label{section:uncertain_prop_quantum}



Naively, the receiver could design their measurement scheme by directly estimating $G$ with their learned evolution operator $\widehat{\mathcal{G}}$, i.e., design against the eigenvalues $\widehat{g}$ of $\widehat{G} := [(\widehat{\mathcal{G}}([\psi_{b}]_{\ell})^{\dagger} (\widehat{\mathcal{G}}([\psi_{b}]_{k}) ]_{\ell, k}$. Notably, however, simply designing against $\widehat{g}$ could result in suboptimal decoding under the true transmission dynamics.

We, thus, highlight how the functional coverage of \Cref{section:wavefunc_coverage} can be propagated downstream to $g$ to frame a robust measurement problem. Using \Cref{thm:func_cov}, a coverage guarantee of the form $\mathcal{P}_{\psi_{b}, \psi_{a}}(\| \mathcal{G}(\psi_{b}) - \psi_{a} \|^{2}_{s-\tau} \le \widehat{q}_{N;\tau}^{*}) \ge 1-\alpha$ can be established using $\mathcal{D}_{C}$. Using such a guarantee, we wish to consequently define the prediction region around $\widehat{G} := [(\mathcal{G}([\psi_{b}]_{\ell})^{\dagger} (\mathcal{G}([\psi_{b}]_{k}) ]_{\ell, k}$ with coverage guarantees on $G := [([\psi_{a}]_{\ell})^{\dagger} ([\psi_{a}]_{k})]_{\ell, k}$. Doing so is possible by bounding the pairwise difference in the matrix elements, which results in $\mathcal{P}_{\psi_{b}, \psi_{a}}(\| G - \widehat{G} \|^{2}_{F} \le 3 M^2 (2\widehat{q}_{N;\tau}^{*} + (\widehat{q}_{N;\tau}^{*})^{2})) \ge 1-\alpha$. This derivation is deferred to \Cref{section:func_prop_der}. With this, we now use the Wielandt-Hoffman theorem to cover the spectra $g$ of matrices in an uncertainty region to connect this to the parametric form of the $I_{S;R}(\phi, g)$ objective.

\begin{theorem}\label{thm:wielandt_hoffman}
    (Wielandt-Hoffman Theorem from \cite{ikramov2009theorems}) 
    Let $A$ and $B$ be normal matrices of order $n$ having the eigenvalues $\alpha_1, \alpha_2, ..., \alpha_n$ and $\beta_1, \beta_2, ..., \beta_n$, respectively. Then, there exists a permutation $\pi$ of the indices $1, 2, ..., n$ such that
    \begin{equation}
        \sum_{i=1}^{n} | \alpha^{(i)} - \beta_{\pi(i)} |^{2} \le \| A - B \|^{2}_{F}
    \end{equation}
\end{theorem}

Notably, when $A$ and $B$ are Hermitian matrices, if $\alpha_{1} \ge...\ge \alpha_{n}$, the minimizing permutation $\pi$ is similarly the descending order of $\beta$, i.e., $\beta_{\pi(1)} \ge...\ge \beta_{\pi(n)}$ \cite{wilkinson1970elementary}. 
Denoting by $\widehat{g}\in\mathbb{R}^{M}$ the vector of \textit{sorted} eigenvalues of $\widehat{G}$, it then follows that $\mathcal{P}_{\psi_{b}, g}(\| g - \widehat{g} \|^{2}_{2}\le 3 M^2 (2\widehat{q}_{N;\tau}^{*} + (\widehat{q}_{N;\tau}^{*})^{2}))\ge 1-\alpha$. We denote this prediction region over the spectra as $\mathcal{C}_{N;\tau}^{(g)}(\psi_{b})$.

\subsubsection{Robust State Discrimination Suboptimality}\label{section:operator_est}
We, therefore, now wish to leverage this parametric coverage and consider a robust formulation of the objective of interest, namely
\begin{equation}\label{eqn:rob_mutual_info}
    \phi_{\mathrm{rob}}^{*}(\psi_{b}, N) := \argmax_{\phi\in[0,2\pi)^{M}} \min_{\widehat{g}\in\mathcal{C}_{N;\tau}^{(g)}(\psi_{b})} I_{S;R}(\phi, \widehat{g})
\end{equation}
We begin by establishing a bound on the resulting suboptimality for this setting:
\begin{equation}\label{eqn:mi_subopt}
    \Delta_{S;R}(\psi_{b}, g, N) := \max_{\phi\in[0,2\pi)^{M}} I_{S;R}(\phi, g) - \max_{\phi\in[0,2\pi)^{M}} \min_{\widehat{g}\in\mathcal{C}_{N;\tau}^{(g)}(\psi_{b})} I_{S;R}(\phi, \widehat{g})
\end{equation}
We do so by demonstrating our objective $\phi$ satisfies the assumptions of \Cref{thm:subopt_gap}. Doing so requires making mild assumptions on the measurement protocol, that is, on $(S;R)$ as formalized in \Cref{assup:min_prob}. Establishing \Cref{thm:subopt_gap} required non-trivial study of the interaction between the spectral uncertainty and the objective to establish a Lipschitz constant. The full proof of this statement is deferred to \Cref{section:suboptimalty_gap}.

\begin{lemma}\label{lemma:subopt_gap_quantum}
    Let $\{(A^{(i)}, U^{(i)})\}, (A', U'),B(A), N, \mathcal{D}_{C},\widehat{q}^{*}_{N;\tau},\mathcal{G}$, and $\tau\in\{1,...,s\}$ be as defined in \Cref{thm:func_cov} for $A^{(i)} := \psi^{(i)}_{b}$ and $U^{(i)} := \psi^{(i)}_{a}$ and $\mathcal{C}_{N;\tau}^{*}(a)$ be the resulting margin-padded predictor for a marginal coverage level $1-\alpha$. Let $\Delta_{S;R}(\psi_{b}, g)$ be as defined in \Cref{eqn:mi_subopt} for a measurement protocol $(S;R)$ satisfying \Cref{assup:min_prob}. Then, for a constant $L < \infty$,
    \begin{equation}
        \mathcal{P}_{\psi'_{b}, g'}\left(0 \le \Delta_{S;R}(\psi'_{b}, g', N) \le 2\sqrt{3} LM \sqrt{2\widehat{q}_{N;\tau}^{*} + (\widehat{q}_{N;\tau}^{*})^{2}} \right) \ge 1-\alpha.
    \end{equation}
\end{lemma}

\subsubsection{Experimental Results}

\begin{table*}[t]
\centering
\caption{Mutual information for PGM, nominal, and robust measurement schemes across varying GRF smoothness parameters $\rho$, GUS state sizes $M$, and spectral truncations $N$. Results are presented over 30 i.i.d. test samples, with the $I(\phi)$ columns indicating the average across samples and $H : I(\phi^{(\mathrm{A})}) > I(\phi^{(\mathrm{B})})$ the p-value of the paired t-test of $I(\phi^{(\mathrm{A})}) > I(\phi^{(\mathrm{B})})$.}
\label{tab:rob_opt_comparison}
\resizebox{\textwidth}{!}{
\begin{tabular}{ccc | ccccc}
\toprule
 {$\rho$} & {M} & {$N$} &
 {$I(\phi^{(\mathrm{PGM})})$} &
 {$I(\phi^{(\mathrm{nom})})$} &
 {$I(\phi^{(\mathrm{rob})})$} &
 {$H : I(\phi^{(\mathrm{rob})}) > I(\phi^{(\mathrm{PGM})})$} &
 {$H : I(\phi^{(\mathrm{rob})}) > I(\phi^{(\mathrm{nom})})$} \\
\midrule

\multicolumn{8}{c}{\textbf{Step-Index Fiber}} \\
\midrule
1.5 & 3 & 32 & 0.2245 & 0.2613 & \textbf{0.2634} & $< 0.001$ & $< 0.001$ \\
1.5 & 3 & 48 & 0.2205 & 0.2532 & \textbf{0.2588} & $< 0.001$ & $< 0.001$ \\
1.5 & 4 & 32 & 0.2144 & 0.2221 & \textbf{0.2222} & $< 0.001$ & 0.268 \\
1.5 & 4 & 48 & 0.2139 & 0.2203 & \textbf{0.2212} & $< 0.001$ & $< 0.001$ \\
1.8 & 3 & 32 & 0.2150 & 0.2508 & \textbf{0.2518} & $< 0.001$ & $< 0.001$ \\
1.8 & 3 & 48 & 0.2205 & 0.2279 & \textbf{0.2532} & $< 0.001$ & $< 0.001$ \\
1.8 & 4 & 32 & 0.2132 & 0.2206 & \textbf{0.2213} & $< 0.001$ & $< 0.001$ \\
1.8 & 4 & 48 & 0.2157 & 0.2170 & \textbf{0.2204} & $< 0.001$ & $< 0.001$ \\

\midrule
\multicolumn{8}{c}{\textbf{GRIN Fiber}} \\
\midrule
1.5 & 3 & 32 & 0.2157 & 0.2507 & \textbf{0.2526} & $< 0.001$ & $< 0.001$ \\
1.5 & 3 & 48 & 0.2159 & 0.2161 & \textbf{0.2476} & $< 0.001$ & $< 0.001$ \\
1.5 & 4 & 32 & 0.2134 & 0.2206 & \textbf{0.2211} & $< 0.001$ & 0.002 \\
1.5 & 4 & 48 & 0.2127 & 0.2156 & \textbf{0.2160} & $< 0.001$ & 0.093 \\
1.8 & 3 & 32 & 0.2160 & 0.2529 & \textbf{0.2536} & $< 0.001$ & $< 0.001$ \\
1.8 & 3 & 48 & 0.2176 & 0.2264 & \textbf{0.2522} & $< 0.001$ & $< 0.001$ \\
1.8 & 4 & 32 & 0.2187 & 0.2264 & \textbf{0.2271} & $< 0.001$ & $< 0.001$ \\
1.8 & 4 & 48 & 0.2095 & 0.2118 & \textbf{0.2149} & $< 0.001$ & $< 0.001$ \\

\bottomrule
\end{tabular}
}
\end{table*}

We now demonstrate the improvement in using the resulting robust formulation over the nominal formulation. We specifically wish to compare the robust formulation against the corresponding nominal approach, where the predicted $\widehat{\mathcal{G}}(\psi_{b})$ is simply assumed to be well-specified with the decision appropriately defined using said prediction. Towards this end, we conduct a hypothesis test on independent trials on 30 test samples sampled using paired t-tests to test
\begin{equation*}
    H_{0} : I_{S;R}(\phi_{\mathrm{rob}}^{*}, g) = I_{S;R}(\phi_{\mathrm{nom}}^{*}, g)
    \qquad
    H_{A} : I_{S;R}(\phi_{\mathrm{rob}}^{*}, g) > I_{S;R}(\phi_{\mathrm{nom}}^{*}, g)
\end{equation*}
where $\phi_{\mathrm{rob}}^{*}$ is as defined in \Cref{eqn:rob_mutual_info} and $\phi_{\mathrm{nom}}^{*} := \argmax_{\phi\in[0,2\pi)^{M}} I_{S;R}(\phi, \widehat{g})$. We similarly run these hypothesis tests against the baseline PGM measurement scheme, denoted $\phi_{\mathrm{PGM}}$. We repeat this procedure across a number of parameter choices, namely for the GRF smoothness parameter $\rho$, the number of GUS states $M$, and the spectral truncation $N$. Uncertainty regions for $G$, and subsequently $g$, were defined as per \Cref{section:uncertain_prop_quantum} with the corrected calibration quantiles computed per \Cref{section:coverage_guar_exp}; in particular, regions were defined using $\alpha = 0.10$. Notably, while coverage guarantees for $G$ were proven with regions having a radius of $\sqrt{3} M (2\widehat{q}_{N;\tau}^{*} + (\widehat{q}_{N;\tau}^{*})^{2})^{1/2}$, we found this to be highly conservative and for $\widehat{q}_{N;\tau}^{*}$ to suffice as the radius for the $G$ prediction region; future work may seek to theoretically characterize the settings in which this reduced radius suffices for coverage. 

The results are presented in \Cref{tab:rob_opt_comparison}. From this, we see that across a diverse collection of problem settings, robust state discrimination significantly outperforms both the baseline PGM and nominal measurement schemes. Notably, holding $\rho$ and $M$ constant, the magnitude of the improvement resulting from robustness generally increases for higher $N$. This is a reflection of the result demonstrated in \Cref{lemma:subopt_gap_quantum}, namely that the suboptimality bound grows with the conservatism of the prediction region. As discussed in \Cref{section:wavefunc_coverage}, this conservatism is reduced as the quantile correction factor decays, which occurs with the observation of a larger spectral truncation or with increasing smoothness of the underlying functions. Hence, the greater improvement with higher $N$ reflects reduced conservatism of the prediction regions. By the same token, the correction factor is reduced in considering a higher $\rho$ for a fixed $M$ and $N$, which similarly translates to a greater magnitude of improvement from considering the robust formulation over the nominal or PGM schemes. 

\section{Discussion}\label{section:discussion}
We have here presented an approach to extend conformal prediction to enable functional coverage over Sobolev spaces and demonstrated the utility of such coverage in robust decision-making pipelines. This work suggests many avenues of extension both for generalizing the presented method. Extending this framework to enable robust design optimization in iterative setups, such as those related to shape optimization as discussed in \Cref{section:shape_opt} or material discovery, would be of great interest \cite{sriram2025open,moon2025catbench}. Such problems fundamentally differ from the framework discussed here, as the single-step prediction of the problem parameter is instead replaced by an iterative cycling between parameter predictions and optimization updates. Further, this manuscript focused on characterizing the uncertainty of spectral neural operators; extending this to alternate operator learning schemes based on discretized domains, such as Fourier neural operators, would also be of interest \cite{bonev2023spherical,li2020fourier}. Additionally, previous works have additionally leveraged a similar robust optimization scheme for linear control problems \cite{patel2024conformallinear}. A natural extension of such work, therefore, is to develop robust PDE control methods by similarly leveraging the conformal set ideas developed herein.


\newpage
\bibliographystyle{unsrt}
\bibliography{references}

\newpage
\appendix 
\section{Functional Coverage Guarantee}\label{section:coverage_guarantee}
\begin{theorem}
    Let $\{(A^{(i)}, U^{(i)})\}\cup (A', U')\overset{\mathrm{iid}}{\sim}\mathcal{P}$ satisfy \Cref{assumption:output_bound} for some $B(A)$. Further, let $\mathcal{D}_{C} := \{(A^{(i)}, \Pi_N U^{(i)})\}$. Let $\alpha\in(0,1)$ and $\widehat{q}_{N;\tau}$ be the $k$-th order statistic for $k := \ceil{(N_{C}+1)(1-\alpha)}$ of \Cref{eqn:score_func} over $\mathcal{D}_{C}$ for a fixed $\mathcal{G}$ and $\tau\in\{1,...,s\}$. Then
    \begin{equation}
        \mathcal{P}_{\{(A^{(i)}, U^{(i)})\}, (A', U')}\left(
        \| \mathcal{G}(A') - U' \|_{{s-\tau}}^{2}
        \le \widehat{q}_{N;\tau} + B(A') N^{-2\tau} \right) \ge 1-\alpha.
    \end{equation}
\end{theorem}
\begin{proof}
    We consider the event of interest conditionally on $s_{N;\tau}(A', U')\le\widehat{q}_{N;\tau}$. Then:
    \begin{align*}
        \| \mathcal{G}(A') - U' &\|_{{s-\tau}}^{2}
        = \sum_{n\in\mathbb{Z}^{d}} (1 + \| n \|_{2}^{2})^{s-\tau} ([\mathcal{G}(A')]_{n} - U'_{n})^{2}  \\
        &= \underbrace{ \sum_{n\in\mathbb{Z}^{d} : | n |_{\infty} \le N} (1 + \| n \|_{2}^{2})^{s-\tau} ([\mathcal{G}(A')]_{n} - U'_{n})^{2}}_{\mathcal{E}_{\le N}} + \underbrace{ \sum_{n\in\mathbb{Z}^{d} : | n |_{\infty} > N} (1 + \| n \|_{2}^{2})^{s-\tau} (U'_{n})^{2} }_{\mathcal{E}_{> N}}
    \end{align*}
    We now demonstrate how each of the above two terms can be bounded. For the former, the result immediately follows in noting that $\mathcal{E}_{\le N}$ is precisely $s_{N;\tau}(A', U')$, from which we have $\mathcal{E}_{\le N}\le\widehat{q}_{N;\tau}$ directly by assumption on $(A', U')$. For the latter, we appeal to standard techniques for Fourier truncation analysis as follows
    \begin{align*}
        \sum_{n\in\mathbb{Z}^{d} : | n |_{\infty} > N} & (1 + \| n \|_{2}^{2})^{s-\tau} (U'_{n})^{2} 
        =  \sum_{n\in\mathbb{Z}^{d} : | n |_{\infty} > N} \frac{(1 + \| n \|_{2}^{2})^{s}}{(1 + \| n \|_{2}^{2})^{\tau}} (U'_{n})^{2}  \\
        &\le \frac{1}{(1 + N^{2})^{\tau}}  \sum_{n\in\mathbb{Z}^{d} : | n |_{\infty} > N} (1 + \| n \|_{2}^{2})^{s} (U'_{n})^{2}
        \le \frac{B(A)}{(1 + N^{2})^{\tau}} 
        \le B(A) N^{-2\tau}.
    \end{align*}
    We conclude by noting that $\mathcal{P}_{\{(A^{(i)}, U^{(i)})\}, (A', U')}(s_{N;\tau}(A', U')\le\widehat{q}_{N;\tau}) \ge 1 - \alpha$ from standard results of conformal prediction, completing the proof.  
\end{proof}

\section{Suboptimality Gap Proofs}\label{section:coverage_bound}
\begin{theorem}
    Let $\{(A^{(i)}, U^{(i)})\}, (A', U'), B(A), N, \mathcal{D}_{C},\widehat{q}_{N;\tau},\mathcal{G}$, and $\tau\in\{1,...,s\}$ be as defined in \Cref{thm:func_cov} and $\mathcal{C}_{N;\tau}^{*}(a)$ be the resulting margin-padded predictor for a marginal coverage level $1-\alpha$. Further, let $\Delta^{*}(a, u, N)$ be as defined in \Cref{eqn:subopt_gap} under this region $\mathcal{C}_{N;\tau}^{*}(a)$ for an objective $J[w,u]$. Then, if $J$ satisfies \Cref{assumption:obj_smooth},
    \begin{equation}
        \mathcal{P}_{\{(A^{(i)}, U^{(i)})\}, (A', U')}\left(0 \le \Delta^{*}(A', U', N) \le 2L \sqrt{\widehat{q}_{N;\tau} + B(A') N^{-2\tau}} \right) \ge 1-\alpha.
    \end{equation}
\end{theorem}

\begin{proof}
    This follows immediately from related proofs given in previous work, namely as:
    \begin{align*}
        &\min_{w\in\mathcal{W}} \max_{\widehat{u}\in\mathcal{C}_{N;\tau}^{*}(a')} J[w, \widehat{u}] - \min_{w} J[w, u'] \\
        &\le \max_{w\in\mathcal{W}} | \max_{\widehat{u} \in \mathcal{C}_{N;\tau}^{*}(a')} J[w, \widehat{u}] - J[w, u'] | \\ 
        &\le L \max_{\widehat{u} \in \mathcal{C}_{N;\tau}^{*}(a')} \| \widehat{u} - u' \|_{s-\tau} 
        \le 2L \sqrt{\widehat{q}_{N;\tau} + B(a') N^{-2\tau}}
    \end{align*}
    Since we have that $\mathcal{P}_{\{(A^{(i)}, U^{(i)})\}, (A', U')}(U' \in \mathcal{C}_{N;\tau}^{*}(A')) \ge 1 - \alpha$, the result immediately follows.
\end{proof}

\begin{corollary}
    Let $\{(A^{(i)}, U^{(i)})\}, (A', U'), B(A), N, \mathcal{D}_{C},\widehat{q}_{N;\tau},\mathcal{G}$, and $\tau\in\{1,...,s\}$ be as defined in \Cref{thm:func_cov} and $\mathcal{V}_{N;\tau}(a)$ be the resulting finite-dimensional predictor for a marginal coverage level $1-\alpha$ given by \Cref{eqn:finite_pred_region}. Further, let $\Delta^{(\mathcal{V})}(a, u, N)$ be as defined in \Cref{eqn:finite_subopt_gap} for an objective $J[w,u]$. Then, if $J$ satisfies \Cref{assumption:obj_smooth},
    \begin{equation}
        \mathcal{P}_{\{(A^{(i)}, U^{(i)})\}, (A', U')}\left(0 \le \Delta^{(\mathcal{V})}(A', U', N) \le L \sqrt{4 \widehat{q}_{N;\tau} + B(A') N^{-2\tau}} \right) \ge 1-\alpha.
    \end{equation}
\end{corollary}

\begin{proof}
    This follows comparably to above, as
    \begin{align*}
        &\min_{w\in\mathcal{W}} \max_{\widehat{v}\in\mathcal{V}_{N;\tau}(a')} J[w, \widehat{v}] - \min_{w\in\mathcal{W}} J[w, u'] 
        \le \max_{w\in\mathcal{W}} | \max_{\widehat{v}\in\mathcal{V}_{N;\tau}(a')} J[w, \widehat{v}] - J[w, u'] | \\ 
        &\le L \max_{\widehat{v}\in\mathcal{V}_{N;\tau}(a')} \| \widehat{v} - u' \|_{s-\tau} 
        = L \sqrt{\max_{\widehat{v}\in\mathcal{V}_{N;\tau}(a')} \| \widehat{v} - u' \|^{2}_{s-\tau}} \\
        &:= L \sqrt{\max_{\widehat{v}\in\mathcal{V}_{N;\tau}(a')} \sum_{n\in\mathbb{Z}^{d} : | n |_{\infty} \le N} (1 + \| n \|_{2}^{2})^{s-\tau} ([\widehat{v}]_{n} - u'_{n})^{2} + \sum_{n\in\mathbb{Z}^{d} : | n |_{\infty} > N} (1 + \| n \|_{2}^{2})^{s-\tau} (u'_{n})^{2}} \\
        &\le L \sqrt{4 \widehat{q}_{N;\tau} + B(a') N^{-2\tau}}.
    \end{align*}
    Since $\mathcal{P}_{\{(A^{(i)}, U^{(i)})\}, (A', U')}(\Pi_N U'\in\mathcal{V}_{N;\tau}(A')) \ge 1 - \alpha$, the result immediately follows.
\end{proof} 

\section{Calibration Across PDE Families {Proofs}}

\subsection{Elliptic PDEs}\label{section:elliptic_bound_pf}
We here state the relevant bounds provided by classical elliptic regularity theory for reference from which \Cref{corollary:elliptic_coverage} immediately follows. We present below a special case of the statement from \cite{Sturm_EllipticPDE_2017} for scalar fields on $\mathbb{T}^{d}$.


\begin{theorem}
     (Theorem 6 of \cite{Sturm_EllipticPDE_2017}) Let $L$ be an elliptic operator of order $\ell$ on $\mathbb{T}^{d}$ such that $Lu = f$. Then, there exists $C_{s,L}$ such that for $u \in \mathcal{H}^{s+\ell}(\mathbb{T}^{d}) \cap \mathrm{Ker}(L)^{\perp}$, $f \in \mathcal{H}^{s}(\mathbb{T}^{d})$ and $\|u\|_{s+\ell} \leq C_{s,L} \|f\|_{s}$.
\end{theorem}

\subsection{Parabolic PDEs}\label{section:parabolic_bound_pf}
The coverage result follows immediately in demonstrating a tail bound as in the elliptic PDE setting demonstrated above.

\begin{corollary}
    Let $L$ be a symmetric, uniformly elliptic operator of order 2 on $\mathbb{T}^{d}$ such that $\partial_t u = -Lu$, $u(0) = u_0\in\mathcal{H}^{s}(\mathbb{T}^{d})$, and the coefficients $a_{k,l}=a_{l,k}\in C^1(\mathbb{T}^d)$ are real-valued. Then, for any $s\in\mathbb{R}$, $\beta\ge 0$, and $t > 0$, there exist constants $C_{s,\beta} < \infty$ and $\omega_s\in\mathbb{R}$ such that
        \begin{equation}
            \| u_t \|_{s + 2\beta} \le C_{s,\beta} (t^{-\beta} + 1) e^{\omega_s t} \| u_{0} \|_{s}
        \end{equation}
\end{corollary}
\begin{proof}
    To demonstrate this, we can directly leverage results from the classical theory of semigroups of linear operators as largely presented in \cite{pazy2012semigroups}. To prove this, we proceed in two steps: first, we cite the fact that analytic semigroups have a smoothing property, and then, we demonstrate that the solution operators of parabolic PDEs form an analytic semigroup. 

    \begin{theorem}[Theorem 6.13]\label{theorem:analytic_smoothing}
    Let $-A$ be the infinitesimal generator of an analytic semigroup $T(t)$.  
    If $0 \in \rho(A)$, then:
    \begin{enumerate}
        \item $T(t) : X \to D(A^{\beta})$ for every $t > 0$ and $\beta \ge 0$.
        
        \item For every $x \in D(A^{\beta})$ we have
        \[
            T(t) A^{\beta} x = A^{\beta} T(t) x.
        \]
        
        \item For every $t > 0$ the operator $A^{\beta} T(t)$ is bounded and
        \begin{equation}
            \| A^{\beta} T(t) \| \le M_{\beta} t^{-\beta} e^{-\delta t}.
            \tag{6.25}
        \end{equation}
        
        \item Let $0 < \beta \le 1$ and $x \in D(A^{\beta})$, then
        \begin{equation}
            \| T(t)x - x \| \le C_{\beta} t^{\beta} \| A^{\beta} x \|.
            \tag{6.26}
        \end{equation}
    \end{enumerate}
    \end{theorem}

    In this particular case, therefore, we wish to specifically leverage the result (c) from this theorem. To do so, we wish to demonstrate that parabolic PDEs of the type studied herein satisfy the hypotheses of \Cref{theorem:analytic_smoothing}. That is, we show the solution operators of parabolic PDEs form an analytic semigroup, from which it suffices to consider its resulting infinitesimal generator to arrive at the final conclusion. 

    To do so, we modify the statement of Theorem 3.6 given in Chapter 7 of Pazy: the version originally presented gave the desired statement of an infinitesimal generator for differential operators defined over Dirichlet boundary conditions, while we sought a statement over periodic boundary conditions here. Additionally, we are interested in the Hilbert space induced by the 2-Sobolev norm, while their results were stated generically over $\mathcal{W}^{2,p}$ Sobolev spaces. Intuitively, the proof relies on first proving that the linear operator $A$ generates a $C_0$-semigroup of contractions on $X$ and then showing that this group is analytic by demonstrating $A$ is sectorial.

    \begin{theorem}[Pazy, \textit{Modified} Theorem~3.6]\label{theorem:parabolic_smoothing}
    Let
    \[
    A(x,D)u := - \sum_{k,l=1}^d \frac{\partial}{\partial x_k}
    \Big( a_{k,l}(x)\frac{\partial u}{\partial x_l} \Big)
    \]
    be a symmetric, strongly elliptic differential operator of order $2$ on
    $\mathbb{T}^d$, where the coefficients $a_{k,l}=a_{l,k}\in C^1(\mathbb{T}^d)$
    are real-valued. Define the operator $A$ on $\mathcal{L}^2(\mathbb{T}^d)$ by
    \[
    D(A) := \mathcal{H}^{2}(\mathbb{T}^d), \qquad Au := A(x,D)u \quad \text{for } u\in D(A).
    \]
    Then $-A$ is the infinitesimal generator of an analytic $C_0$-semigroup
    $\{T(t)\}_{t\ge 0}$ of contractions on $\mathcal{L}^2(\mathbb{T}^d)$, i.e.
    \[
    \|T(t)\|_{\mathrm{op}(\mathcal{L}^2,\mathcal{L}^2)} \le 1 \qquad \text{for all } t\ge 0.
    \]
    \end{theorem}
    \begin{proof}
        To prove this statement, we demonstrate that the conditions of the Lumer-Phillips Theorem hold, provided below for convenience.

        \begin{theorem}[Lumer--Phillips]\label{thm:lumer_phillips}
        Let $A$ be a linear operator with dense domain $D(A)$ in a Banach space $X$.
        \begin{enumerate}
            \item[(a)] If $A$ is dissipative and there exists $\lambda_0>0$ such that
            \[
            \operatorname{Range}(\lambda_0 I - A) = X,
            \]
            then $A$ is the infinitesimal generator of a $C_0$-semigroup of
            contractions on $X$.
        \end{enumerate}
        \end{theorem}

        In particular, we care to use result (a) of this statement, namely that over the space $X = \mathcal{L}^2(\mathbb{T}^d)$, $-A$ is the infinitesimal generator. To do so, we proceed to demonstrate that the conditions hold. We begin by demonstrating that $\mathcal{H}^2(\mathbb{T}^d)$ is dense in $\mathcal{L}^2(\mathbb{T}^d)$. It is well-known that $\mathcal{C}^{\infty}(\mathbb{T}^d)$ is dense in $\mathcal{L}^2(\mathbb{T}^d)$ by classical Fourier analysis, since any $f\in\mathcal{L}^2(\mathbb{T}^d)$ can be approximated by a Fourier series. Since $\mathcal{C}^{\infty}(\mathbb{T}^d)\subset\mathcal{H}^2(\mathbb{T}^d)$, the density of $\mathcal{H}^2(\mathbb{T}^d)$ then immediately follows.

        We now demonstrate that $-A$ is dissipative. Since we consider here the Hilbert space $\mathcal{L}^2(\mathbb{T}^d)$, the definition of dissipativity is $\Re \langle -Au, u \rangle_{\mathcal{L}^2(\mathbb{T}^d)} \le 0$ for any $u\in\mathcal{H}^2(\mathbb{T}^d)$. We simply prove this from definition as follows:
        \begin{align*}
            \langle -Au, u \rangle_{\mathcal{L}^2(\mathbb{T}^d)} 
            &= \left\langle \sum_{k,l=1}^d \frac{\partial}{\partial x_k}
    \Big( a_{k,l}(x)\frac{\partial u}{\partial x_l} \Big), u \right\rangle_{\mathcal{L}^2(\mathbb{T}^d)} \\
            &= \sum_{k,l=1}^d  \left\langle \frac{\partial}{\partial x_k}
    \Big( a_{k,l}(x)\frac{\partial u}{\partial x_l} \Big) , u \right\rangle_{\mathcal{L}^2(\mathbb{T}^d)} \\
            &= - \sum_{k=1}^d \sum_{l=1}^d \left\langle a_{k,l}(x)\frac{\partial u}{\partial x_l}, \frac{\partial u}{\partial x_k} \right\rangle_{\mathcal{L}^2(\mathbb{T}^d)} \\
            &= - \sum_{k=1}^d \left\langle \sum_{l=1}^d a_{k,l}(x)\frac{\partial u}{\partial x_l}, \frac{\partial u}{\partial x_k} \right\rangle_{\mathcal{L}^2(\mathbb{T}^d)} \\
            &= - \sum_{k=1}^d \left\langle [a(x) \nabla u]_{k}, \frac{\partial u}{\partial x_k} \right\rangle_{\mathcal{L}^2(\mathbb{T}^d)}
            = -\left\langle a(x) \nabla u, \nabla u \right\rangle_{\mathcal{L}^2(\mathbb{T}^d;\mathbb{C}^d)}
        \end{align*}
        where the third step follows integration by parts, where the boundary term vanishes under the assumed periodic boundary conditions, and the final quantities now represent the inner product on the Hilbert space of \textit{vector}-valued functions $\mathcal{L}^2(\mathbb{T}^d;\mathbb{C}^d)$. By the uniform ellipticity property of $A$, we have that there exists a $\theta > 0$ such that for all $x\in\mathbb{T}^d$ $\sum_{k,l=1}^d a_{k,l}(x)\xi_{k}\xi_{l} \ge \theta|\xi|^{2}$ for any $\xi\in\mathbb{R}^d$. We, therefore, have that
        \begin{align*}
            -\Re&\left\langle a(x) \nabla u, \nabla u \right\rangle_{\mathcal{L}^2(\mathbb{T}^d;\mathbb{C}^d)}
            = -\int_{\mathbb{T}^{d}} \Re \left\langle a(x) \nabla u(x), \nabla u(x) \right\rangle_{\mathbb{C}^d} dx \\
            &= -\int_{\mathbb{T}^{d}} \left(\left\langle a(x) \Re (\nabla u(x)), \Re (\nabla u(x)) \right\rangle_{\mathbb{R}^d} + \left\langle a(x) \Im (\nabla u(x)), \Im (\nabla u(x)) \right\rangle_{\mathbb{R}^d}\right) dx \\
            &\le -\int_{\mathbb{T}^{d}} \theta (\| \Re (\nabla u(x)) \|_{\mathbb{R}^d}^2 + \| \Im (\nabla u(x)) \|_{\mathbb{R}^d}^2) dx
            = -\int_{\mathbb{T}^{d}} \theta \| \nabla u \|_{\mathbb{C}^d}^2 dx
            \le 0,
        \end{align*}
        demonstrating the desired dissipativity condition.

        The final property to demonstrate is that there exists a $\lambda_0$ such that $\operatorname{Range}(\lambda_0 I + A) = X$. To demonstrate this, we wish to demonstrate that there exists a $\lambda_0$ such that, for any $f\in\mathcal{L}^2(\mathbb{T}^d)$, there exists a $u\in\mathcal{H}^2(\mathbb{T}^d)$ such that $(\lambda_0 I + A) u = f$. To do so, we follow the standard approach to appeal to the Lax-Milgram Theorem to demonstrate that such a solution function $u$ exists for each $f$. To do so, we must exhibit a coercive, bounded bilinear form $B_{\lambda}(u, v) : \mathcal{H}^{1}(\mathbb{T}^d)\times\mathcal{H}^{1}(\mathbb{T}^d)\to\mathbb{C}$ and a bounded linear functional $F\in(\mathcal{H}^{1}(\mathbb{T}^d))^{*}$, from which it follows that there is a unique $u\in\mathcal{H}^{1}(\mathbb{T}^d)$ such that $B_{\lambda}(u, v) = F(v)$. Clearly, we wish for this to match the form of the desired expression $(\lambda_0 I + A) u = f$, for which reason we begin by defining a bilinear form as follows:
        \begin{equation*}
            B_{\lambda}(u, v) := \lambda \langle u, v \rangle_{\mathcal{L}^2(\mathbb{T}^d)} + \sum_{k=1}^d \sum_{l=1}^d \left\langle a_{k,l}(x)\frac{\partial u}{\partial x_l}, \frac{\partial v}{\partial x_k} \right\rangle_{\mathcal{L}^2(\mathbb{T}^d)}
        \end{equation*}
        for $u,v\in\mathcal{H}^1(\mathbb{T}^d)$, $\lambda > 0$, and the linear function $F(v) := \langle f, v \rangle_{\mathcal{L}^2(\mathbb{T}^d)}$. We now show that $B$ is both bounded and coercive. Since $a_{k,l}\in C^1(\mathbb{T}^d)$ and $\mathbb{T}^d$ is compact, it follows that $a_{k,l}\in \mathcal{L}^{\infty}(\mathbb{T}^d)$. We, therefore, have that
        \begin{align*}
            | B_{\lambda}(u, v) |
            &\le | \lambda \langle u, v \rangle_{\mathcal{L}^2(\mathbb{T}^d)} | + 
            \sum_{k=1}^d \sum_{l=1}^d \left| \left\langle a_{k,l}(x)\frac{\partial u}{\partial x_l}, \frac{\partial v}{\partial x_k} \right\rangle_{\mathcal{L}^2(\mathbb{T}^d)} \right| \\ 
            &\le \lambda \| u \|_{\mathcal{L}^2(\mathbb{T}^d)} \| v \|_{\mathcal{L}^2(\mathbb{T}^d)} + 
            \sum_{k=1}^d \sum_{l=1}^d \left\| a_{k,l}(x) \right\|_{\mathcal{L}^\infty(\mathbb{T}^d)} \left\| \frac{\partial u}{\partial x_l} \right\|_{\mathcal{L}^2(\mathbb{T}^d)} \left\| \frac{\partial v}{\partial x_k} \right\|_{\mathcal{L}^2(\mathbb{T}^d)} \\ 
            &\le \lambda \| u \|_{\mathcal{L}^2(\mathbb{T}^d)} \| v \|_{\mathcal{L}^2(\mathbb{T}^d)} + 
            M \left(\sum_{l=1}^d \left\| \frac{\partial u}{\partial x_l} \right\|_{\mathcal{L}^2(\mathbb{T}^d)}\right) \left(\sum_{k=1}^d \left\| \frac{\partial v}{\partial x_k} \right\|_{\mathcal{L}^2(\mathbb{T}^d)}\right) \\ 
            &\le \lambda \| u \|_{\mathcal{L}^2(\mathbb{T}^d)} \| v \|_{\mathcal{L}^2(\mathbb{T}^d)} + 
            M \sqrt{d} \left(\sum_{l=1}^d \left\| \frac{\partial u}{\partial x_l} \right\|^{2}_{\mathcal{L}^2(\mathbb{T}^d)}\right)^{1/2} \sqrt{d} \left(\sum_{k=1}^d \left\| \frac{\partial v}{\partial x_k} \right\|^{2}_{\mathcal{L}^2(\mathbb{T}^d)}\right)^{1/2} \\ 
            &\le \lambda \| u \|_{\mathcal{L}^2(\mathbb{T}^d)} \| v \|_{\mathcal{L}^2(\mathbb{T}^d)} + 
            Md \left\| \nabla u \right\|_{\mathcal{L}^2(\mathbb{T}^d;\mathbb{C}^d)} \left\| \nabla v  \right\|_{\mathcal{L}^2(\mathbb{T}^d;\mathbb{C}^d)}
        \end{align*}
        Finally, by using the fact that $\| u \|^2_{\mathcal{H}^1(\mathbb{T}^d)} = \| u \|^2_{\mathcal{L}^2(\mathbb{T}^d)} + \| \nabla u \|^2_{\mathcal{L}^2(\mathbb{T}^d;\mathbb{C}^d)}$, we have:
        \begin{align*}
            \| u \|_{\mathcal{L}^2(\mathbb{T}^d)} \le \| u \|_{\mathcal{H}^1(\mathbb{T}^d)}
            \qquad
            \| \nabla u \|_{\mathcal{L}^2(\mathbb{T}^d;\mathbb{C}^d)} \le \| u \|_{\mathcal{H}^1(\mathbb{T}^d)}
        \end{align*}
        from which we see
        \begin{align*}
            | B_{\lambda}(u, v) |
            &\le \lambda \| u \|_{\mathcal{H}^1(\mathbb{T}^d)} \| v \|_{\mathcal{H}^1(\mathbb{T}^d)} + 
            Md \| u \|_{\mathcal{H}^1(\mathbb{T}^d)} \| v \|_{\mathcal{H}^1(\mathbb{T}^d)} \\
            &= (\lambda + Md) \| u \|_{\mathcal{H}^1(\mathbb{T}^d)} \| v \|_{\mathcal{H}^1(\mathbb{T}^d)}
        \end{align*}
        Having demonstrated the boundedness of the bilinear form, we now prove that this form is coercive. This follows immediately from the result shown earlier in the proof. That is, to show $B$ is coercive, we must demonstrate that there exists a constant $c > 0$ such that
        \begin{equation*}
            \Re B_{\lambda}(u, u) \ge c \| u \|^2_{\mathcal{H}^1(\mathbb{T}^d)}
        \end{equation*}
        From before, we had $\Re\left\langle a(x) \nabla u, \nabla u \right\rangle_{\mathcal{L}^2(\mathbb{T}^d;\mathbb{C}^d)}\ge \theta \| \nabla u \|^2_{\mathcal{L}^2(\mathbb{T}^d;\mathbb{C}^d)}$. Therefore, it follows
        \begin{align*}
            \Re B_{\lambda}(u, u)
            &= \lambda \| u \|^2_{\mathcal{L}^2(\mathbb{T}^d)} + \sum_{k=1}^d \sum_{l=1}^d \Re \left\langle a_{k,l}(x)\frac{\partial u}{\partial x_l}, \frac{\partial u}{\partial x_k} \right\rangle_{\mathcal{L}^2(\mathbb{T}^d)} \\
            &\ge \lambda \| u \|^2_{\mathcal{L}^2(\mathbb{T}^d)} + \theta \| \nabla u \|^2_{\mathcal{L}^2(\mathbb{T}^d;\mathbb{C}^d)}
            \ge \min\{\lambda, \theta \} \| u \|^2_{\mathcal{H}^1(\mathbb{T}^d)},
        \end{align*}
        where the final statement again comes from the previously referenced norm equivalence. It, therefore, suffices to demonstrate the boundedness of the linear functional to complete this proof. For a linear functional, boundedness is defined as there existing a constant $C$ such that $| F(v) | \le C \| v \|_{\mathcal{H}^1(\mathbb{T}^d)}$. This follows immediately as
        \begin{align*}
            | F(v) |
            &\le \| f \|_{\mathcal{L}^2(\mathbb{T}^d)} \| v \|_{\mathcal{L}^2(\mathbb{T}^d)}
            \le \| f \|_{\mathcal{L}^2(\mathbb{T}^d)} \| v \|_{\mathcal{H}^1(\mathbb{T}^d)},
        \end{align*}
        with $C := \| f \|_{\mathcal{L}^2(\mathbb{T}^d)}$. Having demonstrated these conditions, we now have that the Lax-Milgram Theorem implies the existence of a \textit{weak} solution $u\in\mathcal{H}^1(\mathbb{T}^d)$ for $(\lambda I + A) u = f$. 

        To finally conclude that there exists a \textit{strong} solution $u\in\mathcal{H}^2(\mathbb{T}^d)$ for $(\lambda I + A) u = f$ as desired, we appeal to results from elliptic regularity. In particular, we appeal to the following result from \cite{evans2022partial} to reach the desired conclusion.

        \begin{theorem}[Evans 6.3.1, Interior $\mathcal{H}^{2}$-regularity (Adjusted notation)]
        Assume
        \begin{equation*}
        a_{k,l} \in C^{1}(U), \qquad b_{k}, c \in L^{\infty}(U), \qquad (k,l = 1,\dots,n),
        \end{equation*}
        and $f \in \mathcal{L}^{2}(U)$. Suppose furthermore that $u \in \mathcal{H}^{1}(U)$ is a weak solution of the elliptic PDE $Lu = f \text{in } U$. Then $u \in \mathcal{H}^{2}_{\mathrm{loc}}(U)$.
        \end{theorem}

        In this setting, we are taking $a_{k,l}$ to be as defined earlier, $b_k = 0$, and $c = \lambda$. By assumption, therefore, we have that the conditions on the coefficients of $L$ and the smoothness of $f \in \mathcal{L}^{2}(\mathbb{T}^d)$ are satisfied. It, therefore, immediately follows that $u\in\mathcal{H}_{\mathrm{loc}}^2(\mathbb{T}^d)$ for $(\lambda I + A) u = f$. Critically, since $\mathbb{T}^d$ lacks any boundary points, local regularity is equivalent to global regularity, meaning it immediately follows that $u\in\mathcal{H}^2(\mathbb{T}^d)$, as desired. To conclude this proof, we simply take note that any $\lambda_0 > 0$ can be chosen arbitrarily, i.e. taking $\lambda_0 = 1$ concludes this proof of the conditions of the Lumer-Phillips theorem. We, therefore, have that $-A$ is the infinitesimal generator of a $C_0$-semigroup of contractions on $\mathcal{L}^{2}(\mathbb{T}^d)$.

        It only remains to be shown that this semigroup is analytic. To do so, we first note that, since $A$ is symmetric (by the symmetry on $a_{k,l}$) with a domain $\mathcal{H}^2(\mathbb{T}^d)$ that is dense in $\mathcal{L}^2(\mathbb{T}^d)$ \textit{and} $\operatorname{Range}(\lambda_0 I + A) = \mathcal{L}^2(\mathbb{T}^d)$, it immediately follows by the von Neumann criterion that $A$ is self-adjoint. In addition, we demonstrated the non-negativity of the operator above, namely that $\langle Au, u \rangle \ge 0$. 
        
        From this, we use the fact that any non-negative, self-adjoint operator is necessarily sectorial. The operator $-A$ is called sectorial if there exists a $\theta\in(0,\pi/2]$ such that $\overline{\Sigma_{\theta}}\subset\rho(-A)$ where $\Sigma_{\theta} := \{ z \in\mathbb{C} \backslash\{0\} : | \arg z | < \theta \}$ and $\| (z I + A)^{-1} \| \le C / | z |$ for all $z\in\overline{\Sigma_{\theta}}$. 

        To show the former fact, we simply use the fact that, since $A$ is self-adjoint and non-negative, the spectrum of $-A$ satisfies $\sigma(-A)\subset(-\infty,0]$, meaning $\rho(-A)\supset \{z : \Re z > 0\} = \Sigma_{\theta}$ for any $\theta\in(0,\pi/2)$, demonstrating this first property.
        
        For the latter, we now consider the operator $z I + A$ for $z\in\mathbb{C}$ with $\Re z > 0$. By the non-negativity of $A$, we know $\langle A u, u \rangle\ge0$. It, therefore, follows that, for any $f\in\mathcal{L}^2(\mathbb{T}^d)$, we can consider the corresponding $u := (z I + A)^{-1} f$, from which it follows that
        \begin{align*}
            \langle f, u \rangle
            &= \langle (z I + A) u, u \rangle
            = z \langle u, u \rangle +  \langle A u, u \rangle \\
            \implies \Re\langle f, u \rangle
            &= \Re z \langle u, u \rangle + \Re \langle A u, u \rangle
            = \Re z \| u \|^{2}_{\mathcal{L}^2(\mathbb{T}^d)} + \Re \langle A u, u \rangle 
            \ge \Re z \| u \|^{2}_{\mathcal{L}^2(\mathbb{T}^d)}
        \end{align*}
        On the other hand, we have
        \begin{align*}
            \Re\langle f, u \rangle
            \le | \langle f, u \rangle |
            \le \| f \|_{\mathcal{L}^2(\mathbb{T}^d)} \| u \|_{\mathcal{L}^2(\mathbb{T}^d)}
        \end{align*}
        It, therefore, follows that
        \begin{align*}
            \Re z \| u \|^{2}_{\mathcal{L}^2(\mathbb{T}^d)} \le \| f \|_{\mathcal{L}^2(\mathbb{T}^d)} \| u \|_{\mathcal{L}^2(\mathbb{T}^d)}
            \implies \| u \| \le \frac{1}{\Re z} \| f \|_{\mathcal{L}^2(\mathbb{T}^d)}
        \end{align*}
        Since $f$ was chosen arbitrarily, it follows that $\| (z I + A)^{-1} \|_{\mathrm{op}(\mathcal{L}^2,\mathcal{L}^2)} \le \frac{1}{\Re z}$. Since $\Re z = |z| \cos(\arg z)$ and $| \arg z | \le \theta$, we have that $\cos(\arg z) \ge \cos(\theta)$, meaning $\frac{1}{\Re z} \le \frac{1}{|z| \cos(\theta)}$, demonstrating the desired decay with $C = 1/\cos(\theta)$. We, therefore, have that $-A$ is a sectorial operator. To finally conclude that $-A$, therefore, generates an analytic semigroup, we finally appeal to the below theorem.

        \begin{theorem}[Engel--Nagel, Theorem 4.6 (partially stated)]
        Let $(A, D(A))$ be an operator on a Banach space $X$. The following statements are equivalent:
        \begin{enumerate}
            \item[(a)] $A$ generates a bounded analytic semigroup $(T(z))_{z\in\Sigma_\delta\cup\{0\}}$ on $X$.
            \item[(e)] $A$ is sectorial.
        \end{enumerate}
        \end{theorem}

        Given the demonstration of the latter fact, we immediately get the former, from which the proof of the analyticity of the generated semigroup is complete.
    \end{proof}

    We first appeal to \Cref{theorem:parabolic_smoothing} and then perform a shift on the operator to appeal to \Cref{theorem:analytic_smoothing} in establishing the final result. In particular, for $\partial_t u = -L u$, we have that $-L$ is the infinitesimal generator of an analytic semigroup of contractions $T(t) := e^{-t L}$ on $\mathcal{L}^2(\mathbb{T}^d)$.
    
    From here, to appeal to \Cref{theorem:analytic_smoothing}, we require that $0\in\rho(A)$ for the operator $A$ of interest. For this reason, we instead consider the shifted operator $A := \omega I + L$ for any $\omega > 0$, since $L + \omega I \ge \omega I$ and $L$ is self-adjoint and non-negative, implying $0 \in\rho(L + \omega I)$. The semigroup generated by such an operator is then $e^{-t(\omega I + L)} = e^{-t\omega}e^{-tL}$. The operator shift, therefore, merely results in a scaling of the semigroup, meaning the resulting semigroup remains analytic.

    We, therefore, have that $-A$ is the infinitesimal generator of an analytic semigroup with $0\in\rho(A)$, from which it follows from \Cref{theorem:analytic_smoothing} that
    \begin{align*}
        \| (\omega I + L)^{\beta} e^{-t\omega} e^{-tL} \|
        = e^{-t\omega} \| (\omega I + L)^{\beta} e^{-t L} \|
        \le M_{\beta} t^{-\beta} e^{-\delta t}.
    \end{align*}
    Notice that, since this is a statement of the operator norm, the bound holds accordingly for any $u_0\in\mathcal{H}^{s}(\mathbb{T}^{d})$, meaning we have by definition
    \begin{align*}
        \frac{\| (\omega I + L)^{\beta} e^{-t L} u_0 \|_{s}}{ \| u_0 \|_{s} } &\le M_{\beta} t^{-\beta} e^{(\omega-\delta) t} \\
        \iff 
        \| (\omega I + L)^{\beta} u(t,\cdot) \|_{s}
        &\le M_{\beta} t^{-\beta} e^{(\omega-\delta) t} \| u_0 \|_{s}
    \end{align*}
    where we used that, by definition of $e^{-t L}$ being the semigroup of solution operators, $u(t,\cdot) = e^{-t L} u_0$. We finally use the fact that, for symmetric uniformly elliptic operators on $\mathbb{T}^d$, there exists $C>0$ such that
    \begin{align*}
        \|u(t,\cdot)\|_{s + 2\beta}
        \le C\big(\|(L+\omega I)^\beta u(t,\cdot)\|_{s} + \|u(t,\cdot)\|_{s}\big)
    \end{align*}
    Combining this with the previous result, we see
    \begin{align*}
        \|u(t,\cdot)\|_{s + 2\beta}
        \le C\big(M_{\beta} t^{-\beta} e^{(\omega-\delta) t} \| u_0 \|_{s} + \|u(t,\cdot)\|_{s}\big)
    \end{align*}
    To reach the final conclusion, we simply bound $\|u(t,\cdot)\|_{s}$, which follows immediately as
    \begin{align*}
        \|u(t,\cdot)\|_{s}
        = \| e^{-tL} u_{0} \|_{s}
        \le \| e^{-tL} \|_{\mathrm{op}(\mathcal{H}^s,\mathcal{H}^s) } \| u_{0} \|_{s}
        \le C_{s} e^{\omega_s t} \| u_{0} \|_{s}
    \end{align*}
    From this, we reach the final conclusion by defining $\tilde{\omega_s} := \max\{\omega-\delta, \omega_s\}$ and noting
    \begin{align*}
        \|u(t,\cdot)\|_{s + 2\beta}
        &\le C\big(M_{\beta} t^{-\beta} e^{(\omega-\delta) t} + C_{s} e^{\omega_s t} \big)  \| u_{0} \|_{s} \\
        &\le C_{s,\beta} \big(t^{-\beta} e^{\tilde{\omega_s} t} + e^{\tilde{\omega_s} t} \big)  \| u_{0} \|_{s}
        = C_{s,\beta} (t^{-\beta} + 1) e^{\tilde{\omega_s} t} \| u_{0} \|_{s}
    \end{align*}
\end{proof}

\section{Multi-Stage Optimization Cost Analysis}\label{section:multi_stage_cost_pf}
\begin{lemma}
    Suppose that $\{N_t\}_{t=1}^{T}$ is a sequence of truncation points such that $N_{t-1}\le N_{t}$. Let $\{(A^{(i)}, U^{(i)})\}, (A', U'),B,\mathcal{G}$, and $\tau\in\{1,...,s\}$ be as defined in \Cref{thm:func_cov}, with $\mathcal{D}^{(t)}_{C}$ defined with respect to $\Pi_{N_{t}}$
    for each $N_{t}$ and $\widehat{q}_{N_{t};\tau}$ defined over $\mathcal{D}^{(t)}_{C}$ for a coverage level $1-\alpha$. Suppose the resulting finite-dimensional predictors $\{\mathcal{V}_{N_t;\tau}(a)\}$ satisfy \Cref{assumption:obj_regularity,assumption:set_structure}. 
    Then, if iterates $w_{t}^{(k)}$ are obtained with \Cref{alg:multistage_rfpto} such that $\eta_t = 1/L_t$ for an objective $J[w,u]$ satisfying \Cref{assumption:obj_smooth} and $\{\mathcal{E}^{(t)}\}$ are defined as in \Cref{eqn:multi_stage_cost}, we have
    \begin{equation}
        \mathcal{E}(\{N_{t}\}_{t=1}^{T}) \le
        C_{N_1} \left\lceil \frac{L_{1}}{2\mu} \log\left(\frac{L_{1} \| w_{1}^{*} - w^{(0)} \|^{2}}{2\varepsilon_1}\right) \right\rceil +
        \sum_{t=2}^{T} C_{N_{t}} \left\lceil \frac{L_{t}}{2\mu} \log\left(\frac{L_{t} L B^{1/2} N_{t-1}^{-\tau}}{\mu\varepsilon_t}\right) \right\rceil
    \end{equation}
\end{lemma}
\begin{proof}
    By definition, we have
    \begin{align*}
        \mathcal{E}(\{N_{t}\}_{t=1}^{T})
        := \sum_{t=1}^{T} C_{N_t} K^{*}_{t}
    \end{align*}
    To upper bound this quantity, therefore, it suffices to upper bound each of the $K^{*}_{t}$ counts. To do so, it suffices to exhibit a $K_t$ achieving the desired $\phi_{t}(w_{t}^{(K_t)}) - \phi_{t}(w_{t}^{*}) \le \varepsilon_{t}$ criterion, from which it follows that $K^{*}_t \le K_t$ by definition. 
    
    Given the assumed smoothness and strong convexity of $\phi_t$, the $K_t$ iteration count can be bounded by results from classical convex optimization. By smoothness, we have that, for any iterate $k$, $\phi_{t}(w^{(k)}) - \phi_{t}(w_{t}^{*}) \le (L_t/2) \| w^{(k)} - w_{t}^{*} \|^{2}$. It, therefore, suffices to find a $K_t$ such that $(L_t/2) \| w^{(K_t)} - w_{t}^{*} \|^{2} \le \varepsilon_{t}$. By Theorem 3.6 of \cite{garrigos2023handbook}, we have that $\| w^{(K_t)} - w_{t}^{*} \| \le (1-\eta_{t}\mu)^{K_t} \| w_{t}^{(0)} - w_{t}^{*} \|$ for the gradient descent iterates of $\phi_{t}$. Given that $w_{t}^{(0)} := w_{t-1}^{*}$, this equivalently yields 
    \begin{align*}
        (L_t/2) \| w^{(K_t)} - w_{t}^{*} \|^{2} \le (L_t/2) (1-\eta_{t}\mu)^{2K_t} \| w_{t-1}^{*} - w_{t}^{*} \|^{2} \\
    \end{align*}
    Using this, we see it suffices to have
    \begin{align*}
        & (1-\eta_{t}\mu)^{2K_t} \| w_{t-1}^{*} - w_{t}^{*} \|^{2} \le \frac{2\varepsilon_t}{L_t} \\
        &\iff K_t \ge \frac{1}{2\log\left(\frac{1}{1-\eta_{t}\mu}\right)} \log\left(\frac{L_{t} \| w_{t-1}^{*} - w_{t}^{*} \|^{2}}{2 \varepsilon_t}\right)
    \end{align*}
    With the particular choice of $\eta_t = 1/L_t$, we have that $1\ge \mu/L_{t} = \mu\eta_t \ge0$, from which we can simplify the above expression using the property $\log(\frac{1}{1-x})\ge x \implies 1/\log(\frac{1}{1-x})\le1/x$ as
    \begin{equation}\label{eqn:steps_lower_bd}
        K_t \ge \frac{L_{t}}{2\mu} \log\left(\frac{L_{t} \| w_{t-1}^{*} - w_{t}^{*} \|^{2}}{2\varepsilon_t}\right)
    \end{equation}
    For the uninteresting case of $t=1$, we simply replace $w_{t-1}^{*}$ with $w^{(0)}$, from which we have
    \begin{equation}
        K_1 \ge \frac{L_{1}}{2\mu} \log\left(\frac{L_{1} \| w_{1}^{*} - w^{(0)} \|^{2}}{2\varepsilon_t}\right)
    \end{equation}
    For any $t > 1$, we now wish to produce a stability guarantee on $\| w_{t-1}^{*} - w_{t}^{*} \|$. To do so, we make use of a ``tail approximation'' property resulting from the specific form our prediction regions $\mathcal{V}_{N_{t};\tau}(a)$ take. In particular, since successive prediction regions are nested (i.e., $\mathcal{V}_{N_{t'};\tau}(a)\subset \mathcal{V}_{N_{t};\tau}(a)$ for any $t'\le t$), we have that, for any $v^{(N_{t})}\in\mathcal{V}_{N_{t};\tau}(a)$, there is a $v^{(N_{t'})}\in\mathcal{V}_{N_{t'};\tau}(a)$ such that $\| v^{(N_{t})} - v^{(N_{t'})} \|_{s-\tau}\le B^{1/2} N^{-\tau}_{t'}$. This follows as we can simply consider a $v^{(N_{t'})}$ such that $v^{(N_{t'})} = \Pi_{N_{t'}} v^{(N_{t})}$, from which we have that
    \begin{align*}
        \| v^{(N_{t})} - v^{(N_{t'})} \|^{2}_{s-\tau}
        &= \sum_{n\in\mathbb{Z}^{d} : N_{t'} < |n|_{\infty}\le N_{t}} (1 + \| n \|_{2}^{2})^{s-\tau} ([v^{(N_{t})}]_{n})^{2} \\
        &= \sum_{n\in\mathbb{Z}^{d} : N_{t'} < |n|_{\infty}\le N_{t}} \frac{(1 + \| n \|_{2}^{2})^{s}}{(1 + \| n \|_{2}^{2})^{\tau}} ([v^{(N_{t})}]_{n})^{2} \\
        &\le \frac{1}{(1 + N_{t'}^{2})^{\tau}}  \sum_{n\in\mathbb{Z}^{d} : N_{t'} < |n|_{\infty}\le N_{t}} (1 + \| n \|_{2}^{2})^{s}([v^{(N_{t})}]_{n})^{2}
        \le B N_{t'}^{-2\tau},
    \end{align*}
    demonstrating the desired property after taking the appropriate square root. 
    
    We now make use of this tail approximation property. In particular, returning to the desired stability property claim, we are considering successive stages $t'=t-1$ and $t$. By tail approximation, for the point $v^{(*;N_{t})}(w)$, we know there exists a point $v\in\mathcal{V}_{N_{t-1};\tau}(a)$ such that $\| v^{(N_{t})} - v \|_{s-\tau}\le B^{1/2} N_{t-1}^{-\tau}$. Making use of this $v$ and the Lipschitz property of $J$, we have:
    \begin{align*}
        \phi_{t}(w) - \phi_{t-1}(w)
        &:= J[w, v^{(*;N_{t})}(w)] - J[w, v^{(*;N_{t-1})}(w)] \\
        &\le J[w, v^{(*;N_{t})}(w)] - J[w, v] \\
        &\le L \| v^{(*;N_{t})}(w) - v \|_{s-\tau}
        \le L B^{1/2} N_{t-1}^{-\tau}
    \end{align*}
    By strong convexity again, we have that $\mu/2 \| w_{t-1}^{*} - w_{t}^{*} \|^{2} \le |\phi_{t}(w_{t-1}^{*}) - \phi_{t}(w_{t}^{*}) |$. Using the fact that 
    \begin{align*}
        \phi_{t}(w)
        := J[w, v^{(*;N_{t})}(w)]
        \ge J[w, v^{(*;N_{t-1})}(w)]
        =: \phi_{t-1}(w),
    \end{align*}
    which follows as $\mathcal{V}_{N_{t-1};\tau}(a)\subset \mathcal{V}_{N_{t};\tau}(a)$, and rephrasing the previous property as $\phi_{t}(w) \le \phi_{t-1}(w) + L B^{1/2} N_{t-1}^{-\tau}$, we have
    \begin{align*}
        \| w_{t-1}^{*} - w_{t}^{*} \|^{2} 
        &\le \frac{2}{\mu} (\phi_{t}(w_{t-1}^{*}) - \phi_{t}(w_{t}^{*})) \\
        &\le \frac{2}{\mu} (L B^{1/2} N_{t-1}^{-\tau} + \phi_{t-1}(w_{t-1}^{*}) - \phi_{t-1}(w_{t}^{*})) \\
        &\le \frac{2}{\mu} (L B^{1/2} N_{t-1}^{-\tau} + \phi_{t-1}(w_{t}^{*}) - \phi_{t-1}(w_{t}^{*}))
        = \frac{2}{\mu} L B^{1/2} N_{t-1}^{-\tau}
    \end{align*}
    From this, we get the final sufficiency condition that taking
    \begin{align*}
        K_t := \left\lceil \frac{L_{t}}{2\mu} \log\left(\frac{L_{t} L B^{1/2} N_{t-1}^{-\tau}}{\mu\varepsilon_t}\right) \right\rceil
    \end{align*}
    achieves the desired $\phi_{t}(w_{t}^{(K_t)}) - \phi_{t}(w_{t}^{*}) \le \varepsilon_{t}$ criterion, from which the final conclusion on $K^{*}_t$ follows.
\end{proof}

\section{Experimental Details}\label{section:exp_details}
\subsection{Model Architecture}
We provide below the architecture of the spectral neural operator used in \Cref{section:experiments}.

\begin{table}[H]
\centering
\caption{Architecture of the Spectral Neural Operator. The input is a $2 \times N \times N$ tensor.}
\label{tab:snn_architecture}
\begin{tabular}{lccc}
\toprule
\textbf{Layer Block} & \textbf{Input Channels} & \textbf{Output Channels} & \textbf{Activation} \\
\midrule
Initial Convolution & 2 & $N_h$ & ReLU \\
\midrule
\multicolumn{4}{l}{\textit{Hidden Block (repeated $N_L$ times)}} \\
\quad Convolutional & $N_h$ & $N_h$ & ReLU \\
\midrule
Final Convolution & $N_h$ & 2 & None \\
\bottomrule
\end{tabular}
\end{table}

\subsection{Experimental Setup Details}\label{section:exp_setup_details}
We provide below the PDE hyperparameters used in the experiments, both across calibration and robust optimization tasks.

\begin{table}[H]
\centering
\caption{Hyperparameter values used across the experimental pipeline.}
\label{tab:hyperparameters}
\begin{tabular}{ll}
\toprule
\textbf{Parameter} & \textbf{Value} \\
\midrule
\multicolumn{2}{l}{\textit{Conformal Prediction \& Theorem}} \\
Sobolev Smoothness ($s$) & 2.0 \\
Correction Decay ($\tau$) & 2.0 \\
Spatial Dimensions ($d$) & 2 \\
\midrule
\multicolumn{2}{l}{\textit{Physics \& GRF Parameters}} \\
GRF Correlation ($\tau$) & 1.0 \\
GRF Offset Width ($\sigma$) & 0.5 \\
Domain Size ($L$) & $2\pi$ \\
Evolution Time ($T$) & 0.1 \\
Solver Time Steps & 50 \\
\midrule
\multicolumn{2}{l}{\textit{Hamiltonian-Specific Parameters}} \\
Step-Index Core Radius Factor & 0.2 \\
Step-Index Potential Depth & 1.0 \\
GRIN Strength & 0.1 \\
Heat Equation Viscosity ($\tau$) & 0.01 \\
\bottomrule
\end{tabular}
\end{table}


\newpage
\section{Function Bound Derivations}\label{section:func_bounds}
\subsection{Wavefunction}

\begin{lemma}
    Let $V\in\mathcal{C}^{\infty}(\mathbb{T}^{d})$. Let $\widehat{H} = -\Delta + V$ and $\mathcal{U} := \exp(-\mathrm{iT}\widehat{H})$ for $T\ge 0$. Let $s\in[0,2]$. For any $\psi\in\mathcal{H}^{s}(\mathbb{T}^{d})$ such that $\| \psi \|_{\mathcal{L}^{2}} = 1$,
    \begin{equation}
        \| \mathcal{U} \psi \|^{2}_{s}
        \le (\sqrt{2} (\max\{1 + 2 \| V \|^{2}_{\infty}, 2 \}))^{s} \| \psi \|^{2}_{s}
    \end{equation}
\end{lemma}
\begin{proof}
    For notational ease, we condense the notation of norms to simply indicate the smoothness of the space with the domain fixed to be $\mathbb{T}^{d}$. That is, rather than denoting spaces and norms as $\mathcal{H}^s(\mathbb{T}^{d})$, we simply denote them as $\mathcal{H}^s$. To now bound the norm of interest, we first introduce the graph norm induced by $\widehat{H}$, defined as $\| \psi \|^{2}_{\widehat{H}} := \| \psi \|^{2}_{\mathcal{L}^{2}} + \| \widehat{H} \psi \|^{2}_{\mathcal{L}^{2}}$. By the definition of $\widehat{H} := -\Delta + V$, we notice
    \begin{align*}
        \| \mathcal{U} \psi \|^{2}_{\widehat{H}}
        :&= \| \mathcal{U} \psi \|^{2}_{\mathcal{L}^{2}} + \| \widehat{H} \mathcal{U} \psi \|^{2}_{\mathcal{L}^{2}}
        = \| \psi \|^{2}_{\mathcal{L}^{2}} + \| \widehat{H} \mathcal{U} \psi \|^{2}_{\mathcal{L}^{2}} \\
        &= \| \psi \|^{2}_{\mathcal{L}^{2}} + \| \mathcal{U} \widehat{H} \psi \|^{2}_{\mathcal{L}^{2}}
        = \| \psi \|^{2}_{\mathcal{L}^{2}} + \| \widehat{H} \psi \|^{2}_{\mathcal{L}^{2}}
        =: \| \psi \|^{2}_{\widehat{H}},
    \end{align*}
    where we crucially used the commutativity of $\widehat{H} \mathcal{U} = \mathcal{U} \widehat{H}$, as $\mathcal{U}$ is an operator that is completed defined by $\widehat{H}$. This, therefore, means $\psi$ is preserved in graph norm under transformation by $\mathcal{U}$. We now proceed through the proof for the case of $s=2$, i.e., over the space $\mathcal{H}^2$ and conclude for all $s\in[0,2]$ by appealing to Sobolev interpolation results for $s\in[0,2)$. 
    
    We now critically use the fact that $\| \cdot \|^{2}_{\widehat{H}}$ and $\| \cdot \|^{2}_{\mathcal{H}^{2}}$ are equivalent norms, from which it follows that there exist constants $c, C\in\mathbb{R}$ such that
    \begin{align*}
        c \| \psi \|^{2}_{\mathcal{H}^{2}} \le \| \psi \|^{2}_{\widehat{H}} \le C \| \psi \|^{2}_{\mathcal{H}^{2}}
    \end{align*}
    Using the exhibited preservation of graph norm, we therefore see
    \begin{align*}
        \| \mathcal{U} \psi \|^{2}_{\mathcal{H}^{2}}
        \le \frac{1}{c} \| \mathcal{U} \psi \|^{2}_{\widehat{H}}
        \le \frac{1}{c} \| \psi \|^{2}_{\widehat{H}}
        \le \frac{C}{c} \| \psi \|^{2}_{\mathcal{H}^{2}},
    \end{align*}
    meaning it suffices to demonstrate such constants $c,C$ to complete the proof. To do so, we first note that an equivalent norm to $\| \psi \|^{2}_{\mathcal{H}^{2}}$ can be defined over $\mathbb{T}^d$ as $\| \psi \|^{2}_{\mathcal{L}^2} + \| \Delta \psi \|^{2}_{\mathcal{L}^2}$. To demonstrate this equivalence, we provide explicit constants as follows
    \begin{align*}
        \| \psi \|^2_{\mathcal{L}^2} + \| \Delta \psi \|^{2}_{\mathcal{L}^2} &:= \sum_{n\in\mathbb{Z}^{d}} | \widehat{\psi}_{n} |^{2} + \sum_{n\in\mathbb{Z}^{d}} | \widehat{\Delta \psi}_{n} |^{2} \\
        &= \sum_{n\in\mathbb{Z}^{d}} | \widehat{\psi}_{n} |^{2} + \sum_{n\in\mathbb{Z}^{d}} (\| n \|^{2}_{2} | \widehat{\psi}_{n} |)^{2}
        = \sum_{n\in\mathbb{Z}^{d}} (1 + \| n \|^{4}_{2}) | \widehat{\psi}_{n} |^{2}\\
        \| \psi \|^2_{\mathcal{H}^{2}} &:= \sum_{n\in\mathbb{Z}^{d}} (1 + \| n \|^{2}_{2})^{2} | \widehat{\psi}_{n} |^{2}
    \end{align*}
    We now denote by $P_{\widehat{H}}(n) := 1 + \| n \|^{4}_{2}$ and $P_{\mathcal{H}^{2}}(n) := (1 + \| n \|^{2}_{2})^{2} = 1 + 2 \| n \|^{2}_{2} + \| n \|^{4}_{2}$. To demonstrate the equivalence of norms, we now exhibit constants $c', C'$ such that $c' P_{\widehat{H}}(n) \le P_{\mathcal{H}^{2}}(n) \le C' P_{\widehat{H}}(n)$. The lower bound follows trivially:
    \begin{align*}
        P_{\widehat{H}}(n) := 1 + \| n \|^{4}_{2} 
        \le 1 + \underbrace{2 \| n \|^{2}_{2}}_{\ge 0} + \| n \|^{4}_{2} =: P_{\mathcal{H}^{2}}(n),
    \end{align*}
    from which  it follows that $c' = 1$. For the upper bound, notice that this is equivalent to $P_{\mathcal{H}^{2}}(n) \le C' P_{\widehat{H}}(n) \iff 0 \le C' P_{\widehat{H}}(n) - P_{\mathcal{H}^{2}}(n)$, which we get by taking $C' = 2$:
    \begin{align*}
        C' P_{\widehat{H}}(n) - P_{\mathcal{H}^{2}}(n) 
        &:= 2 (1 + \| n \|^{4}_{2}) - (1 + 2 \| n \|^{2}_{2} + \| n \|^{4}_{2} \\
        &= 2 + 2 \| n \|^{4}_{2} - 1 - 2 \| n \|^{2}_{2} - \| n \|^{4}_{2} \\
        &= 1 - 2 \| n \|^{2}_{2} + \| n \|^{4}_{2} 
        = (1 - \| n \|^{2}_{2})^2 \ge 0,
    \end{align*}    
    as desired. This demonstrates that
    \begin{align*}
        \| \psi \|^{2}_{\mathcal{L}^2} + \| \Delta \psi \|^{2}_{\mathcal{L}^2} \le \| \psi \|^2_{\mathcal{H}^{2}} \le 2 (\| \psi \|^{2}_{\mathcal{L}^2} + \| \Delta \psi \|^{2}_{\mathcal{L}^2})
    \end{align*}
    We finally make use of this fact to get the desired bound 
    \begin{align*}
        \| \psi \|^{2}_{\widehat{H}}
        &:= \| \psi \|^{2}_{\mathcal{L}^2} + \| \widehat{H} \psi \|^{2}_{\mathcal{L}^2} \\
        &:= \| \psi \|^{2}_{\mathcal{L}^2} + \| (V - \Delta) \psi \|^{2}_{\mathcal{L}^2} \\
        &\le \| \psi \|^{2}_{\mathcal{L}^2} + (\| V \psi \|_{\mathcal{L}^2} + \| \Delta \psi \|_{\mathcal{L}^2})^{2} \\
        &\le \| \psi \|^{2}_{\mathcal{L}^2} + 2 \| V \psi \|^2_{\mathcal{L}^2} + 2 \| \Delta \psi \|^{2}_{\mathcal{L}^2} \\
        &\le \| \psi \|^{2}_{\mathcal{L}^2} + 2 \| V \|^{2}_{\infty}  \| \psi \|^{2}_{\mathcal{L}^2} + 2 \| \Delta \psi \|^{2}_{\mathcal{L}^2} \\
        &\le \max\{1 + 2 \| V \|^{2}_{\infty}, 2 \} (\| \psi \|^{2}_{\mathcal{L}^2} + \| \Delta \psi \|^{2}_{\mathcal{L}^2}) 
        \le \max\{1 + 2 \| V \|^{2}_{\infty}, 2 \} \| \psi \|^{2}_{\mathcal{H}^{2}}.
    \end{align*}
    meaning we can define $C := \max\{1 + 2 \| V \|^{2}_{\infty}, 2 \}$. The lower bound follows similarly:
    \begin{align*}
        \| \psi \|^{2}_{\mathcal{H}^{2}}
        &\le 2 (\| \psi \|^{2}_{\mathcal{L}^2} + \| \Delta \psi \|^{2}_{\mathcal{L}^2}) \\
        &\le 2 (\| \psi \|^{2}_{\mathcal{L}^2} + \| (V - \widehat{H}) \psi \|^{2}_{\mathcal{L}^2}) \\
        &\le 2 (\| \psi \|^{2}_{\mathcal{L}^2} + 2 \| V \psi \|^{2}_{\mathcal{L}^2} + 2 \| \widehat{H} \psi \|^{2}_{\mathcal{L}^2}) \\
        &\le 2 (\| \psi \|^{2}_{\mathcal{L}^2} + 2 \| V \|^{2}_{\infty} \| \psi \|^{2}_{\mathcal{L}^2} + 2 \| \widehat{H} \psi \|^{2}_{\mathcal{L}^2}) \\
        &\le 2 ((1 + 2 \| V \|^{2}_{\infty}) \| \psi \|^{2}_{\mathcal{L}^2} + 2 \| \widehat{H} \psi \|^{2}_{\mathcal{L}^2}) \\
        &\le 2 \max\{1 + 2 \| V \|^{2}_{\infty}, 2 \} \| \psi \|^{2}_{\widehat{H}}
    \end{align*}
    from which it follows we can take $1/c := 2 \max\{1 + 2 \| V \|^{2}_{\infty}, 2 \}$. Therefore, 
    \begin{align*}
        \frac{C}{c} 
        = C\left(\frac{1}{c} \right)
        = 2 (\max\{1 + 2 \| V \|^{2}_{\infty}, 2 \})^{2}
    \end{align*}
    We now wish to appeal to an interpolation result to extend this bound result to $s\in[0,2)$. Clearly, in the case of $s=0$, we recover the $\mathcal{L}^2$ norm, from which we immediately have that
    \begin{align*}
        \| \mathcal{U} \psi \|^{2}_{\mathcal{H}^{0}}
        = \| \psi \|^{2}_{\mathcal{H}^{0}},
    \end{align*}
    by the unitarity of $\mathcal{U}$. By Sobolev interpolation, for a bounded operator $\| \mathcal{U} \|_{\mathrm{op}(\mathcal{H}^{r}, \mathcal{H}^{r})} \le C_{r}$ and $\| \mathcal{U} \|_{\mathrm{op}(\mathcal{H}^{r'}, \mathcal{H}^{r'})} \le C_{r'}$, we have an operator bound on the intermediate interpolated Sobolev spaces, i.e., for any $t$ such that $t = \theta r + (1-\theta) r'$ for $\theta\in[0,1]$, $\| \mathcal{U} \|_{\mathrm{op}(\mathcal{H}^{t}, \mathcal{H}^{t})} \le C_{r}^{1-\theta}C_{r'}^{\theta}$. It, therefore, immediately follows, considering the special case of $r=0$ and $r'=2$ we have along with the corresponding bounds $C_{r} = 1$ and $C_{r'} = \sqrt{2} (\max\{1 + 2 \| V \|^{2}_{\infty}, 2 \})$, that for $s\in[0,2]$,
    \begin{align*}
        \| \mathcal{U} \psi \|^{2}_{s}
        \le (\sqrt{2} (\max\{1 + 2 \| V \|^{2}_{\infty}, 2 \}))^{s} \| \psi \|^{2}_{s},
    \end{align*}
\end{proof}

\subsection{Poisson Equation}
\begin{remark}
    For any zero-mean $u, f\in\mathcal{H}^{s}(\mathbb{T}^{d})$, i.e., $\int_{\mathbb{T}^{d}} u(x)\, dx = \int_{\mathbb{T}^{d}} f(x)\, dx = 0$, for which $\Delta u = f$, $\| u \|^{2}_{s} \le 4 \| f \|^{2}_{s-2}$.
\end{remark} 
\begin{proof}
    By the standard solution in Fourier space over zero-mean fields, $u_n = -\widehat{f}_n / \| n \|^{2}_{2}$ for $n\neq 0$ and $\widehat{u}_0 = 0$. It, therefore, immediately follows that
    \begin{align*}
        \| u \|^{2}_{s}
        &:= \sum_{\substack{n\in\mathbb{Z}^{d} \\ n\neq 0}} (1 + \| n \|_{2}^{2})^{s} \widehat{u}^{2}_{n}
        = \sum_{\substack{n\in\mathbb{Z}^{d} \\ n\neq 0}} \frac{(1 + \| n \|_{2}^{2})^{s}}{\| n \|^{4}_{2}} \widehat{f}^{2}_n 
        = \sum_{\substack{n\in\mathbb{Z}^{d} \\ n\neq 0}} \underbrace{\frac{(1 + \| n \|_{2}^{2})^{2}}{\| n \|^{4}_{2}}}_{:= R(n)} (1 + \| n \|_{2}^{2})^{s-2} \widehat{f}^{2}_n \\
        &\le \max_{\substack{n\in\mathbb{Z}^{d} \\ n\neq 0}} (R(n)) \sum_{n\in\mathbb{Z}^{d}} (1 + \| n \|_{2}^{2})^{s-2} \widehat{f}^{2}_n
        := \max_{\substack{n\in\mathbb{Z}^{d} \\ n\neq 0}} (R(n)) \| f \|^{2}_{s-2}
    \end{align*}
    It, therefore, suffices to bound this quantity $R(n)$, which we do as follows
    \begin{align*}
        &R(n) := \frac{(1 + \| n \|_{2}^{2})^{2}}{\| n \|^{4}_{2}}
        = \left(1 + \frac{1}{\| n \|^{2}_{2}}\right)^{2} \\
        &\max_{\substack{n\in\mathbb{Z}^{d} \\ n\neq 0}} \left(1 + \frac{1}{\| n \|^{2}_{2}}\right)^{2}
        = \left(1 + \frac{1}{1}\right)^{2}
        \le 2^2,
    \end{align*}
    completing the proof.
\end{proof}

\subsection{Heat Equation}
\begin{remark}
    For $u(\cdot, 0)\in\mathcal{H}^{s}(\mathbb{T}^{d})$ such that $\partial_{t} u = \tau \Delta u$ and $T > 0$, $\| u(\cdot, T) \|^{2}_{s} \le \| u(\cdot, 0) \|^{2}_{s}$.
\end{remark}
\begin{proof}
    We proceed similarly to the proof of \Cref{remark:poisson_bound}. For notational ease, we denote by $\{\widehat{u_{n}^{(T)}}\}$ the Fourier decomposition of $u(\cdot, T)$. Using this notation, the Fourier solution of the heat equation is then given by $\widehat{u_{n}^{(T)}} = \widehat{u_{n}^{(0)}} \exp(-\tau \| n \|^{2}_{2} T)$. From here, the bound follows trivially, as
    \begin{align*}
        \| u(\cdot, T) \|^{2}_{s}
        := \sum_{n\in\mathbb{Z}^{d}} (1 + \| n \|_{2}^{2})^{s} \left(\widehat{u_{n}^{(T)}}\right)^{2}
        &= \sum_{n\in\mathbb{Z}^{d}} (1 + \| n \|_{2}^{2})^{s} \left(\widehat{u_{n}^{(0)}} \exp(-\tau \| n \|^{2}_{2} T)\right)^{2} \\
        &\le \sum_{n\in\mathbb{Z}^{d}} (1 + \| n \|_{2}^{2})^{s} \left(\widehat{u_{n}^{(0)}}\right)^{2}
        =: \| u(\cdot, 0) \|^{2}_{s}
    \end{align*}
\end{proof}

\section{Quantum Key Distribution Background}\label{section:quantum_disc_problem_setup_extensive}
We adopt the standard setup of continuous-valued quantum key distribution (CV-QKD) and present a brief review of the standard setup, as presented in \cite{notarnicola2023optimizing}. The objective in CV-QKD is for a sender (Alice) to send a quantum key to a receiver (Bob) with the intention that such a key be decipherable with minimal error by Bob under an optimally designed decoding scheme. Formally, the sender has a pure quantum state $\psi\in\mathcal{H}$ for $\mathcal{H}$ a separable Hilbert space with a countable basis $\{\varphi_{n}\}_{n\in\mathbb{Z}}$. 

This state is assumed to be in one of a collection of $M$ \textit{non-orthogonal}, linearly independent pure states $\{\psi_{k}\}_{k=0}^{M-1}$ known as the ``constellation,'' that have prior probabilities $\{q_k\}$. We collect this constellation into a matrix $\psi := \begin{bmatrix} \psi_{0} & ... & \psi_{M-1}\end{bmatrix}$. Commonly, we assume the constellation states follow a symmetry known as ``geometric uniform symmetry'' (GUS), whereby states are assumed to be derived from a single ``base state'' via an operator $\mathcal{S}$. In particular, this symmetry holds if there exists $\mathcal{S} : \mathcal{H}\to\mathcal{H}$ such that $\psi_{k} = \mathcal{S}^{k}\psi_0$ and $\mathcal{S}^{M} = \widehat{\mathbbm{1}}$. We henceforth assume the state discrimination is sought under GUS.

Upon receiving the communicated quantum state, the goal for the receiver is to distinguish which amongst this collection of pure states the received $\psi$ is with one of a set of ``measurement vectors'' $\{\mu_j\}_{j=0}^{M-1}$. Similarly denoting the stacked collection of measurement vectors as $\mathbb{M} := \begin{bmatrix} \mu_{0} & ... & \mu_{M-1} \end{bmatrix}$, this measurement matrix can be related to the state vectors as $\mathbb{M} = \psi A$. Under GUS, $\mathbb{M}$ is a circulant matrix. Denoting the standard DFT and IDFT matrices as $\mathcal{F}_{M}$ and $\mathcal{F}_{M}^{-1}$ respectively, we can use the well-known diagonalizable property of circulant matrices to write $A = \mathcal{F}^{-1}_{M} \Lambda \mathcal{F}_{M}$ for $\Lambda = \mathrm{diag}(\{\lambda_j\}_{j=1}^{M})$ where $\lambda_j(\phi) = e^{i\phi_j} g_{j}$. 

Identifying an optimal measurement protocol under GUS, thus, reduces to finding the optimal collection of angular offsets $\phi$ for a given set of magnitudes $g$, as $g$ is fixed by the constellation. The ``optimal'' protocol can be defined over many metrics: we consider the mutual information between the sender and receiver $I_{S;R}(\phi, g)$, where we adopt the $A$ and $B$ conventions for the sender state (from Alice) and estimated receiver state (by Bob). This mutual information is computed over the measurement distribution
\begin{equation}\label{eqn:final_probs}
    \mathcal{P}(B = j\mid A = k) = \left| \frac{1}{M} \sum_{s=0}^{M-1} e^{-i\phi_s} g_{s}^{1/2} e^{2\pi i s (k-j) / M} \right|^{2}.
\end{equation}
The explicit form of the mutual information is derived and stated in \Cref{section:explicit_mi}. The goal, formally, is then to identify $\phi^*(g) := \argmax_{\phi\in[0,2\pi)^{M}} I_{S;R}(\phi, g)$. Notably, under GUS, oftentimes quantum communication protocols will forgo solving for this optimal measurement scheme in favor of simplicity and take $\phi = \mathbf{0}$. This approach is referred to as a ``pretty good measurement'' (PGM). Critically, while this may seem like an oversimplification, this choice turns out to be \textit{optimal} for GUS under a different choice of metric (the misclassification metric) \cite{cariolaro2015quantum}. Under the mutual information, however, a separation emerges between the optimal measurement scheme and PGM, as we explore over the remaining experiment. 

We restrict our study to the case where the sender is restricted to the transmission of a single photon, meaning states $\psi$ lie in a single Hilbert space $\mathcal{H}$. We specifically consider the space $\mathcal{H} = \mathcal{H}^{s}(\mathbb{T}^{d})$ for some $s= 2$ and $d=2$, where the Fourier modes $\varphi_{n}(x) := e^{2 \pi in\cdot x}$ form a basis of the space for $n\in\mathbb{Z}^{d}$. Formally, we denote by $\psi_{b}$ the complete representation of the state prepared before transmission and $\psi_{a}$ the state after. We assume in reality that these states have been measured up to some truncated spectrum $N$. 
The final dataset, therefore, are measured pairs $\mathcal{D} := \{(\psi_{b}, \Pi_N \psi_{a})\}$, which we then use to train a spectral neural operator as described in \Cref{section:neural_operators} to learn the map $\mathcal{G} : \psi_{b}\to\Pi_N \psi_{a}$. Using this, the receiver now seeks to define an optimal decoding scheme with respect to these \textit{received} states. Under this transmission model, the noise-free Gram matrix is instead replaced by a state-evolved Gram matrix $G := [([\psi_{a}]_{\ell})^{\dagger} ([\psi_{a}]_{k})]_{\ell, k}$ and similarly for $g$. 

\section{Explicit Mutual Information Expression}\label{section:explicit_mi}
We here provide for completeness the explicit form of the mutual information objective. Since $I(S;R) := H(B) - H(B\mid A)$, we derive separately the two terms below:
    \begin{align*}
        &H(B) = \sum_{j=0}^{M-1} \mathcal{P}(B = j) \log(\mathcal{P}(B = j)) \\
        &= \sum_{j=0}^{M-1} \sum_{k=0}^{M-1} \mathcal{P}(A = k) \mathcal{P}(B = j | A = k) \log(\sum_{k=0}^{M-1} \mathcal{P}(A = k) \mathcal{P}(B = j | A = k)) \\
        &= \sum_{j=0}^{M-1} \sum_{k=0}^{M-1} q_k \left| \frac{1}{M} \sum_{s=0}^{M-1} e^{-i\phi_s} g_{s}^{1/2} e^{2\pi i s (k-j) / M} \right|^{2} \log  \left(\sum_{k=0}^{M-1} q_k\left| \frac{1}{M} \sum_{s=0}^{M-1} e^{-i\phi_s} g_{s}^{1/2} e^{2\pi i s (k-j) / M} \right|^{2}\right)
    \end{align*}
    Critically, for $H(B\mid A)$, we can use the fact that $H(X)$ only depends on $\mathrm{Law}(X)$. Under GUS, we have $\mathrm{Law}(B\mid A = k) = \mathrm{Law}(B \mid A=0)$. Therefore, it follows that
    \begin{align*}
        H(B\mid A) 
        &= \sum_{k=0}^{M-1} \mathcal{P}(A = k) H(B | A = k) \\
        &= \sum_{k=0}^{M-1} \mathcal{P}(A = k) H(B | A = 0) \\
        &= H(B | A = 0) \sum_{k=0}^{M-1} \mathcal{P}(A = k) = H(B | A = 0)\\
        H(B | A = 0) &:= \sum_{j=0}^{M-1} \mathcal{P}(B = j\mid A = 0) \log(\mathcal{P}(B = j\mid A = 0)) \\
        &:= \sum_{j=0}^{M-1} \left| \frac{1}{M} \sum_{s=0}^{M-1} e^{-i\phi_s} g_{s}^{1/2} e^{-2\pi i j s / M} \right|^{2} \log\left(\left| \frac{1}{M} \sum_{s=0}^{M-1} e^{-i\phi_s} g_{s}^{1/2} e^{-2\pi i j s / M} \right|^{2}\right) 
    \end{align*}

\newpage
\section{Suboptimality Gap Lemma}\label{section:suboptimalty_gap}
As discussed in the main text, to make claims on the suboptimality of the robust measurement protocol, we require mild assumptions on the measurement protocol given by $(S;R)$. As discussed below, these are assumptions that will be satisfied by most practically relevant protocols. For the introduction of the notation referenced below, see the exposition in \Cref{section:quantum_disc_problem_setup_extensive}.
\begin{assumption}\label{assup:min_prob}
    For a measurement protocol $(S;R)$, $q_k > 0$ for $k=0,...,M-1$. For $j = 0,...,M-1$, $\mathcal{P}(B = j\mid A = 0) > 0$.
\end{assumption}
The assumption that $q_k > 0$ holds for any protocol of interest, else a simplified protocol can be considered where the vacuous states where $q_k = 0$ are eliminated. Similarly, the assumption that $\mathcal{P}(B = j\mid A = 0) > 0$ can be equivalently stated as that there is no measurement vector $\mu_j$ that is orthogonal to the state vector $\psi_0$. This is nearly always true for a well-designed measurement, 
since optimal measurement vectors maximize information gain and thus should ``interact'' with all possible sender states $\psi$ to some extent. Not doing so would mean, for certain transmission states, the receiver is left with no more information after the new measurement as compared to before.

We provide the statement of the general suboptimality gap result to which we wish to appeal below for convenience.

\begin{lemma}
    Let $\{(A^{(i)}, U^{(i)})\}, (A', U'),B(A), N, \mathcal{D}_{C},\widehat{q}^{*}_{N;\tau},\mathcal{G}$, and $\tau\in\{1,...,s\}$ be as defined in \Cref{thm:func_cov} for $A^{(i)} := \psi^{(i)}_{b}$ and $U^{(i)} := \psi^{(i)}_{a}$ and $\mathcal{C}_{N;\tau}^{*}(a)$ be the resulting margin-padded predictor for a marginal coverage level $1-\alpha$. Let $\Delta_{S;R}(\psi_{b}, g)$ be as defined in \Cref{eqn:mi_subopt} for a measurement protocol $(S;R)$ satisfying \Cref{assup:min_prob}. Then, for a constant $L < \infty$,
    \begin{equation}
        \mathcal{P}_{\psi'_{b}, g'}\left(0 \le \Delta_{S;R}(\psi'_{b}, g', N) \le 2\sqrt{3} LM \sqrt{2\widehat{q}_{N;\tau}^{*} + (\widehat{q}_{N;\tau}^{*})^{2}} \right) \ge 1-\alpha.
    \end{equation}
\end{lemma}

\begin{proof}
    To appeal to \Cref{thm:subopt_gap}, we first re-express our suboptimality gap in the min-max form desired, namely as
    \begin{align*}
        \Delta_{S;R}(\psi_{b}, g) 
        &= \min_{\phi\in[0,2\pi)^{M}} \max_{\widehat{g}\in\mathcal{C}_{N;\tau}^{(g)}(\psi_{b})} (-I_{S;R}(\phi, \widehat{g})) - \min_{\phi\in[0,2\pi)^{M}} (-I_{S;R}(\phi, g))
    \end{align*}
    In this re-expressed form, it suffices to demonstrate $-I_{S;R}(\phi, \widehat{g})$ is $L$-Lipschitz in $g$ for any fixed $\phi$. We demonstrate this by first considering a fixed $\phi$, defining the corresponding Lipschitz constant $L(\phi)$, and then considering the supremum over this domain. Since $-I_{S;R}(\phi, g) := H(B\mid A) - H(B) = H(B\mid A = 0) - H(B)$, where the latter equality was discussed in \Cref{section:explicit_mi}, it suffices to demonstrate each of $H(B)$ and $H(B\mid A = 0)$ are Lipschitz in $g$. We do so by first introducing auxiliary variables $p_j$ and $p_{j|0}$, demonstrating that $H(B)$ and $H(B\mid A = 0)$ are respectively Lipschitz in these variables, and concluding by demonstrating the maps $g\to p_j$ and $g\to p_{j|0}$ are Lipschitz.

    To begin, we see $H(B)$ takes the form $\sum_{j} p_j \log(p_j)$, where $p_j := \mathcal{P}(B = j)$. Clearly, $| \partial_{p_j} (p_j \log(p_j)) | = | 1 + \log(p_j) | \le |1 + \log(\epsilon)|$ if $p_j\ge \epsilon > 0$, meaning $H(B)$ is Lipschitz in $p_j$ if $p_j$ is lower bounded. Such a lower bound follows as, under GUS, $\mathcal{P}(B = j\mid A = k) = \mathcal{P}(B = (j - k) \text{ mod } M \mid A = 0)$. By \Cref{assup:min_prob}, we have $\mathcal{P}(B = (j - k) \text{ mod } M \mid A = 0) > 0$. Thus, for any $j$,
    \begin{align*}
        \mathcal{P}(B = j) 
        = \sum_{k=0}^{M-1} \mathcal{P}(B = j | A = k) \mathcal{P}(A = k)
        = \sum_{k=0}^{M-1} q_k \mathcal{P}(B = (j - k) \text{ mod } M \mid A = 0) > 0        
    \end{align*}
    The Lipschitzness of $H(B\mid A = 0) = \sum_{j} \mathcal{P}(B = j\mid A = 0) \log(\mathcal{P}(B = j\mid A = 0))$ follows immediately, since the lower bound on $\mathcal{P}(B = j\mid A = 0)$ exists by \Cref{assup:min_prob}.
    
    To conclude this proof, we simply need to show the maps $g\to p_j$ and $g\to p_{j|0}$ are Lipschitz. Recall that
    \begin{align*}
        p_{j}  := \sum_{k=0}^{M-1} q_k\left| \frac{1}{M} \sum_{s=0}^{M-1} e^{-i\phi_s} g_{s}^{1/2} e^{2\pi i s (k-j) / M} \right|^{2}
        \qquad p_{j|0} := \left| \frac{1}{M} \sum_{s=0}^{M-1} e^{-i\phi_s} g_{s}^{1/2} e^{-2\pi i s j / M} \right|^{2}
    \end{align*}
    Clearly, it suffices to show the map $g\to p_{j|0}$ is Lipschitz, as $g\to p_j$ is then simply the sum of such functions whose Lipschitzness can be shown identically. We do this explicitly by bounding the gradient as follows, where we denote by $c_{s} := e^{-2\pi i s j / M}$:
    \begin{align*}
        p_{j|0}
        &= \left| \frac{1}{M} \sum_{s=0}^{M-1} c_{s} e^{-i\phi_s} g_{s}^{1/2} \right|^{2} \\
        &= \left( \frac{1}{M} \sum_{s=0}^{M-1} c_{s} e^{-i\phi_s} g_{s}^{1/2} \right) \overline{\left( \frac{1}{M} \sum_{s'=0}^{M-1} c_{s'} e^{-i\phi_{s'}} g_{s'}^{1/2} \right)} \\
        &= \frac{1}{M^{2}} \sum_{s=0}^{M-1} \sum_{s'=0}^{M-1} c_{s} \overline{c_{s'}} e^{-i\phi_s} e^{i\phi_{s'}} g_{s}^{1/2} g_{s'}^{1/2}
    \end{align*}
    We now compute the gradient component-wise:
    \begin{align*}
        \partial_{g_{r}} p_{j|0}
        &= \frac{1}{M^{2}} \sum_{s=0}^{M-1} \sum_{s'=0}^{M-1} \left(c_{s} \overline{c_{s'}} e^{-i\phi_s} e^{i\phi_{s'}} \right) \partial_{g_{r}} (g_{s}^{1/2} g_{s'}^{1/2}) \\
        &= \frac{1}{M^{2}} \sum_{s=0}^{M-1} \sum_{s'=0}^{M-1} \left(c_{s} \overline{c_{s'}} e^{-i\phi_s} e^{i\phi_{s'}} \right) \left(\frac{1}{2} g_{s}^{-1/2} g_{s'}^{1/2} \delta_{sr} + \frac{1}{2} g_{s}^{1/2} g_{s'}^{-1/2} \delta_{s'r}\right) \\
        &= \frac{1}{2 M^{2}} \left( \sum_{s=0}^{M-1} \sum_{s'=0}^{M-1} c_{s} \overline{c_{s'}} e^{-i\phi_s} e^{i\phi_{s'}} g_{s}^{-1/2} g_{s'}^{1/2} \delta_{sr} + \sum_{s=0}^{M-1} \sum_{s'=0}^{M-1} c_{s} \overline{c_{s'}} e^{-i\phi_s} e^{i\phi_{s'}} g_{s}^{1/2} g_{s'}^{-1/2} \delta_{s'r}\right) \\
        &= \frac{1}{2 M^{2}} \left( \sum_{s'=0}^{M-1} \underbrace{c_{r} \overline{c_{s'}} e^{-i\phi_r} e^{i\phi_{s'}}}_{b_{s'}} g_{r}^{-1/2} g_{s'}^{1/2} + \sum_{s=0}^{M-1} \underbrace{c_{s} \overline{c_{r}} e^{-i\phi_s} e^{i\phi_{r}}}_{a_{s}} g_{s}^{1/2} g_{r}^{-1/2} \right)
    \end{align*}    
    Critically, we now recognize that $s'$ is a dummy index and can be replaced with the index $s$. With this relabeling, we notice $\overline{b_{s}} = a_{s}$, from which this expression simplifies to $2 \mathrm{Re}(a_{s})$:
    \begin{align*}
        \partial_{g_{r}} p_{j|0} = \frac{1}{2 M^{2}} \left( \sum_{s=0}^{M-1} 2 \mathrm{Re} (c_{r} \overline{c_{s}} e^{-i\phi_r} e^{i\phi_{s}}) g_{r}^{-1/2} g_{s}^{1/2} \right)
        = \frac{1}{M^{2}} \mathrm{Re} \left( c_{r} e^{-i\phi_r} g_{r}^{-1/2} \sum_{s=0}^{M-1} \overline{c_{s}} e^{i\phi_{s}} g_{s}^{1/2} \right) 
    \end{align*}
    Using this per-component expression, we can bound the gradient simply as
    \begin{align*}
        \| \nabla_{g} p_{j|0} \|^{2}
        \le \sup_{g\in\mathcal{C}_{N;\tau}^{(g)}(\psi_{b})} \underbrace{\sum_{r=0}^{M-1} \left| \frac{1}{M^{2}} \mathrm{Re} \left( c_{r} e^{-i\phi_r} g_{r}^{-1/2} \sum_{s=0}^{M-1} \overline{c_{s}} e^{i\phi_{s}} g_{s}^{1/2} \right) \right|^{2}}_{:= \widehat{L}_{p|0}(\phi, g)}
        =: L_{p|0}(\phi)
    \end{align*}
    Notably, this bound $L_{p|0}(\phi)$ is finite if $g_r > 0$. Recall that $G = \psi^{\dagger} \psi$ for a collection of linearly \textit{independent} $\{\psi_k\}$. It, therefore, follows that $G\succ 0$, from which we get that $g_r > 0$ for all $r = 0, ..., M-1$, thereby demonstrating $\| \nabla_{g} p_{j|0} \|$ is bounded and, hence, $g\to p_{j|0}$ is Lipschitz.
    
    Thus, since each of $g\to p_j$ and $g\to p_{j|0}$ are Lipschitz and so too are $p_j\to H(B)$ and $p_{j|0}\to H(B\mid A = 0)$, their composition too is Lipschitz in $g$, from which it follows that there exists some constant $L(\phi)$ for which $-I_{S;R}(\phi, g)$ is the Lipschitz constant in $g$ for this fixed $\phi$. Finally, by the explicit gradient bound derived above, we see the maps $\phi\to L_{p|0}(\phi)$ and $\phi\to L_{p}(\phi)$ are continuous. This follows as the bound $\widehat{L}_{p|0}(\phi, g)$ is a continuous function jointly over $\phi$ and $g$, from which it follows that $\sup_{g\in\mathcal{C}_{N;\tau}^{(g)}(\psi_{b})} \widehat{L}_{p|0}(\phi, g)$ is continuous in $\phi$ as $\mathcal{C}_{N;\tau}^{(g)}(\psi_{b})$ is compact. This coupled with the fact that $p_j\to H(B)$ and $p_{j|0}\to H(B\mid A = 0)$ too are continuous, implies the full map $\phi\to L(\phi)$ is continuous. Finally, as we are considering a periodic domain $[0, 2\pi)^{M}$, we can consider the supremum identically over the closed \textit{compact} domain $L := \sup_{\phi\in [0, 2\pi]^{M}} L(\phi)$, which must be finite by the extreme value theorem, as $L(\phi)$ is bounded and continuous.
\end{proof}

\section{Gram Matrix Uncertainty Propagation Derivation}\label{section:func_prop_der}
We now derive a probabilistic upper bound on the difference between the estimated and true Gram matrices. Note that, by $\mathcal{P}_{\psi_{b}, \psi_{a}}(\| \mathcal{G}(\psi_{b}) - \psi_{a} \|^{2}_{s-\tau} \le \widehat{q}_{N;\tau}^{*}) \ge 1-\alpha$, it follows that, with the same probabilistic guarantee, $\exists \Delta$ such that $\| \Delta \|^{2}_{s-\tau}\le\widehat{q}_{N;\tau}^{*}$ and $\mathcal{G}(\psi_{b}) = \psi_{a} + \Delta$. With this, the desired matrix bound can be immediately established by bounding the difference in each entry as follows, where we use the fact that the true states are normalized in $\mathcal{L}^2$:
\begin{align*}
    \| G - \widehat{G} \|^{2}_{F}
    &:= \sum_{\ell, k = 0}^{M-1} | G_{\ell, k} - \widehat{G}_{\ell, k} |^{2} \\
    &:= \sum_{\ell, k = 0}^{M-1} \left| (\mathcal{G}([\psi_{b}]_{\ell})^{\dagger} (\mathcal{G}([\psi_{b}]_{k}) - ([\psi_{a}]_{\ell})^{\dagger} ([\psi_{a}]_{k}) \right|^{2} \\
    &= \sum_{\ell, k = 0}^{M-1} \left| (\psi_{\ell}^{(a)} + \Delta_{\ell})^{\dagger} (\psi_{k)^{(a)} + \Delta_{k} } - ([\psi_{a}]_{\ell})^{\dagger} ([\psi_{a}]_{k}) \right|^{2} \\
    &= \sum_{\ell, k = 0}^{M-1} \left| 
    (\psi_{\ell}^{(a)})^{\dagger} (\psi_{k)^{(a)} } 
    + (\psi_{\ell}^{(a)})^{\dagger} (\Delta_{k) } 
    + (\Delta_{\ell})^{\dagger} (\psi_{k)^{(a)} } 
    + (\Delta_{\ell})^{\dagger} (\Delta_{k) } 
    - ([\psi_{a}]_{\ell})^{\dagger} ([\psi_{a}]_{k}) \right|^{2} \\
    &\le 3 \sum_{\ell, k = 0}^{M-1} 
    \left(| (\psi_{\ell}^{(a)})^{\dagger} (\Delta_{k) } |^{2}
    + | (\Delta_{\ell})^{\dagger} (\psi_{k)^{(a)} } |^{2}
    + | (\Delta_{\ell})^{\dagger} (\Delta_{k) } |^{2}\right)
\end{align*}
\begin{align*}
    &\le 3 \sum_{\ell, k = 0}^{M-1} \left(
    \| \psi_{\ell}^{(a)} \|^{2}_{\mathcal{L}^{2}} \| \Delta_{k} \|^{2}_{\mathcal{L}^{2}}
    + \| \Delta_{\ell} \|^{2}_{\mathcal{L}^{2}} \| \psi_{k}^{(a)} \|^{2}_{\mathcal{L}^{2}}
    + \| \Delta_{\ell} \|^{2}_{\mathcal{L}^{2}} \| \Delta_{k} \|^{2}_{\mathcal{L}^{2}} \right) \\
    &= 3 \sum_{\ell, k = 0}^{M-1} \left( \| \Delta_{k} \|^{2}_{\mathcal{L}^{2}} + \| \Delta_{\ell} \|^{2}_{\mathcal{L}^{2}}  + \| \Delta_{\ell} \|^{2}_{\mathcal{L}^{2}} \| \Delta_{k} \|^{2}_{\mathcal{L}^{2}} \right) \\
    &\le 3 \sum_{\ell, k = 0}^{M-1} \left( \| \Delta_{k} \|^{2}_{\mathcal{H}^{s-\tau}} + \| \Delta_{\ell} \|^{2}_{\mathcal{H}^{s-\tau}}  + \| \Delta_{\ell} \|^{2}_{\mathcal{H}^{s-\tau}} \| \Delta_{k} \|^{2}_{\mathcal{H}^{s-\tau}} \right) 
    \qquad \text{ (since $ s-\tau\ge 0$) }
    \\
    &\le 3 \sum_{\ell, k = 0}^{M-1} (2\widehat{q}_{N;\tau}^{*} + (\widehat{q}_{N;\tau}^{*})^{2})
    = 3 M^2 (2\widehat{q}_{N;\tau}^{*} + (\widehat{q}_{N;\tau}^{*})^{2}),
\end{align*}

\section{Linear Functional Robust Equivalence}\label{section:lin_robust}
We now discuss the natural setting alluded to in the main text, where a \textit{single} collector location is to be selected for maximal resource collection. 
FOr referenced, the objective can be expressed as the following functional
\begin{equation}
\begin{aligned}
J[w, u] := \int_{\mathcal{B}_{r}(w)} u(x)^{m} dx,
\end{aligned}
\end{equation}
where $m\in\{1,2\}$ across most practical applications. We now demonstrate that, in the special case of considering a \textit{linear} functional that is a shift of an underlying ``template functional,'' the solutions to the nominal and robust solutions are identical. This structure holds for the case above, namely where a single collector location is sought, hence the lack of separation between the nominal and robust solutions.

\begin{remark}
    Let $J[w, u] : (0,2\pi)^{d}\times\mathcal{L}^{2}(\mathbb{T}^{d})\rightarrow\mathbb{R}$ be a functional linear in $u$ whose dependence on $w$ is through a shift of a template functional $L : \mathcal{L}^{2}(\mathbb{T}^{d})\rightarrow\mathbb{R}$, i.e., $J[w, u] = L[u(\cdot-w)]$. Then, for any $u\in\mathcal{L}^{2}(\mathbb{T}^{d})$ and $s\in\mathbb{N}$, $w_{\mathrm{nom}}^{*}(u) = w_{\mathrm{rob}}^{*}(u)$, where
    \begin{equation*}
        w_{\mathrm{nom}}^{*}(u) := \argmin_{w\in(0,2\pi)^{d}} J[w, u]
        \qquad w_{\mathrm{rob}}^{*}(u) := \argmin_{w\in(0,2\pi)^{d}} \max_{\widehat{u}\in\mathcal{B}^{\|\cdot\|_{s}}_{\widehat{q}}(u)} J[w, u]
    \end{equation*}
\end{remark}
\begin{proof}
    By the Riesz representation theorem, for any $w\in(0,2\pi)^{d}$, there is a unique representer $\psi^{(w)}$ such that $J[w, \widehat{u}] = \langle \widehat{u}, \psi^{(w)} \rangle_{s}$. By the assumption that $J[w, u] = L[u(\cdot-w)]$ for some $L$, we have that such representers must similarly have the form $\psi^{(w)}(x) = \psi^{(0)}(x-w)$. Further, we note that, for any $\widehat{u}\in\mathcal{B}^{\|\cdot\|_{s}}_{\widehat{q}}(u)$, $\widehat{u} = u + v$, where $\|v\|_{s}\le\widehat{q}$. From this, we have
    \begin{align*}
        \max_{\widehat{u}\in\mathcal{B}^{\|\cdot\|_{s}}_{\widehat{q}}(u)} J[w, \widehat{u}]
        &= \max_{\widehat{u}\in\mathcal{B}^{\|\cdot\|_{s}}_{\widehat{q}}(u)} \langle \widehat{u}, \psi^{(0)}(x-w) \rangle \\
        &= \max_{v : \|v\|_{s}\le\widehat{q}} \langle u + v, \psi^{(0)}(x-w) \rangle \\
        &= \langle u, \psi^{(0)}(x-w) \rangle + \max_{v : \|v\|_{s}\le\widehat{q}} \langle v, \psi^{(0)}(x-w) \rangle \\
        &= J[w, u] + \widehat{q} \max_{v : \|v\|_{s}\le 1} \langle v, \psi^{(0)}(x-w) \rangle
        =: J[w, u] + \widehat{q} \| \psi^{(0)}(x-w) \|_{\mathcal{H}^{s^{*}}},
    \end{align*}
    where $\| \cdot \|_{\mathcal{H}^{s^{*}}}$ denotes the dual $s$-Sobolev norm. Notably, the Sobolev $s$-norm is shift invariant over $\mathbb{T}^{d}$, as Fourier coefficients only change by a phase upon a shift of the underlying function. That is, since $\widehat{\varphi}^{(w)}_{n} = e^{-2\pi i n w} \widehat{\varphi}^{(0)}_{n}$, $| \widehat{\varphi}^{(w)}_{n} | = | \widehat{\varphi}^{(0)}_{n} |$, from which we have that:
    \begin{align*}
        \| \varphi^{(w)} \|^{2}_{s}
        &:= \sum_{n\in\mathbb{Z}^{d}} (1 + \| n \|_{2}^{2})^{s} (\widehat{\varphi}^{(w)}_{n})^{2} 
        = \sum_{n\in\mathbb{Z}^{d} } (1 + \| n \|_{2}^{2})^{s} (\widehat{\varphi}^{(0)}_{n})^{2}
        =: \| \varphi^{(0)} \|^{2}_{s}
    \end{align*}
    Given the shift invariance of $\| \cdot \|_{s}$, it immediately follows that the dual $\| \cdot \|_{\mathcal{H}^{s^{*}}}$ too is shift invariant, which can be seen as follows. Denote by $T_{w} : \mathcal{L}^{2}(\mathbb{T}^{d})\rightarrow\mathcal{L}^{2}(\mathbb{T}^{d})$ the shift by $w$ operator, i.e., $T_{w}(f(x)) := f(x-w)$, for which the adjoint $T^{*}_{w} = T_{-w}$. With this,
    \begin{align*}
        \| \psi^{(0)}(x-w) \|_{\mathcal{H}^{s^{*}}}
        &:= \max_{v : \|v\|_{s}\le 1} \langle v, \psi^{(0)}(x-w) \rangle \\
        &= \max_{v : \|v\|_{s}\le 1} \langle v, T_{w}(\psi^{(0)}(x)) \rangle \\
        &= \max_{v : \|v\|_{s}\le 1} \langle T_{-w} (v), \psi^{(0)}(x) \rangle \\ 
        &= \max_{v : \|T_{-w} (v)\|_{s}\le 1} \langle T_{-w} (v), \psi^{(0)}(x) \rangle \qquad \text{by translation invariance of } \| \cdot \|_{s} \\ 
        &= \max_{\widetilde{v} : \|\widetilde{v}\|_{s}\le 1} \langle \widetilde{v}, \psi^{(0)}(x) \rangle =: \| \psi^{(0)} \|_{\mathcal{H}^{s^{*}}}
    \end{align*}
    For any $w$, $\max_{\widehat{u}\in\mathcal{B}^{\|\cdot\|_{s}}_{\widehat{q}}(u)} J[w, u] = J[w, u] + \widehat{q} \| \psi^{(0)} \|_{\mathcal{H}^{s^{*}}}$ is constant over $w$. Thus,
    \begin{align*}
        w_{\mathrm{rob}}^{*}(u) := \argmin_{w\in(0,2\pi)^{d}} \max_{\widehat{u}\in\mathcal{B}^{\|\cdot\|_{s}}_{\widehat{q}}(u)} J[w, u]
        &= \argmin_{w\in(0,2\pi)^{d}} J[w, u]
        =: w_{\mathrm{nom}}^{*}(u).
    \end{align*}
\end{proof}

\newpage
\section{Robust Collection Problem Visualization}\label{section:rob_collection_viz}
We here provide visualizations of the field predictions and optimal nominal and robust collection placements across \Cref{fig:rob_nom_collect1} and \Cref{fig:rob_nom_collect2}. In general, if the predicted field has a high concentration in regions, the nominal approach will tend to place most of its collectors in such areas. The robust approach, on the other hand, will tend to favor regions with a more diffuse predicted resource distribution even if the aggregate total is less than the nominal prediction around a peak to hedge against the possibility that the peak prediction is incorrect. In \Cref{fig:rob_nom_collect1}, we see that the nominal approach concentrates two collectors in the bottom left, whereas the robust approach spreads these out; under the true field, we see that the spectral operator prediction was overly optimistic, improving the relative performance of the robust collectors.

\begin{figure}[H]
  \centering 
  \includegraphics[width=0.75\textwidth]{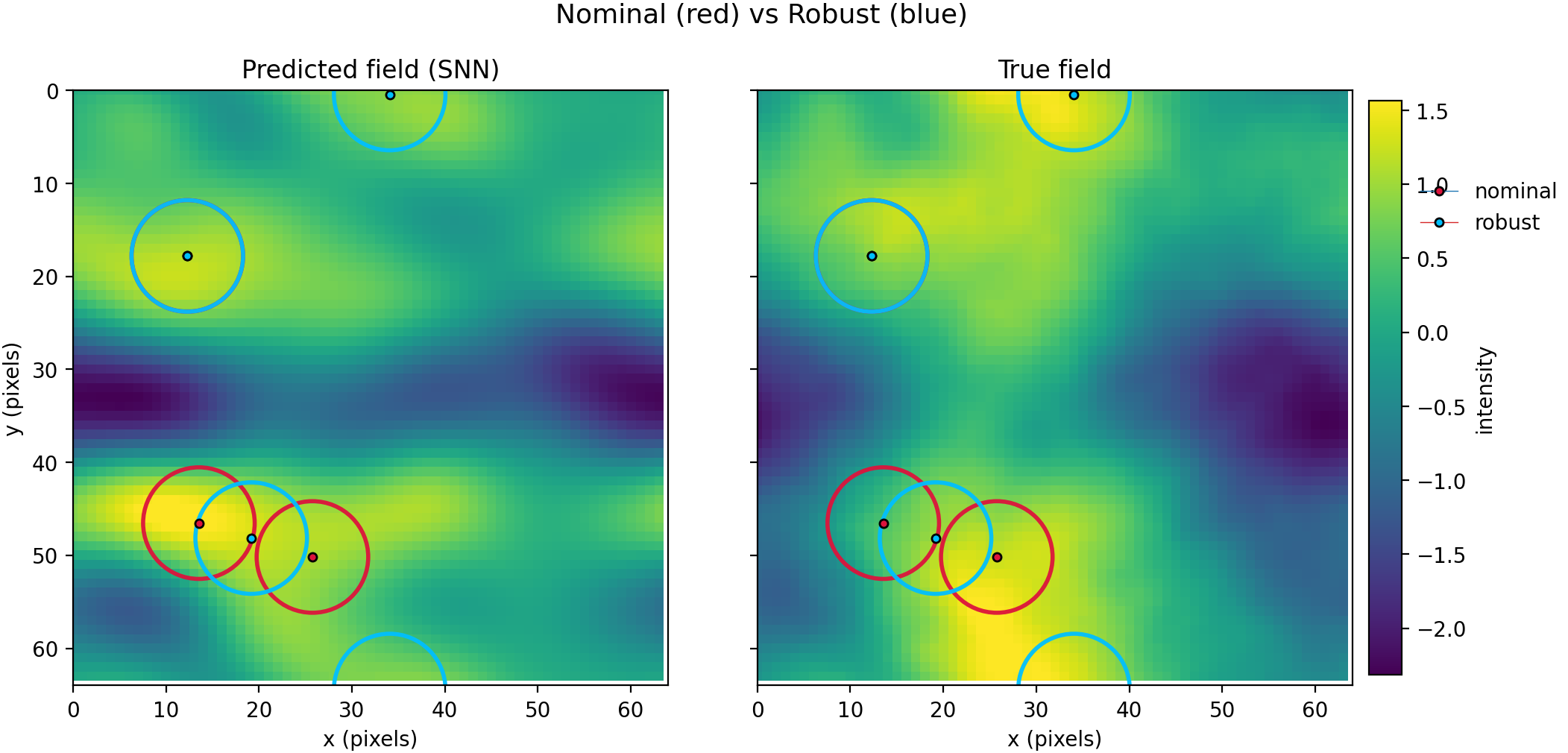}
  \caption{\label{fig:rob_nom_collect1} Visualization of collection solutions for the nominal and robust solutions laid atop the nominal field prediction (left) and true field (right).} 
\end{figure}

\begin{figure}[H]
  \centering 
  \includegraphics[width=0.75\textwidth]{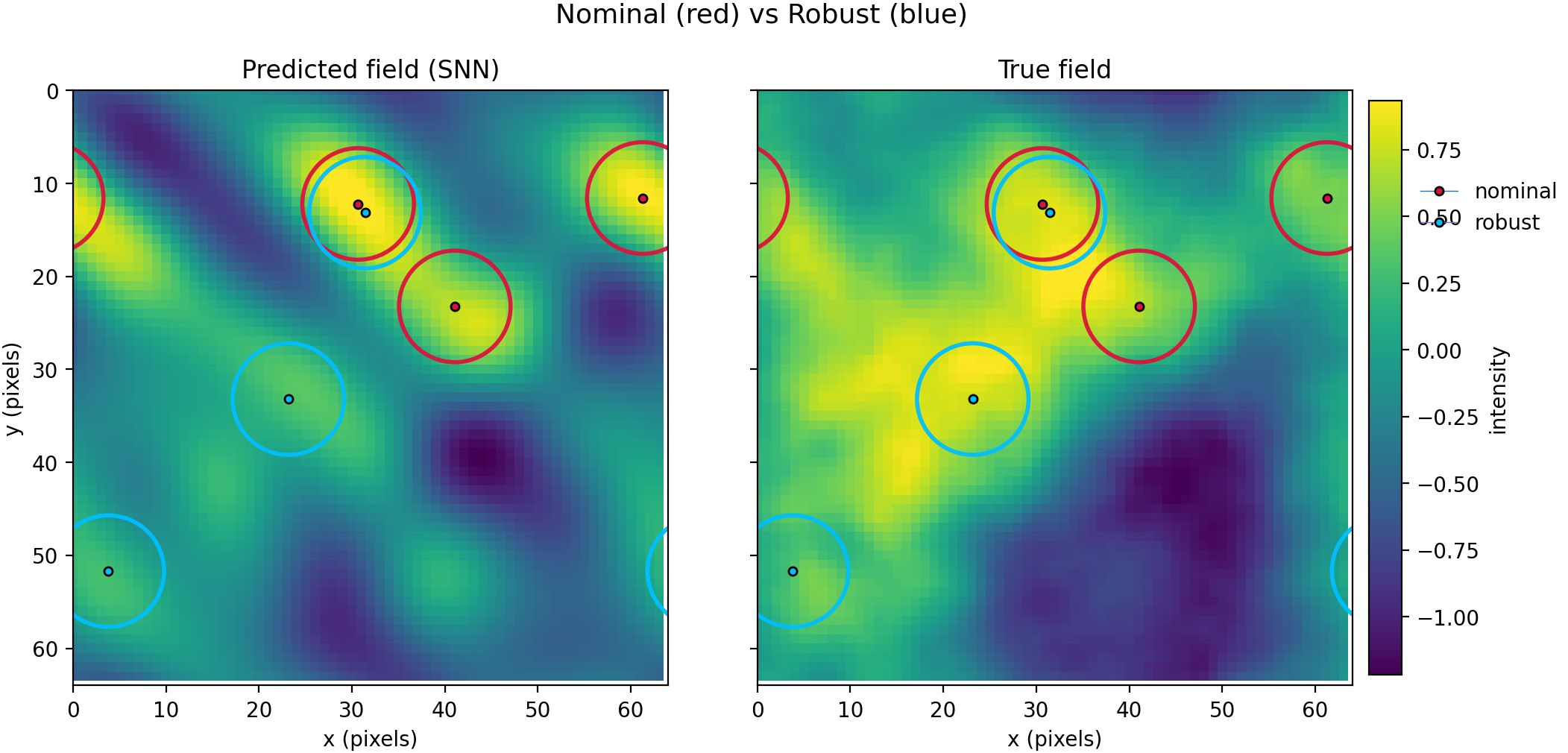}
  \caption{\label{fig:rob_nom_collect2} Visualization of collection solutions for the nominal and robust solutions laid atop the nominal field prediction (left) and true field (right).} 
\end{figure}

\end{document}

%% file: math_commands.tex
\usepackage{amsmath,amsfonts,bm}

\def\eqref#1{equation~\ref{#1}}

\def\ceil#1{\lceil #1 \rceil}

\def\1{\bm{1}}

\DeclareMathAlphabet{\mathsfit}{\encodingdefault}{\sfdefault}{m}{sl}
\SetMathAlphabet{\mathsfit}{bold}{\encodingdefault}{\sfdefault}{bx}{n}

\DeclareMathOperator*{\argmax}{arg\,max}
\DeclareMathOperator*{\argmin}{arg\,min}

%% file: references.bib
@article{angelopoulos2021gentle,
  title={A gentle introduction to conformal prediction and distribution-free uncertainty quantification},
  author={Angelopoulos, Anastasios N and Bates, Stephen},
  journal={arXiv preprint arXiv:2107.07511},
  year={2021}
}

@article{shafer2008tutorial,
  title={A Tutorial on Conformal Prediction.},
  author={Shafer, Glenn and Vovk, Vladimir},
  journal={Journal of Machine Learning Research},
  volume={9},
  number={3},
  year={2008}
}

@article{zou2023hydra,
  title={L-HYDRA: multi-head physics-informed neural networks},
  author={Zou, Zongren and Karniadakis, George Em},
  journal={arXiv preprint arXiv:2301.02152},
  year={2023}
}

@article{psaros2023uncertainty,
  title={Uncertainty quantification in scientific machine learning: Methods, metrics, and comparisons},
  author={Psaros, Apostolos F and Meng, Xuhui and Zou, Zongren and Guo, Ling and Karniadakis, George Em},
  journal={Journal of Computational Physics},
  volume={477},
  pages={111902},
  year={2023},
  publisher={Elsevier}
}

@article{patel2023conformal,
  title={Conformal Contextual Robust Optimization},
  author={Patel, Yash and Rayan, Sahana and Tewari, Ambuj},
  journal={arXiv preprint arXiv:2310.10003},
  year={2023}
}

@book{evans2022partial,
  title={Partial differential equations},
  author={Evans, Lawrence C},
  volume={19},
  year={2022},
  publisher={American Mathematical Society}
}

@article{zhu2019physics,
  title={Physics-constrained deep learning for high-dimensional surrogate modeling and uncertainty quantification without labeled data},
  author={Zhu, Yinhao and Zabaras, Nicholas and Koutsourelakis, Phaedon-Stelios and Perdikaris, Paris},
  journal={Journal of Computational Physics},
  volume={394},
  pages={56--81},
  year={2019},
  publisher={Elsevier}
}

@article{martina2021bayesian,
  title={Bayesian uncertainty quantification for data-driven equation learning},
  author={Martina-Perez, Simon and Simpson, Matthew J and Baker, Ruth E},
  journal={Proceedings of the Royal Society A},
  volume={477},
  number={2254},
  pages={20210426},
  year={2021},
  publisher={The Royal Society}
}

@article{tripathy2018deep,
  title={Deep UQ: Learning deep neural network surrogate models for high dimensional uncertainty quantification},
  author={Tripathy, Rohit K and Bilionis, Ilias},
  journal={Journal of computational physics},
  volume={375},
  pages={565--588},
  year={2018},
  publisher={Elsevier}
}

@article{li2020fourier,
  title={Fourier neural operator for parametric partial differential equations},
  author={Li, Zongyi and Kovachki, Nikola and Azizzadenesheli, Kamyar and Liu, Burigede and Bhattacharya, Kaushik and Stuart, Andrew and Anandkumar, Anima},
  journal={arXiv preprint arXiv:2010.08895},
  year={2020}
}

@article{sahin2024deep,
  title={Deep operator learning-based surrogate models with uncertainty quantification for optimizing internal cooling channel rib profiles},
  author={Sahin, Izzet and Moya, Christian and Mollaali, Amirhossein and Lin, Guang and Paniagua, Guillermo},
  journal={International Journal of Heat and Mass Transfer},
  volume={219},
  pages={124813},
  year={2024},
  publisher={Elsevier}
}

@article{mollaali2023physics,
  title={A Physics-Guided Bi-Fidelity Fourier-Featured Operator Learning Framework for Predicting Time Evolution of Drag and Lift Coefficients},
  author={Mollaali, Amirhossein and Sahin, Izzet and Raza, Iqrar and Moya, Christian and Paniagua, Guillermo and Lin, Guang},
  journal={arXiv preprint arXiv:2311.03639},
  year={2023}
}

@article{chenreddy2022data,
  title={Data-driven conditional robust optimization},
  author={Chenreddy, Abhilash Reddy and Bandi, Nymisha and Delage, Erick},
  journal={Advances in Neural Information Processing Systems},
  volume={35},
  pages={9525--9537},
  year={2022}
}

@article{pathak2022fourcastnet,
  title={Fourcastnet: A global data-driven high-resolution weather model using adaptive fourier neural operators},
  author={Pathak, Jaideep and Subramanian, Shashank and Harrington, Peter and Raja, Sanjeev and Chattopadhyay, Ashesh and Mardani, Morteza and Kurth, Thorsten and Hall, David and Li, Zongyi and Azizzadenesheli, Kamyar and others},
  journal={arXiv preprint arXiv:2202.11214},
  year={2022}
}

@article{wen2022accelerating,
  title={Accelerating carbon capture and storage modeling using fourier neural operators},
  author={Wen, Gege and Li, Zongyi and Long, Qirui and Azizzadenesheli, Kamyar and Anandkumar, Anima and Benson, Sally M},
  journal={arXiv},
  year={2022}
}

@article{bonev2023spherical,
  title={Spherical Fourier Neural Operators: Learning Stable Dynamics on the Sphere},
  author={Bonev, Boris and Kurth, Thorsten and Hundt, Christian and Pathak, Jaideep and Baust, Maximilian and Kashinath, Karthik and Anandkumar, Anima},
  journal={arXiv preprint arXiv:2306.03838},
  year={2023}
}

@article{li2020neural,
  title={Neural operator: Graph kernel network for partial differential equations},
  author={Li, Zongyi and Kovachki, Nikola and Azizzadenesheli, Kamyar and Liu, Burigede and Bhattacharya, Kaushik and Stuart, Andrew and Anandkumar, Anima},
  journal={arXiv preprint arXiv:2003.03485},
  year={2020}
}

@article{li2020multipole,
  title={Multipole graph neural operator for parametric partial differential equations},
  author={Li, Zongyi and Kovachki, Nikola and Azizzadenesheli, Kamyar and Liu, Burigede and Stuart, Andrew and Bhattacharya, Kaushik and Anandkumar, Anima},
  journal={Advances in Neural Information Processing Systems},
  volume={33},
  pages={6755--6766},
  year={2020}
}

@article{sadana2024survey,
  title={A survey of contextual optimization methods for decision-making under uncertainty},
  author={Sadana, Utsav and Chenreddy, Abhilash and Delage, Erick and Forel, Alexandre and Frejinger, Emma and Vidal, Thibaut},
  journal={European Journal of Operational Research},
  year={2024},
  publisher={Elsevier}
}

@inproceedings{patel2024conformal,
  title={Conformal contextual robust optimization},
  author={Patel, Yash P and Rayan, Sahana and Tewari, Ambuj},
  booktitle={International Conference on Artificial Intelligence and Statistics},
  pages={2485--2493},
  year={2024},
  organization={PMLR}
}

@article{chenreddy2024end,
  title={End-to-end Conditional Robust Optimization},
  author={Chenreddy, Abhilash and Delage, Erick},
  journal={arXiv preprint arXiv:2403.04670},
  year={2024}
}

@book{sokolowski1992introduction,
  title={Introduction to shape optimization},
  author={Sokolowski, Jan and Zol{\'e}sio, Jean-Paul and Sokolowski, Jan and Zolesio, Jean-Paul},
  year={1992},
  publisher={Springer}
}

@article{hsu1994review,
  title={A review of structural shape optimization},
  author={Hsu, Yeh-Liang},
  journal={Computers in Industry},
  volume={25},
  number={1},
  pages={3--13},
  year={1994},
  publisher={Elsevier}
}

@article{challis2009level,
  title={Level set topology optimization of fluids in Stokes flow},
  author={Challis, Vivien J and Guest, James K},
  journal={International journal for numerical methods in engineering},
  volume={79},
  number={10},
  pages={1284--1308},
  year={2009},
  publisher={Wiley Online Library}
}

@article{dunning2015introducing,
  title={Introducing the sequential linear programming level-set method for topology optimization},
  author={Dunning, Peter D and Kim, H Alicia},
  journal={Structural and Multidisciplinary Optimization},
  volume={51},
  pages={631--643},
  year={2015},
  publisher={Springer}
}

@article{patel2024conformallinear,
  title={Conformal Robust Control of Linear Systems},
  author={Patel, Yash and Rayan, Sahana and Tewari, Ambuj},
  journal={arXiv preprint arXiv:2405.16250},
  year={2024}
}

@article{ma2024calibrated,
  title={Calibrated Uncertainty Quantification for Operator Learning via Conformal Prediction},
  author={Ma, Ziqi and Azizzadenesheli, Kamyar and Anandkumar, Anima},
  journal={arXiv preprint arXiv:2402.01960},
  year={2024}
}

@incollection{allaire2021shape,
  title={Shape and topology optimization},
  author={Allaire, Gr{\'e}goire and Dapogny, Charles and Jouve, Fran{\c{c}}ois},
  booktitle={Handbook of numerical analysis},
  volume={22},
  pages={1--132},
  year={2021},
  publisher={Elsevier}
}

@book{delfour2011shapes,
  title={Shapes and geometries: metrics, analysis, differential calculus, and optimization},
  author={Delfour, Michel C and Zol{\'e}sio, J-P},
  year={2011},
  publisher={SIAM}
}

@article{ulbrich2024generalized,
  title={On generalized Nash equilibrium problems in infinite-dimensional spaces using Nikaido--Isoda type functionals},
  author={Ulbrich, Michael and Fritz, Julia},
  journal={Optimization Methods and Software},
  pages={1--31},
  year={2024},
  publisher={Taylor \& Francis}
}

@article{wei2015continuous,
  title={Continuous space maximal coverage: Insights, advances and challenges},
  author={Wei, Ran and Murray, Alan T},
  journal={Computers \& operations research},
  volume={62},
  pages={325--336},
  year={2015},
  publisher={Elsevier}
}

@article{yang2020continuous,
  title={The continuous maximal covering location problem in large-scale natural disaster rescue scenes},
  author={Yang, Pei and Xiao, Yiyong and Zhang, Yue and Zhou, Shenghan and Yang, Jun and Xu, Yuchun},
  journal={Computers \& Industrial Engineering},
  volume={146},
  pages={106608},
  year={2020},
  publisher={Elsevier}
}

@article{huang2016design,
  title={The design of electric vehicle charging network},
  author={Huang, Kai and Kanaroglou, Pavlos and Zhang, Xiaozhou},
  journal={Transportation Research Part D: Transport and Environment},
  volume={49},
  pages={1--17},
  year={2016},
  publisher={Elsevier}
}

@inproceedings{fajardo2018placing,
  title={Placing Wi-Fi Hotspots in Havana with locations availability based on fuzzy constraints},
  author={Fajardo-Calder{\'\i}n, Jenny and Lamata, Mar{\'\i}a Teresa and Pelta, David A and Porras, Cynthia and Rosete, Alejandro and Verdegay, Jos{\'e} L},
  booktitle={2018 IEEE International Conference on Fuzzy Systems (FUZZ-IEEE)},
  pages={1--6},
  year={2018},
  organization={IEEE}
}

@article{gopakumar2024valid,
  title={Valid Error Bars for Neural Weather Models using Conformal Prediction},
  author={Gopakumar, Vignesh and Oskarrson, Joel and Gray, Ander and Zanisi, Lorenzo and Pamela, Stanislas and Giles, Daniel and Kusner, Matt and Deisenroth, Marc},
  journal={arXiv preprint arXiv:2406.14483},
  year={2024}
}

@book{brezis2011functional,
  title={Functional analysis, Sobolev spaces and partial differential equations},
  author={Brezis, Haim and Br{\'e}zis, Haim},
  volume={2},
  number={3},
  year={2011},
  publisher={Springer}
}

@inproceedings{fanaskov2023spectral,
  title={Spectral neural operators},
  author={Fanaskov, Vladimir Sergeevich and Oseledets, Ivan V},
  booktitle={Doklady Mathematics},
  volume={108},
  number={Suppl 2},
  pages={S226--S232},
  year={2023},
  organization={Springer}
}

@article{wu2023solving,
  title={Solving high-dimensional pdes with latent spectral models},
  author={Wu, Haixu and Hu, Tengge and Luo, Huakun and Wang, Jianmin and Long, Mingsheng},
  journal={arXiv preprint arXiv:2301.12664},
  year={2023}
}

@article{du2023neural,
  title={Neural spectral methods: Self-supervised learning in the spectral domain},
  author={Du, Yiheng and Chalapathi, Nithin and Krishnapriyan, Aditi},
  journal={arXiv preprint arXiv:2312.05225},
  year={2023}
}

@article{liu2023spfno,
  title={SPFNO: Spectral operator learning for PDEs with Dirichlet and Neumann boundary conditions},
  author={Liu, Ziyuan and Wu, Yuhang and Huang, Daniel Zhengyu and Zhang, Hong and Qian, Xu and Song, Songhe},
  journal={arXiv preprint arXiv:2312.06980},
  year={2023}
}

@article{carnevale2021system,
  title={A system of systems for the optimal allocation of pollutant monitoring sensors},
  author={Carnevale, Claudio and Sangiorgi, Lucia and De Angelis, Elena and Mansini, Renata and Volta, Marialuisa},
  journal={IEEE Systems Journal},
  volume={16},
  number={4},
  pages={6393--6400},
  year={2021},
  publisher={IEEE}
}

@article{boubrima2017optimal,
  title={Optimal WSN deployment models for air pollution monitoring},
  author={Boubrima, Ahmed and Bechkit, Walid and Rivano, Herve},
  journal={IEEE Transactions on Wireless Communications},
  volume={16},
  number={5},
  pages={2723--2735},
  year={2017},
  publisher={IEEE}
}

@article{kessler1998detecting,
  title={Detecting accidental contaminations in municipal water networks},
  author={Kessler, Avner and Ostfeld, Avi and Sinai, Gideon},
  journal={Journal of Water Resources Planning and Management},
  volume={124},
  number={4},
  pages={192--198},
  year={1998},
  publisher={American Society of Civil Engineers}
}

@article{vianna2019set,
  title={The set covering problem applied to optimisation of gas detectors in chemical process plants},
  author={Vianna, S{\'a}vio SV},
  journal={Computers \& Chemical Engineering},
  volume={121},
  pages={388--395},
  year={2019},
  publisher={Elsevier}
}

@article{notarnicola2023optimizing,
  title={Optimizing state-discrimination receivers for continuous-variable quantum key distribution over a wiretap channel},
  author={Notarnicola, Michele N and Jarzyna, Marcin and Olivares, Stefano and Banaszek, Konrad},
  journal={New Journal of Physics},
  volume={25},
  number={10},
  pages={103014},
  year={2023},
  publisher={IOP Publishing}
}

@book{cariolaro2015quantum,
  title={Quantum communications},
  author={Cariolaro, Gianfranco},
  volume={2},
  year={2015},
  publisher={Springer}
}

@inproceedings{ikramov2009theorems,
  title={Theorems of the Hoffman-Wielandt type for the coneigenvalues of complex matrices},
  author={Ikramov, Kh D and Nesterenko, Yu R},
  booktitle={Doklady Mathematics},
  volume={80},
  pages={536--540},
  year={2009},
  organization={SP MAIK Nauka/Interperiodica}
}

@article{vasani2024embracing,
  title={Embracing the quantum frontier: Investigating quantum communication, cryptography, applications and future directions},
  author={Vasani, Vatsal and Prateek, Kumar and Amin, Ruhul and Maity, Soumyadev and Dwivedi, Ashutosh Dhar},
  journal={Journal of Industrial Information Integration},
  pages={100594},
  year={2024},
  publisher={Elsevier}
}

@article{sidhu2021advances,
  title={Advances in space quantum communications},
  author={Sidhu, Jasminder S and Joshi, Siddarth K and G{\"u}ndo{\u{g}}an, Mustafa and Brougham, Thomas and Lowndes, David and Mazzarella, Luca and Krutzik, Markus and Mohapatra, Sonali and Dequal, Daniele and Vallone, Giuseppe and others},
  journal={IET Quantum Communication},
  volume={2},
  number={4},
  pages={182--217},
  year={2021},
  publisher={Wiley Online Library}
}

@article{gisin2007quantum,
  title={Quantum communication},
  author={Gisin, Nicolas and Thew, Rob},
  journal={Nature photonics},
  volume={1},
  number={3},
  pages={165--171},
  year={2007},
  publisher={Nature Publishing Group UK London}
}

@article{bedington2017progress,
  title={Progress in satellite quantum key distribution},
  author={Bedington, Robert and Arrazola, Juan Miguel and Ling, Alexander},
  journal={npj Quantum Information},
  volume={3},
  number={1},
  pages={30},
  year={2017},
  publisher={Nature Publishing Group UK London}
}

@article{zhang2024continuous,
  title={Continuous-variable quantum key distribution system: Past, present, and future},
  author={Zhang, Yichen and Bian, Yiming and Li, Zhengyu and Yu, Song and Guo, Hong},
  journal={Applied Physics Reviews},
  volume={11},
  number={1},
  year={2024},
  publisher={AIP Publishing}
}

@article{jain2022practical,
  title={Practical continuous-variable quantum key distribution with composable security},
  author={Jain, Nitin and Chin, Hou-Man and Mani, Hossein and Lupo, Cosmo and Nikolic, Dino Solar and Kordts, Arne and Pirandola, Stefano and Pedersen, Thomas Brochmann and Kolb, Matthias and {\"O}mer, Bernhard and others},
  journal={Nature communications},
  volume={13},
  number={1},
  pages={4740},
  year={2022},
  publisher={Nature Publishing Group UK London}
}

@article{scarani2009security,
  title={The security of practical quantum key distribution},
  author={Scarani, Valerio and Bechmann-Pasquinucci, Helle and Cerf, Nicolas J and Du{\v{s}}ek, Miloslav and L{\"u}tkenhaus, Norbert and Peev, Momtchil},
  journal={Reviews of modern physics},
  volume={81},
  number={3},
  pages={1301--1350},
  year={2009},
  publisher={APS}
}

@article{xu2020secure,
  title={Secure quantum key distribution with realistic devices},
  author={Xu, Feihu and Ma, Xiongfeng and Zhang, Qiang and Lo, Hoi-Kwong and Pan, Jian-Wei},
  journal={Reviews of modern physics},
  volume={92},
  number={2},
  pages={025002},
  year={2020},
  publisher={APS}
}

@article{sadana2025survey,
  title={A survey of contextual optimization methods for decision-making under uncertainty},
  author={Sadana, Utsav and Chenreddy, Abhilash and Delage, Erick and Forel, Alexandre and Frejinger, Emma and Vidal, Thibaut},
  journal={European Journal of Operational Research},
  volume={320},
  number={2},
  pages={271--289},
  year={2025},
  publisher={Elsevier}
}

@article{gopakumar2025calibrated,
  title={Calibrated Physics-Informed Uncertainty Quantification},
  author={Gopakumar, Vignesh and Gray, Ander and Zanisi, Lorenzo and Nunn, Timothy and Pamela, Stanislas and Giles, Daniel and Kusner, Matt J and Deisenroth, Marc Peter},
  journal={arXiv preprint arXiv:2502.04406},
  year={2025}
}

@article{kitagawa2015quantum,
  title={Quantum description of electromagnetic fields in waveguides},
  author={Kitagawa, Akira},
  journal={arXiv preprint arXiv:1510.06836},
  year={2015}
}

@article{collado2024harmonic,
  title={Harmonic motion modes in parabolic GRIN fibers},
  author={Collado Hern{\'a}ndez, A and Marroqu{\'\i}n Guti{\'e}rrrez, F and Rodr{\'\i}guez-Lara, BM},
  journal={Optics Continuum},
  volume={3},
  number={6},
  pages={1025--1037},
  year={2024},
  publisher={Optica Publishing Group}
}

@book{wilkinson1970elementary,
  title={Elementary proof of the Wielandt-Hoffman theorem and of its generalization},
  author={Wilkinson, James H},
  year={1970},
  publisher={Stanford University}
}

@misc{Sturm_EllipticPDE_2017,
  author       = {Jacob Sturm},
  title        = {Elliptic Partial Differential Equations},
  howpublished = {\url{https://sites.rutgers.edu/jacob-sturm/wp-content/uploads/sites/553/2021/11/Elliptic-PDE-112717.pdf}},
  note         = {Lecture notes, Rutgers University. Version dated 27 Nov 2017},
  year         = {2017},
  urldate      = {2025-06-08}
}

@article{zhang2006application,
  title={Application of CFD in ship engineering design practice and ship hydrodynamics},
  author={Zhang, Zhi-rong and Hui, Liu and Feng, ZHAO and others},
  journal={Journal of Hydrodynamics, Ser. B},
  volume={18},
  number={3},
  pages={315--322},
  year={2006},
  publisher={Elsevier}
}

@article{caron2025machine,
  title={Machine Learning to speed up Computational Fluid Dynamics engineering simulations for built environments: A review},
  author={Caron, Clement and Lauret, Philippe and Bastide, Alain},
  journal={Building and Environment},
  volume={267},
  pages={112229},
  year={2025},
  publisher={Elsevier}
}

@inproceedings{lekeufack2024conformal,
  title={Conformal decision theory: Safe autonomous decisions from imperfect predictions},
  author={Lekeufack, Jordan and Angelopoulos, Anastasios N and Bajcsy, Andrea and Jordan, Michael I and Malik, Jitendra},
  booktitle={2024 IEEE International Conference on Robotics and Automation (ICRA)},
  pages={11668--11675},
  year={2024},
  organization={IEEE}
}

@article{cresswell2024conformal,
  title={Conformal prediction sets improve human decision making},
  author={Cresswell, Jesse C and Sui, Yi and Kumar, Bhargava and Vouitsis, No{\"e}l},
  journal={arXiv preprint arXiv:2401.13744},
  year={2024}
}

@article{kiyani2025decision,
  title={Decision theoretic foundations for conformal prediction: Optimal uncertainty quantification for risk-averse agents},
  author={Kiyani, Shayan and Pappas, George and Roth, Aaron and Hassani, Hamed},
  journal={arXiv preprint arXiv:2502.02561},
  year={2025}
}

@article{cortes2024utility,
  title={Utility-Directed Conformal Prediction: A Decision-Aware Framework for Actionable Uncertainty Quantification},
  author={Cortes-Gomez, Santiago and Patino, Carlos and Byun, Yewon and Wu, Steven and Horvitz, Eric and Wilder, Bryan},
  journal={arXiv preprint arXiv:2410.01767},
  year={2024}
}

@inproceedings{grayguaranteed,
  title={Guaranteed Prediction Sets for Functional Surrogate models},
  author={Gray, Ander and Gopakumar, Vignesh and Rousseau, Sylvain and Destercke, Sebastien},
  booktitle={The 41st Conference on Uncertainty in Artificial Intelligence}
}

@article{garrigos2023handbook,
  title={Handbook of convergence theorems for (stochastic) gradient methods},
  author={Garrigos, Guillaume and Gower, Robert M},
  journal={arXiv preprint arXiv:2301.11235},
  year={2023}
}

@article{sriram2025open,
  title={The Open DAC 2025 Dataset for Sorbent Discovery in Direct Air Capture},
  author={Sriram, Anuroop and Brabson, Logan M and Yu, Xiaohan and Choi, Sihoon and Abdelmaqsoud, Kareem and Moubarak, Elias and de Haan, Pim and L{\"o}we, Sindy and Brehmer, Johann and Kitchin, John R and others},
  journal={arXiv preprint arXiv:2508.03162},
  year={2025}
}

@article{moon2025catbench,
  title={CatBench framework for benchmarking machine learning interatomic potentials in adsorption energy predictions for heterogeneous catalysis},
  author={Moon, Jinuk and Jeon, Uchan and Choung, Seokhyun and Han, Jeong Woo},
  journal={Cell Reports Physical Science},
  volume={6},
  number={12},
  year={2025},
  publisher={Elsevier}
}

@book{pazy2012semigroups,
  title={Semigroups of linear operators and applications to partial differential equations},
  author={Pazy, Amnon},
  volume={44},
  year={2012},
  publisher={Springer Science \& Business Media}
}
